\documentclass{article}

\PassOptionsToPackage{square,comma,numbers}{natbib}
\usepackage[final]{neurips_2024}

\usepackage[utf8]{inputenc} % allow utf-8 input
\usepackage[T1]{fontenc}    % use 8-bit T1 fonts
\usepackage{hyperref}       % hyperlinks
\usepackage{url}            % simple URL typesetting
\usepackage{booktabs}       % professional-quality tables
\usepackage{amsfonts}       % blackboard math symbols
\usepackage{nicefrac}       % compact symbols for 1/2, etc.
\usepackage{microtype}      % microtypography
\usepackage{xcolor}         % colors
\usepackage{amsmath}
\usepackage{junghoon}
\usepackage{caption}
\usepackage{subcaption}
\usepackage{algorithm}
\usepackage{wrapfig}
\usepackage{thm-restate}
\usepackage{cleveref}
\usepackage{etoc}

\usepackage{amsmath}        % For mathematical symbols
\usepackage{tcolorbox}       % For colored boxes
\usepackage{amsmath}

\usepackage{algpseudocode}
\usepackage{xcolor}          % For color definitions

\title{Model-Based Transfer Learning \cwinline{for} \\ {\cwinline{Contextual} Reinforcement Learning}}

\author{
  Jung-Hoon Cho\\
  MIT\\
  \texttt{jhooncho@mit.edu} \\
  \And
  Vindula Jayawardana\\
  MIT\\
  \texttt{vindula@mit.edu} \\
  \AND
  Sirui Li\\
  MIT\\
  \texttt{siruil@mit.edu} \\
  \And
  Cathy Wu\\
  MIT\\
  \texttt{cathywu@mit.edu}
}

\begin{document}

\maketitle

\begin{abstract}
    \jhedit{Deep reinforcement learning (RL) \cwinline{is a powerful approach to complex decision making.} 
    \cwinline{However, one issue that limits its practical application is its \jhedit{brittleness}, sometimes failing to train in the presence of small changes in the environment.}}
    \jhedit{Motivated by the success of zero-shot transfer---where pre-trained models perform well on related tasks---we consider the problem of selecting a good set of training tasks to maximize generalization} performance across a range of tasks.
    Given the high cost of training, \cwwwinline{it is critical to select} training tasks \cwwwinline{strategically, but not well understood how to do so}. 
    \jhedit{We hence introduce Model-Based Transfer Learning (MBTL), which layers on top of existing RL methods to effectively solve contextual RL problems.}
    \cwinline{\cwwwinline{MBTL models the generalization performance in two parts: 1) the performance set point, modeled using Gaussian processes, and 2) performance loss (generalization gap), modeled as a linear function of contextual similarity.} 
    \cwwwinline{MBTL combines these two pieces of information within a Bayesian optimization (BO) framework to strategically select training tasks.}}
    \jhedit{We \cwinline{show theoretically that the method exhibits sublinear regret} in the number of \cwinline{training} tasks and discuss conditions to \cwinline{further tighten} regret bounds.}
    \jhedit{We \cwinline{experimentally} validate our methods using urban traffic and standard continuous control benchmarks.} \cwinline{The experimental results suggest that MBTL can \cwwwinline{achieve up to \jhhedit{43x} improved sample efficiency compared with canonical independent training and multi-task training. Further experiments demonstrate the efficacy of \jhedit{BO} and the insensitivity to the underlying RL algorithm and hyperparameters}.
    This work lays the foundations for investigating explicit modeling of generalization, thereby enabling principled yet effective methods for \jhedit{contextual RL}.}
    \jhedit{Code is available at \href{https://github.com/jhoon-cho/MBTL/}{https://github.com/jhoon-cho/MBTL/}.}
\end{abstract}

\section{Introduction}\label{sec:intro} 
Deep reinforcement learning (DRL) has made remarkable strides in addressing complex problems across various domains \cite{mnih_human-level_2015, silver_mastering_2016,al_deeppool_2019,bellemare_autonomous_2020, degrave_magnetic_2022, fawzi_discovering_2022,mankowitz_faster_2023}. 
Despite these successes, DRL \jhedit{algorithms} often exhibit brittleness when exposed to \cwinline{small} variations \jhedit{like different number of lanes, weather conditions, or flow density in traffic benchmarks \cite{jayawardana_mitigating_2024}}, significantly limiting their scalability and generalizability \cite{jayawardana_impact_2022}.
Such variations can be \cwwwinline{modeled using the framework of contextual Markov Decision Processes (CMDP), where task variations can be parameterized within a context space}~\cite{hallak_contextual_2015, modi_markov_2018, benjamins_contextualize_2023}. 

\cwwwinline{There are two predominant solution modalities for CMDPs \vjinline{\cite{kirk_survey_2023}}: independent training and multi-task training.} 
\jhedit{Independent training} \cwwwinline{constructs a separate model} for each task variant \cwwwinline{(say, $N \gg 1$)}, \cwwwinline{which is compute-intensive. At the other extreme,} multi-task training constructs a single ``universal'' policy, and thus can be more compute-efficient, but suffers from \jhedit{model capacity \jhedit{and negative transfer} issues \cite{kang_learning_2011, teh_distral_2017, andreas_modular_2017, standley_which_2020}.}
There is \cwwwinline{thus} a need for more \cwwwinline{reliable training methodologies for} generalization across tasks \cwwwinline{variants}. 
\cwwwinline{In this work, we consider training an intermediate set of $K$ models, where $N > K > 1$, in an effort to balance performance and efficiency; we refer to this strategy as \textit{multi-policy training}.}

\begin{figure}[!t]
    \centering
    \includegraphics[width=0.95\textwidth]{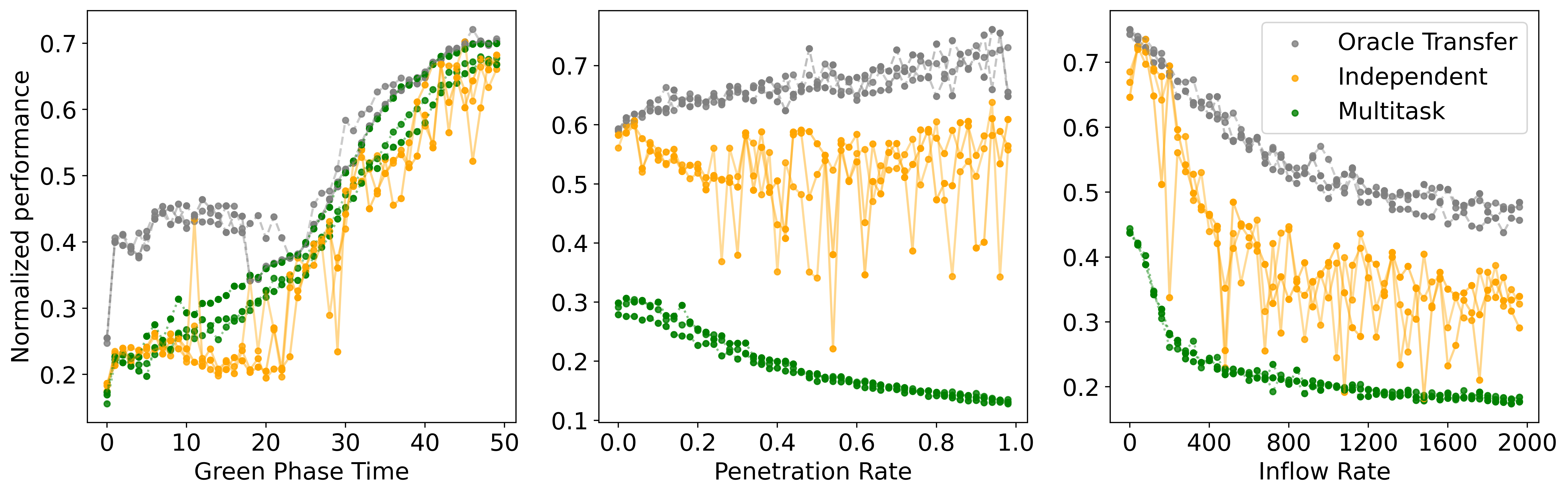}
    \caption{\jhhedit{Normalized performance comparison across different problem variations in Eco-Driving benchmark. Traditional DRL approaches (e.g., Independent training or multi-task training) exhibit greater training instability, whereas Oracle Transfer, zero-shot transfer with full information, shows the potential for performance improvement by multi-policy training.}}
    \label{fig:gap}
    \vspace{-15pt}
\end{figure}

We build upon zero-shot transfer, \cwwwinline{a widely-used practical technique} which directly applies a policy trained in one \cwwwinline{context} (source task) to another (target task) without adaptation. 
\jhhedit{Figure~\ref{fig:gap} shows that multi-policy training with zero-shot transfer has the potential to improve the performance even over the independent training.}
In this article, we \cwwwinline{strategically select source tasks by explicitly modeling} the generalization \cwwwinline{performance} to estimate the value of training \cwwwinline{a new source task}.

\cwwwinline{\textbf{A note on terminology}. For brevity, we refer to \textit{task variants} as \textit{tasks} in the remainder of this article. We also use the language of \textit{task} and \textit{context} interchangeably.
We emphasize that this work focuses on within-domain generalization (e.g., traffic signal control for intersection scenario variants) rather than across-domain generalization (e.g., distinct traffic control tasks).}
\cwwwinline{Additionally,} it is crucial to differentiate between \cwwinline{\textit{training} reliability}, which \cwwwinline{concerns the ability to reliably train models} across tasks, and \textit{model} \cwwwinline{reliability (or robustness)}, which \cwwwinline{concerns} the resistance of \cwwinline{a trained} model to differences in \cwwwinline{tasks}. \cwwwinline{This article is concerned with training reliability.}

\cwwwinline{The main contributions of this work are:}
\begin{itemize}
    \item We introduce \textit{Model-Based Transfer Learning (MBTL)}, a novel framework for \jhhedit{solving CMDP sample efficiently} (Figure~\ref{fig:mbtl-process}). \cwwwinline{To the best of our knowledge, this is the first work to explicitly model generalization performance for \jhhedit{contextual RL (CRL)}. As such, our work opens the door for further investigation into reliable model-based methodologies for \jhhedit{CRL}.} 
    \item We provide theoretical analysis for the sublinear regret of MBTL and give conditions for achieving tighter regret bounds.
    \item We \cwwwinline{empirically} validate our methods in \cwwwinline{urban traffic and standard continuous control benchmarks for contextual RL, observing \textbf{up to \jhhedit{43}x} improvements in sample efficiency. We further include ablations on the components of the algorithm.}
\end{itemize}

\begin{figure}[!t]
    \centering
    \includegraphics[width=0.93\linewidth]{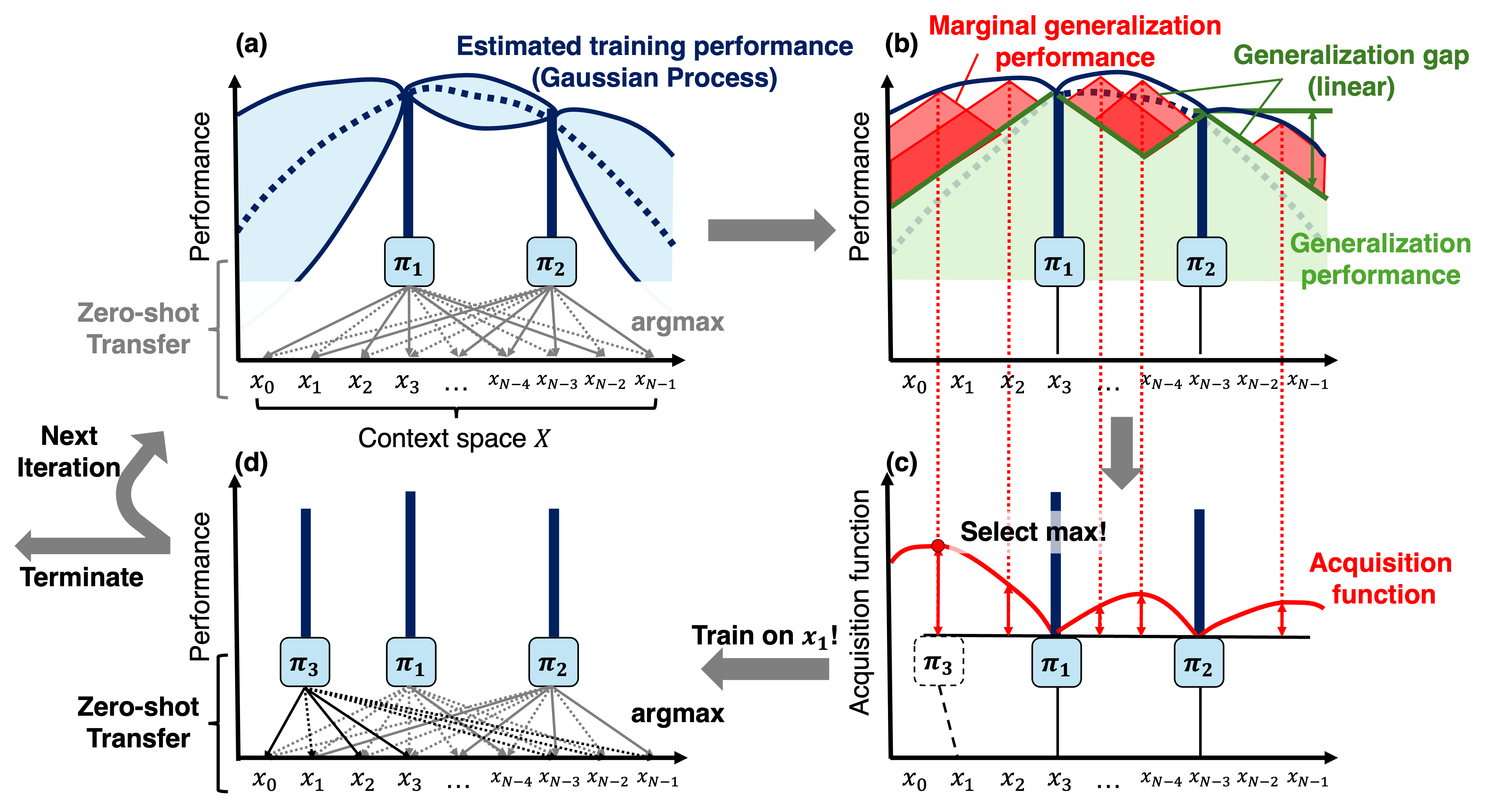}
    \vspace{-8pt}
    \caption{\jhedit{\textbf{\cwwwinline{Overview illustration for Model-based Transfer Learning}.} (a) Gaussian process regression is used to estimate the training performance \cwwwinline{across tasks using} existing policies; (b) marginal generalization performance (red area) is calculated using upper confidence bound of estimated training performance, \cwwwinline{generalization gap,} and generalization performance; (c) selects the next \cwwwinline{training} task that maximizes the acquisition function \cwwwinline{(marginal generalization performance)}; (d) once the selected task is trained, calculate generalization performance using zero-shot transfer.}}
    \vspace{-16pt}
    \label{fig:mbtl-process}
\end{figure}

\cwwwinline{The remainder of the paper is organized as follows. After introducing notation in Section~\ref{sec:prelim},}
we formally define the problem in Section~\ref{sec:problem}. A key contribution of our work is the introduction of a Gaussian process model acquisition function specifically tailored to the source task selection problem, which is detailed in Section~\ref{sec:mbtl}. In Section~\ref{sec:regret-analysis}, we provide a theoretical analysis of the regret bounds of our method, followed by an empirical evaluation across diverse applications in Section~\ref{sec:experiments}.

\section{Preliminaries \cwwinline{and notation}}\label{sec:prelim}

\textbf{Contextual MDP.}
A standard MDP is defined by the tuple $\jhedit{M}=(S, A, P, R, \rho)$ where $S$ represents the state space, $A$ is the action space, $P$ denotes the transition dynamics, $R$ is the reward function, and $\rho$ is the distribution over initial states \cite{sutton_reinforcement_2018}. A \cwwwinline{contextual MDP} (\jhedit{CMDP}), denoted by $\jhedit{\mathcal{M}=} (S, A, P_x, R_x, \rho_x)_{\cwwwinline{x \in X}}$, \cwwwinline{is a collection of \jhedit{context-}MDPs $\jhedit{\mathcal{M}_x}$} parameterized by a context variable $x$ \cwwwinline{within a context space} $X$ \cwwwinline{(assumed bounded)}. 
\cwwwinline{The context variable $x$} \cwinline{can} influence dynamics, rewards, \cwwwinline{and initial state distributions}~\cite{hallak_contextual_2015,modi_markov_2018,benjamins_contextualize_2023}.
\jhedit{We define source task performance $J(\jhedit{\pi_x},x\cwwwinline{;\text{Alg}})$ as follows: we train a policy $\pi_x$ on a task with the context $x\in X$ using RL algorithm $\text{Alg}$ (e.g., PPO, SAC) and evaluate the policy by the expected return in the same MDP $\mathcal{M}_x$ with context $x$. For brevity, we will write it as $J(\jhedit{\pi_x},x)$.}
We \cwwwinline{distinguish} between estimated values $\hat{J}(x)$ and observed outcomes $J(\jhedit{\pi_x},x)$, with the latter measured after training and evaluation. 

\begin{wrapfigure}[17]{r}{.45\textwidth}
    \vspace{-16pt}
    \centering
    \includegraphics[width=0.99\linewidth]{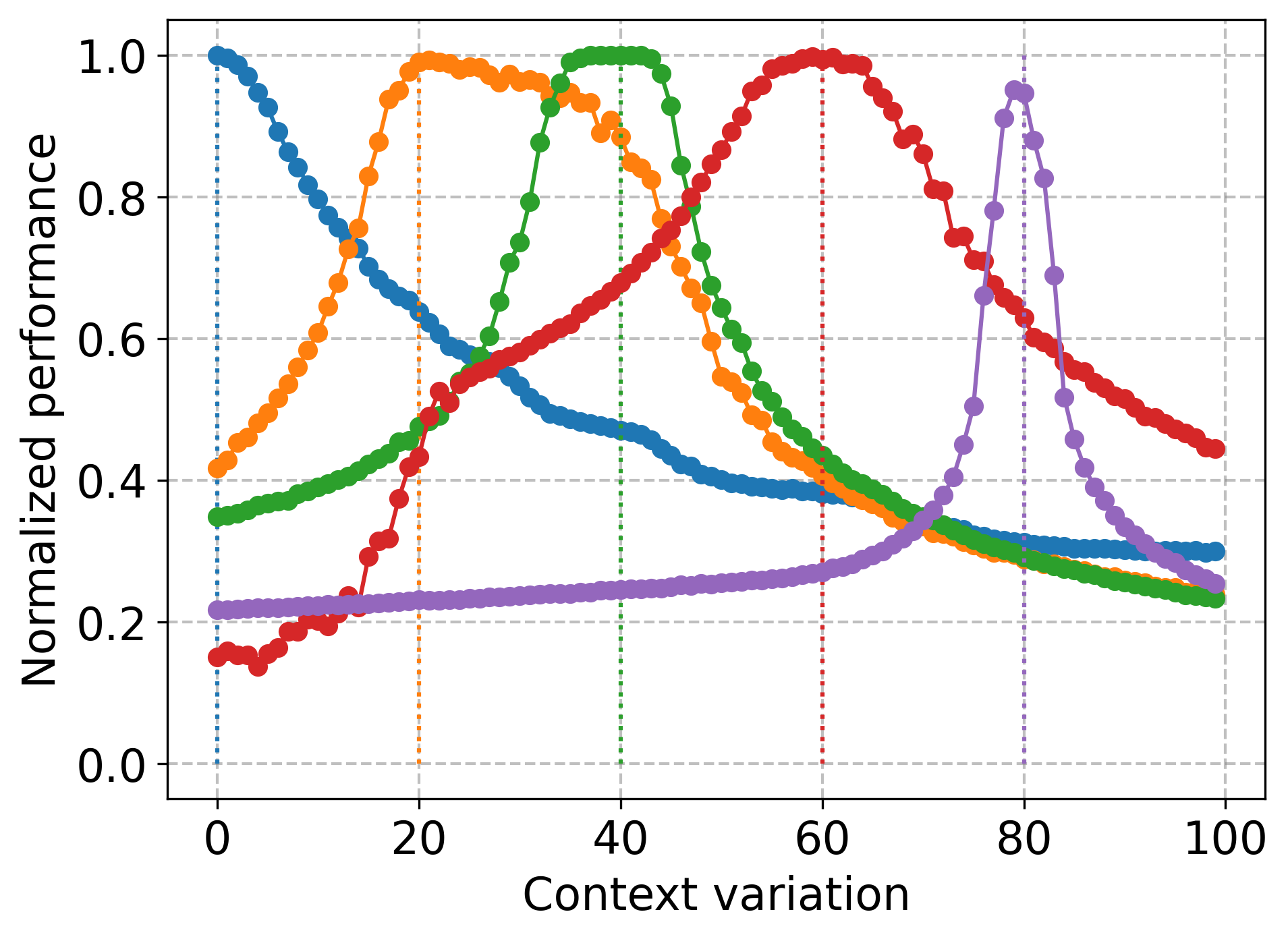}
    \vspace{-15pt}
    \caption{\jhedit{\cwwwinline{Example} generalization gap \cwwwinline{depicted} for Cartpole \jhedit{CMDP}. The solid lines \cwwwinline{show the true zero-shot transfer} \jhedit{generalization} performance across contexts. \cwwwinline{Source tasks are} indicated by dotted lines.}}
    \label{fig:gen-gap-empiric}
\end{wrapfigure}

\textbf{Generalization gap via zero-shot transfer.}
\jhedit{Consider zero-shot transfer from \jhedit{the trained policy $\pi_x$ from} a source task (\jhedit{context-}MDP) to solve another target task (\jhedit{context-}MDP) \jhedit{with the context $x'\in X$} in the \jhedit{CMDP}.} 
Zero-shot transfer involves applying a \jhedit{policy} trained on a source task $\jhedit{\mathcal{M}_x}$ to a different target task $\jhedit{\mathcal{M}_{x'}}$, with \vjinline{$x, x' \in X$}.
\jhhedit{This \cwwwinline{experiences} performance degradation, \cwwwinline{also called} \textit{generalization gap} \cite{higgins_darla_2017, kirk_survey_2023}. 
For instance, Figure~\ref{fig:gen-gap-empiric} depicts that the performance degrades as the target task diverges from the source task, \cwwwinline{corresponding to an} increasing generalization gap.
\cwwinline{\cwwwinline{Nonetheless, leveraging the notion} that training is expensive but zero-shot transfer is cheap, we are interested in optimally selecting a set of source (training) tasks, such that the generalization performance on the target range of tasks is maximized.}}
\jhedit{We observe the \textit{\jhedit{generalization} performance}, denoted by $\jhedit{J(\pi_x,x'})$, by evaluating the target task $x'$ based on the \jhedit{policy} trained \cwwinline{using} source task $x$ \cwwinline{via} zero-shot generalization.}
We define the generalization gap as the absolute performance difference in average reward when transferring from source task $ x $ to target task $ x' $:
\begin{equation}
    \label{eqn:U_pred}
    \underbrace{\Delta J(\jhedit{\pi_x,x'})}_{\text{Generalization gap}}=|\underbrace{J(\jhedit{\pi_x},x)}_{\text{Source task performance}}-\underbrace{J(\jhedit{\pi_x,x'})}_{\text{\jhedit{Generalization} performance}}|.
\end{equation}

\section{Problem formulation}\label{sec:problem} 
\textbf{Sequential source task\jhdelete{s} selection problem.}
The selection of source \vjinline{MDPs} from the \jhedit{CMDP}s is \cwwinline{key to solving the overall \jhedit{CMDP}} \cite{bao_information-theoretic_2019}.
\cwwinline{We therefore} \jhedit{introduce} \cwwwinline{the} \textit{sequential source task selection (SSTS) problem}, \cwwinline{which seeks to} maximize the expected performance across a dynamically selected set of tasks.
\cwwinline{This problem is cast as a sequential decision problem, in which the selection of tasks is informed through feedback from the observed task performance of the tasks selected and trained thus far}. 
The notation $ x_k $ indicates the selected source task at the $ k $-th transfer step, where $ k $ ranges from 1 to $ K $. 
\jhedit{For brevity, we will denote $\pi_{x_k}$ as $\pi_k$.} \cwwwinline{We} denote the sequence $x_1,x_2,...,x_k$ by $\jhedit{x_{1:k}}$ \jhedit{\cwwwinline{and} $\pi_1,\pi_2,...,\pi_k$ by $\jhedit{\pi_{1:k}}$}.

\begin{definition}[Sequential Source Task\jhdelete{s} Selection Problem]
    \label{def:source-task-selection}
    \cwwwinline{This} problem \cwwinline{seeks to optimize the expected generalization performance across a \jhedit{CMDP} $\jhedit{\mathcal{M}_{x'\in X}}$ by selecting a task $x \in X$ at each \vjinline{training} stage}.
    Specifically, at each selection step $k$, we \cwwwinline{wish} to choose a \cwwwinline{distinct} task $x_k$ such that the \cwwwinline{expected} cumulative \cwwwinline{generalization} performance is maximized. \cwwwinline{This can be expressed by keeping track at each step, which policy to use for which task, and the corresponding generalization performance. Upon training the \jhedit{policy} $\pi_{x_k}$ for source task $x_k$, the \cwwwinline{cumulative} generalization performance, \cwwwinline{which we abuse notation to denote as} $\jhedit{{J}(\pi_{1:k},\cdot})={J}(\jhedit{\pi_{x_k},\cdot;\pi_{1:k-1}})$}. Formally, this can be recursively defined based on previous observations $\{\jhedit{J(\pi_1,x)} , \dots, \jhedit{J(\pi_{k-1},x)} \}$ for all $x \in X$ as follows:
\begin{align}
    {J}(\jhedit{\pi_{x_k},x';\pi_{1:k-1}}) &= \max\left(\jhedit{J(\pi_k,x')},{J}(\jhedit{\pi_{1:k-1},x'})\right) \quad \forall x'  \in X \quad\text{if } k > 1.
    \label{eqn:U_obs}
\end{align}
\cwwwinline{And ${J}(\pi_{1:1},x)\equiv {J}(\pi_1,x)$}. 
\cwwwinline{Then, the overall sequential decision problem can be written as:}
    \begin{equation}
    \label{eq:SSTS}
        \max_{x_k} \ \ {V}(x_k;\jhedit{\pi_{1:k-1}}) = \max_{x_k} \mathbb{E}_{x'\sim \mathcal{U}(X)}\left[{J}(\jhedit{\pi_{x_k},x'; \pi_{1:k-1}})\right] \quad
        \text{s.t. } x_k \in X \setminus x_{1:k-1}.
    \end{equation}
\end{definition}
The state at each step $k$ is defined by the \cwinline{best} generalization performance \cwinline{for each task, achieved by \jhedit{policies} trained in earlier stages}, represented as $J(\jhedit{\pi_{1:k-1},x'})$ for each target task $x'$. The action at each step is choosing a new task $x_k$. 
\cwinline{In general, SSTS exhibits stochastic transitions, for example due to randomness in RL training. For simplicity, in this work, we assume deterministic transitions; that is, training \cwwwinline{context-}MDP $x$ will always yield the same performance ${J}(\jhedit{\pi_x},x)$ and generalization gap $\Delta {J}(\jhedit{\pi_x,x'}), \forall x' \in X$.} 
\jhedit{The problem's \cwwwinline{maximum} horizon is defined by $|X|$, \cwwwinline{but can be terminated early if conditions are met (e.g., performance level, \jhedit{training budget}).}}

\vspace{-5pt}
\section{Model-Based Transfer Learning \jhedit{(MBTL)}}\label{sec:mbtl} 
\vspace{-5pt}

\cwwwinline{In this section, we introduce an algorithm \cwwwinline{called Model-based Transfer Learning} to solve the SSTS problem.}
\cwwwinline{MBTL models the generalization performance in two parts: 1) the performance set point, modeled using Gaussian processes, and 2) generalization gap, modeled as a linear function of contextual similarity. MBTL combines these two pieces of information within a Bayesian optimization (BO) framework to sequentially select training tasks that maximize generalization performance}.

\subsection{\cwinline{Modeling assumptions}}\label{sec:assume}
\vspace{-5pt}
\jhedit{We \cwwwinline{consider a} task set $X$ \cwwwinline{that} is continuous and the performance function $J(\jhedit{\pi},x), V(x)$ for a policy $\pi$ \cwwwinline{to be} smooth over the task space $X$. In practice, such as control systems, tasks often vary continuously and smoothly rather than abruptly. For example, adjusting the angle of a robotic arm by a small amount typically results in a small change in the system and optimal action.}
\jhedit{Inspired by the empirical generalization gap performance as observed in Figure~\ref{fig:gen-gap-empiric}, we estimate the generalization gap with a linear function \jhhedit{of contextual similarity}.}
\begin{assumption}[Linear generalization gap]
    \label{assume:linear-generalization}
    A linear \jhedit{function} is used to model the generalization gap, formally $\Delta \hat{J}(\jhedit{\pi_x},x')=\jhhhedit{{J}(\jhedit{\pi_x},x)-\hat{J}(\jhedit{\pi_x},x')}= \theta |x-x'|$, where $\theta$ is the slope of the linear function and $x$ and $x'$ are the context of the source task and target task, respectively.
\end{assumption}

\jhedit{The relaxation of this assumption \vjinline{can yield additional \cwwwinline{efficiency} benefits but \cwwwinline{also adds modeling complexity, and thus is an interesting direction for future work}.}}

\subsection{Bayesian optimization}\label{sec:bo}
\vspace{-5pt}
Bayesian optimization (BO) is a powerful strategy for finding the global optimum of an objective function when obtaining the function is costly, or the function itself lacks a simple analytical form \cite{mockus_bayesian_1989, brochu_tutorial_2010}. 
\jhedit{BO} integrates prior knowledge with observations to efficiently decide the next task to train \jhedit{by using the acquisition function}. 
\cwwwinline{MBTL is a BO method which optimizes for promising source tasks by leveraging Assumption~\ref{assume:linear-generalization} in its acquisition function. The role of BO is to approximate ${V}(x_k;\jhedit{\pi_{1:k-1}})$ (see Equation~\ref{eq:SSTS}) using the data acquired thus far. The next source task $x_k$ is then selected using this estimate.}

\textbf{Gaussian process (GP) \jhedit{regression}.}
Within the framework of BO, we model the source training performance $ \jhedit{\hat{J}}(\jhedit{\pi_x,}x) \ \jhedit{\forall x\in X\setminus x_{1:k}}$ using Gaussian process (GP) regression. 
Specifically, the function $ \jhedit{\hat{J}}(\jhedit{\pi_x,}x)$ is assumed to follow a GP $ (\jhedit{\hat{J}}(\jhedit{\pi_x,}x) \sim \mathcal{GP}(\jhhhedit{m(x)}, k(x, \jhedit{\tilde{x}}))) $, where \jhhhedit{$m(x)$ is mean and} $k(x, \jhedit{\tilde{x}}) $ is the covariance function, representing the expected product of deviations of $ \jhedit{\hat{J}}(\jhedit{\pi_x,}x) $ and $ \jhedit{\hat{J}}(\jhedit{\pi_{\tilde{x}},}\jhedit{\tilde{x}}) $ from their respective means.
Let $ D_{k-1} $ denote the data observed up to iteration $ k-1 $, consisting of the pairs $ \{(x_i, \jhedit{J(\pi_{i},x_i)}\jhdelete{V(x_{\jhedit{1:i}})})\}_{i=1,...,k-1}$. \jhedit{The estimated performance $\hat{J}_{k}$ after querying $k-1$ samples is updated as more samples \cwwwinline{are obtained}.}
The posterior prediction of $ \jhedit{\hat{J}_k} $ at a new point $ x $, given the data $ D_{k-1} $ and previous inputs $ \jhedit{x_{1:k-1}} $, is normally distributed as $ P(\hat{J}_{k} \mid D_{k-1}) = \mathcal{N}(\mu_k(x), \sigma_k^2(x)) $.
$ \mu_k(x) $ and $ \sigma_k^2(x) $ are defined as $\mu_k(x) = \mathbb{E}[\jhedit{\hat{J}}(\jhedit{\pi_x,}x)] + \mathbf{k}^\top (\mathbf{K} + \sigma^2 \mathbf{I})^{-1} \mathbf{y}$ and $\sigma_k^2(x) = k(x, x) - \mathbf{k}^\top (\mathbf{K} + \sigma^2 \mathbf{I})^{-1} \mathbf{k}$, with $ \mathbf{k} $ being the vector of covariances between $x$ and each $ x_i $ in the observed data, i.e., $ \mathbf{k} = [k(x, x_1), \ldots, k(x, x_{k-1})] $, and $ \mathbf{K} $ is the covariance matrix for the observed inputs, defined as $\mathbf{K} = [k(x_i, x_j)]_{1 \leq i, j \leq k-1}$.
This enables the GP to update its beliefs about the \jhedit{posterior prediction} with every new observation, progressively improving the estimation.

\textbf{Acquisition function.}
The acquisition function plays a critical role in BO by guiding the selection of the next source training task. At each decision step $ k $, the task $ x_k $ is chosen by maximizing the acquisition function, as denoted by $x_k=\argmax_x a(x;\jhedit{x_{1:k-1}})$.
One effective strategy employed in the acquisition function is the upper confidence bound (UCB) acquisition function, which considers the trade-off between the expected performance of a task based on the current task (exploitation) and the measure of uncertainty associated with the task's outcome (exploration) \cite{srinivas_information-theoretic_2012}. 
Especially in our case, the acquisition function can be designed \jhedit{as UCB function subtracted by generalization gap and so-far best performance}.
It is defined as follows:
\begin{equation}\label{eqn:acquisition}
    a(x;\jhedit{x_{1:k-1}})=\mathbb{E}_{x' \in X}[[\mu_{k-1}(x) + \beta_k^{1/2}\sigma_{k-1}(x) - \Delta \hat{J}(\jhedit{\pi_x,x'}) - \jhhhedit{\hat{J}}(\jhedit{\pi_{{1:k-1}},x'})]_{+}]
\end{equation}
where $[\cdot]_{+}$ represents $\max(\cdot,0)$ and we can use various forms of $\beta_k$\jhedit{, which is the trade-off parameter between exploitation and exploration}.

\vspace{-5pt}
\subsection{\jhedit{Regret analysis}}\label{sec:regret-analysis} 
\vspace{-5pt}
\jhedit{We use regret to quantify the effectiveness of our source task selection based on BO.}
Specifically, we define regret at iteration $ k $ as $ r_k = \jhhedit{V}(x_k^*;\jhedit{\pi_{1:k-1}}) - \jhhedit{V}(x_k;\jhedit{\pi_{1:k-1}})$, where $\jhhedit{V}(x_k^*;\jhedit{\pi_{1:k-1}})$ represents the maximum \jhedit{generalization} performance achievable across all tasks, and $ \hat{V}(x_k;\jhedit{\pi_{1:k-1}})$ is the performance at the current task selection $x_k$.
Consequently, the cumulative regret after $ K $ iterations is given by $ R_K = \sum_{k=1}^K r_k $, summing the individual regrets over all iterations. Following the framework presented by \citet{srinivas_information-theoretic_2012}, our goal is to establish that this cumulative regret grows sublinearly with respect to the number of iterations. 
\jhedit{Mathematically, we aim to prove that $ \lim_{K \rightarrow \infty} \frac{R_K}{K} = 0 $, indicating that, on average, the performance of our strategy approaches the optimal performance as the number of iterations increases.}

\jhedit{\textbf{Regret of MBTL.}}
\jhedit{Having established the general framework for regret analysis, we now turn our attention to the specific regret properties of our MBTL algorithm.}
\jhedit{To analyze the regret of MBTL}, consider the scaling factor for the UCB acquisition function given by $\beta_k = 2 \log(|X| \pi^2 k^2 / 6\delta)$ in \Cref{eqn:acquisition}. It is designed to achieve sublinear regret with high probability, aligning with the theoretical guarantees outlined in Theorem 1 and 5 from \cite{srinivas_information-theoretic_2012}.
\begin{theorem}[\jhedit{Sublinear Regret}]
    \label{theorem:regret-delta-beta-log}
    Given $\delta \in (0,1)$, and with the scaling factor $\beta_k$ as defined, the cumulative regret $R_K$ is bounded by $\sqrt{K C_1\beta_K\gamma_K}$ with a probability of at least $1 - \delta$. The formal expression of this probability is $Pr\left[R_K\leq\sqrt{K C_1\beta_K\gamma_K}\right]\geq 1-\delta$, where $C_1:=\frac{8}{\log(1+\sigma^{-2})}\geq8\sigma^2$ and $\gamma_K=\mathcal{O}(\log K)$ for the squared exponential kernel.
\end{theorem}

\textbf{Impact of search space elimination.}
In this section, we demonstrate that strategic reduction of the possible sets, guided by insights from previous task selections or source task training performance, leads to significantly tighter regret bounds than Theorem~\ref{theorem:regret-delta-beta-log}. \jhedit{By focusing on the most promising regions of the task space, our approach enhances learning efficiency and maximizes the \jhedit{policy}'s performance and applicability.}
Given the generalization gap observed in Figure~\ref{fig:gen-gap-empiric}, \jhedit{we observe that performance loss decreases as the context similarity increases.}
We model the degradation from the source task using a linear function in Assumption~\ref{assume:linear-generalization}.
Training on the source task can solve a significant portion of the remaining tasks. Our method progressively eliminates partitions of the task space at a certain rate with each iteration. 
\jhedit{If the source task selected in the previous steps could solve the remaining target task sufficiently}, we can eliminate the search space at a desirable rate. \jhedit{Formally, we can define the search space at $k$-th iteration as follows:}
\begin{definition}[\jhedit{Search space}]
    We \jhedit{define the search space} $X_k$ \jhedit{at iteration $k$} as a subset of $X$, with each element $x'\in X_k$, such that $J(\jhedit{\pi_{1:k-1},x'}) \leq \hat{J}(\pi_{x_k},x')\jhedit{-}\Delta \hat{J}(\jhedit{\pi_{x_k},x'})$.
\end{definition}

Given the generalization gap observed in Figure~\ref{fig:gen-gap-empiric}, we model the degradation from the source task using a linear function in Assumption~\ref{assume:linear-generalization}. While the figure might not strictly appear linear, the linear approximation simplifies analysis and is supported by empirical observations. Training on the source task can solve a significant portion of the remaining tasks. Our method progressively eliminates partitions of the task space at a certain rate with each iteration. If the source task selection in the previous step sufficiently addresses the remaining target tasks, we can reduce the search space at a desirable rate. Consequently, at each step, we effectively focus on a reduced search space.

We leverage the reduced uncertainty in well-sampled regions to tighten the regret bound while slightly lowering the probability $\delta$ in Theorem~\ref{theorem:regret-delta-beta-log}.
For the regret analysis, we propose the following theorem based on the generalization of Lemma 5.2 and 5.4 in \cite{srinivas_information-theoretic_2012} to the eliminated search space.

\begin{restatable}{theorem}{regretdeltaxk}
    For a given $ \delta' \in (0,1) $ and scaling factor $ \beta_k=2\log(|X|\pi^2 k^2 /6\delta') $, the cumulative regret $ R_K $ is bounded by $\sqrt{ C_1\beta_K\gamma_K\sum_{k=1}^K\left(\frac{|X_k|}{|X|}\right)^2}$ with probability at least $ 1 - \delta' $.
    \label{theorem:regret-delta-beta-log-x-k}
\end{restatable}
\jhedit{Here, $|X|$ denotes the cardinality of the set $X$, the number of elements in $X$.}
Theorem~\ref{theorem:regret-delta-beta-log-x-k} matches the Theorem~\ref{theorem:regret-delta-beta-log} when $X_k=X$ for all $k$. This theorem implies that regret has a tighter or equivalent bound if we can design the smaller search space instead of searching the whole space.
The comprehensive proof is provided in Appendix~\ref{appsec:proof-log-x-k}. 

Here are some examples of restricted search space: If we consider an example where $|X_k|=\frac{1}{\sqrt{k}}|X|$, the regret can be bounded tighter than that of Theorem~\ref{theorem:regret-delta-beta-log}.
\begin{restatable}{corollary}{regretlogk}
    \label{cor:regret-x-k-logk}
    Consider $|X_k|=\frac{1}{\sqrt{k}}|X|$. The regret bound would be $R_K\leq\sqrt{C_1\beta_K\gamma_K\log K}$ with a probability of at least $1 - \delta'$.
\end{restatable}

\begin{wrapfigure}[12]{r}{.37\textwidth}
    \centering
    \vspace{-35pt}
    \includegraphics[width=0.99\linewidth]{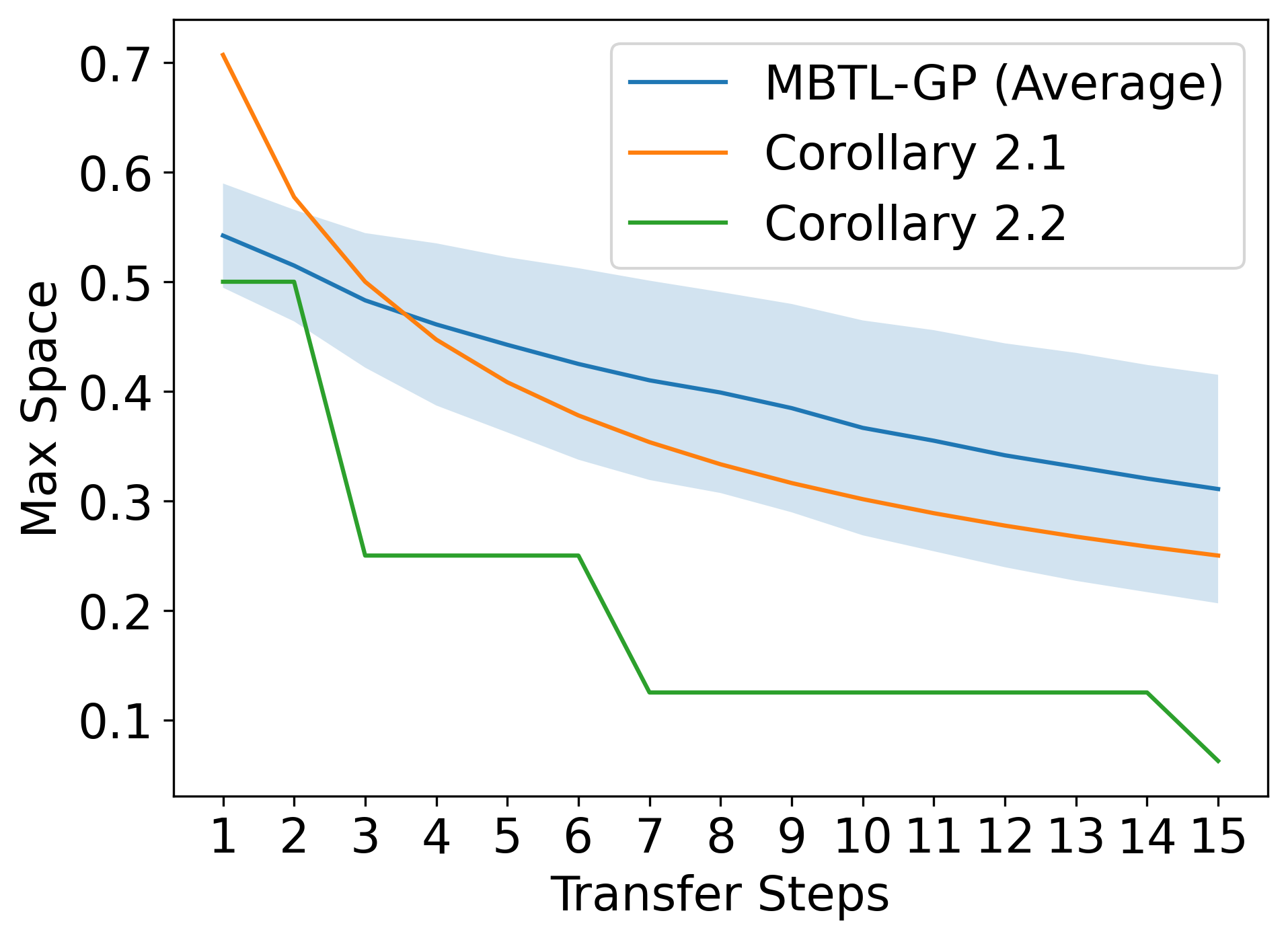}
    \caption{\jhedit{Empirical results of the restriction of search space by MBTL compared to two examples from Corollaries 2.1 and 2.2.}}
    \label{fig:max_space}
\end{wrapfigure}

In cases where the search space is defined using \jhedit{MBTL-GS}, the largest segment's length would reduce geometrically, described by $|X_k| \leq 2^{-\lfloor\log_2 k\rfloor} |X|$.
\begin{restatable}{corollary}{regretps}
    \label{cor:regret-pseudo-x-k}
    The regret bound for the $|X_k| \leq 2^{-\lfloor\log_2 k\rfloor} |X|$ would be $R_K\leq\sqrt{C_1\beta_K\gamma_K\pi^2/6}$ with a probability of at least $1 - \delta'$.
\end{restatable}
\jhedit{Proofs for Corollaries~\ref{cor:regret-x-k-logk} and \ref{cor:regret-pseudo-x-k} are provided in Appendix~\ref{appsec:proof-regret-logk} and \ref{appsec:proof-regret-ps}, respectively. Based on our experiments presented in Section~\ref{sec:experiments}, the rate of elimination of the largest segment for MBTL is shown in Figure~\ref{fig:max_space}.}

\section{Experiments and analysis}\label{sec:experiments}
\vspace{-5pt}
\subsection{Setup}\label{sec:setup}
\vspace{-5pt}
Our experiments \cwwwinline{consider CMDPs that} span \cwwwinline{standard and real-world benchmarks}. \cwwwinline{In particular, we consider standard continuous control benchmarks from the CARL library \cite{benjamins_contextualize_2023}. In addition, we study problems from RL for intelligent transportation systems, using \cite{yan_unified_2022} to model the CMDPs. Surprisingly, despite the relatively low complexity of the CMDPs considered, standard deep RL algorithms appear to struggle to solve the tasks.}

\textbf{Baselines.} \cwwwinline{We consider two types of baselines when evaluating our proposed algorithm: canonical and multi-policy. The canonical baselines are selected to validate multi-policy training; the multi-policy training baselines are heuristic methods designed to validate the Bayesian optimization approach. The canonical baselines include}: (1) \textbf{Independent training}, which involves \jhedit{independently} training \jhedit{separate models} on \cwwwinline{each} task; and (2) \textbf{Multi-task RL}, \jhedit{where a single \cwwwinline{``universal''} context-conditioned \jhedit{policy} is trained for all tasks}\jhdelete{where we train a single \jhedit{policy} on all tasks simultaneously}. \cwwwinline{The multi-policy baselines include:} (3) \textbf{Random selection}, where \jhedit{each}\jhdelete{the next} training task is chosen \cwwwinline{uniformly at} randomly; 
\jhedit{(\jhhedit{4}) \textbf{Greedy strategy}, \cwwwinline{which greedily selects the next source task based on Assumption~\ref{assume:linear-generalization}} and fixed training performance}; and
\jhedit{(\jhhedit{5})} \textbf{Sequential Oracle transfer}, which \cwwwinline{has access to generalized performance corresponding to policies for all tasks (including those not yet selected) and uses that information to greedily select the best source task}. 

\textbf{Proposed method.} \jhhhedit{In early iterations of BO, GP lacks sufficient observations (often just one or two) and thus relies heavily on its prior. To mitigate this, we incorporate a brief warm-up phase in \textbf{MBTL}. Specifically, we collect three additional data points using a simple greedy approach that selects tasks with the worst observed performance so far (inspired by \cite{cho_temporal_2023}). After this initialization, the method switches to full BO for source-task selection.} \cwwwinline{We \jhhhedit{use the scaling factor of} \jhhhedit{$\beta_k=2\log(|X|k^2)$}.}

\textbf{DRL algorithms and performance measure.} We utilize Deep Q-Networks (DQN) for discrete \jhedit{action} \cwwwinline{spaces}~\cite{mnih_human-level_2015} and \jhhhedit{Trust Region Policy Optimization (TRPO)~\cite{schulman_trust_2015} or Proximal Policy Optimization (PPO)~\cite{schulman_proximal_2017}} for continuous action \cwwwinline{spaces}. 
\jhhedit{For statistical reliability, we run each experiment three times with different random seeds.}
We evaluate our method by the average performance across all $N$ target tasks after training \cwwwinline{up to $K=15$} source tasks \cwwwinline{or} the number of \cwwwinline{source tasks} needed to achieve a certain level of \cwwwinline{performance}. 
We employ min-max normalization of the rewards for each task\jhedit{, and we provide comprehensive details about our model in Appendix~\ref{appsec:detail-gp}.}

\vspace{-5pt}
\subsection{Traffic \cwwwinline{benchmark experiments}}\label{sec:traffic}
\vspace{-5pt}
\begin{table}[t]
  \caption{Comparative performance of different methods on traffic \jhedit{CMDP}s \jhhhedit{($K=15$)}}
  \label{tab:performance-traffic}
  \Large
  \resizebox{0.99\textwidth}{!}{
  \begin{tabular}{@{}ccccccccc@{}}
    \toprule
    \multicolumn{2}{c}{\textbf{Benchmark (CMDP)}}&\multicolumn{2}{c}{\textbf{Baselines}}&\multicolumn{2}{c}{\textbf{\jhhedit{Multi-policy Baselines}}}&\multicolumn{1}{c}{\textbf{MBTL}}&\multicolumn{1}{c}{\textbf{Oracle}}\\
    \midrule
    \midrule
    \textbf{Domain}     & \textbf{Context Variation}      & \textbf{Independent} & \textbf{Multi-task} & \textbf{Random}  & \textbf{Greedy} & \textbf{Ours} & \textbf{Sequential}\\
    \midrule
    \multicolumn{2}{c}{\textbf{Number of Trained Models}}&$N$&$1$&$K$&$K$&$K$&$N$\\
    \midrule
    \textbf{Traffic Signal} & Road Length & \textbf{0.9409} & 0.8242 & 0.9366 & 0.9349 & \textbf{0.9409} & 0.9432 \\
    \textbf{Traffic Signal} & Inflow & 0.8646 & 0.8319 & 0.8699 & 0.8682 & \textbf{0.8729} & 0.8773 \\
    \textbf{Traffic Signal} & Speed Limit & 0.8857 & 0.6083 & \textbf{0.8872} & \textbf{0.8874} & 0.8866 & 0.8877 \\
    \midrule
    \textbf{Eco-Driving} & Penetration Rate & 0.5260 & 0.1945 & 0.6212 & 0.5992 & \textbf{0.6519} & 0.6668 \\
    \textbf{Eco-Driving} & Inflow & 0.4061 & 0.2229 & 0.5077 & \textbf{0.5299} & \textbf{0.5356} & 0.5531 \\
    \textbf{Eco-Driving} & Green Phase & 0.3850 & 0.4228 & 0.4724 & 0.4678 & \textbf{0.4932} & 0.5058 \\
    \midrule
    \textbf{AA-Ring-Acc} & Hold Duration & 0.8362 & \textbf{0.9219} & \textbf{0.9307} & 0.9021 & \textbf{0.9329} & 0.9567 \\
    \textbf{AA-Ring-Vel} & Hold Duration & 0.9589 & 0.9688 & \textbf{0.9820} & \textbf{0.9819} & \textbf{0.9820} & 0.9822 \\
    \textbf{AA-Ramp-Acc} & Hold Duration & 0.4276 & 0.5374 & \textbf{0.6599} & \textbf{0.6570} & \textbf{0.6282} & 0.7120 \\
    \textbf{AA-Ramp-Vel} & Hold Duration & 0.5473 & 0.5257 & \textbf{0.7210} & 0.6461 & \textbf{0.7426} & 0.7691 \\
    \midrule
    \multicolumn{2}{c}{\textbf{Average}} & 0.6778 & 0.6059 & 0.7589 & 0.7474 & 0.7667 & 0.7854 \\ \bottomrule
    \end{tabular}}
    \scriptsize{\textdagger Higher the better. Bold values represent the highest value(s) within the statistically significant range for each CMDP, excluding the oracle. Detailed results with variance for each method are provided in Appendix~\ref{appsec:table-full}.\\}
    \scriptsize{\textdaggerdbl AA: Advisory autonomy benchmark, Ring: Single lane ring, Ramp: Highway ramp, Acc: Acceleration guidance, Vel: Speed guidance.}
    \vspace{-15pt}
\end{table}

\cwwwinline{We consider three traffic \jhhedit{benchmarks}.} 
First, while most traffic lights \cwwwinline{still} operate on fixed schedules, \cwwwinline{RL can be used to} design adaptive (1) \textbf{Traffic signal control} to optimize traffic \cite{chu_multi-agent_2020, li_traffic_2016}. However, considering that every intersection \cwwwinline{is} different, challenges persist in generalizing across intersection configurations~\cite{jayawardana_impact_2022}. 
Given the significant portion of greenhouse gas emissions in the United States \cwwwinline{due to} transportation~\cite{us_epa_sources_2023}, the second traffic \jhedit{domain} is (2) \textbf{\cwwwinline{Dynamic eco-driving at signalized intersections}}, which \cwwwinline{concerns learning energy-efficient driving strategies in urban settings}. DRL-based eco-driving strategies have been developed \cite{guo_hybrid_2021, wegener_automated_2021,jayawardana_learning_2022} but still \cwwwinline{experience} difficulties in generalization. 
\cwwwinline{Our final} traffic \jhedit{domain} is (3) \textbf{Advisory autonomy}, \cwwwinline{in} which real-time speed \cwwwinline{or acceleration} advisories \cwwwinline{guide} human drivers to \cwwwinline{act as vehicle-based traffic controllers} in mixed \cwwwinline{traffic environments}~\cite{sridhar_piecewise_2021, cho_temporal_2023, hasan_cooperative_2024}. 
\cwwwinline{The context space $X$ is discretized into $N=\{50, 50, 40\}$ contexts for the three \jhedit{domains}, respectively.}
In Appendix~\ref{appsec:experiment}, we provide details about our experiments.

\begin{figure}[!ht]
    \centering
    \includegraphics[width=0.98\textwidth]{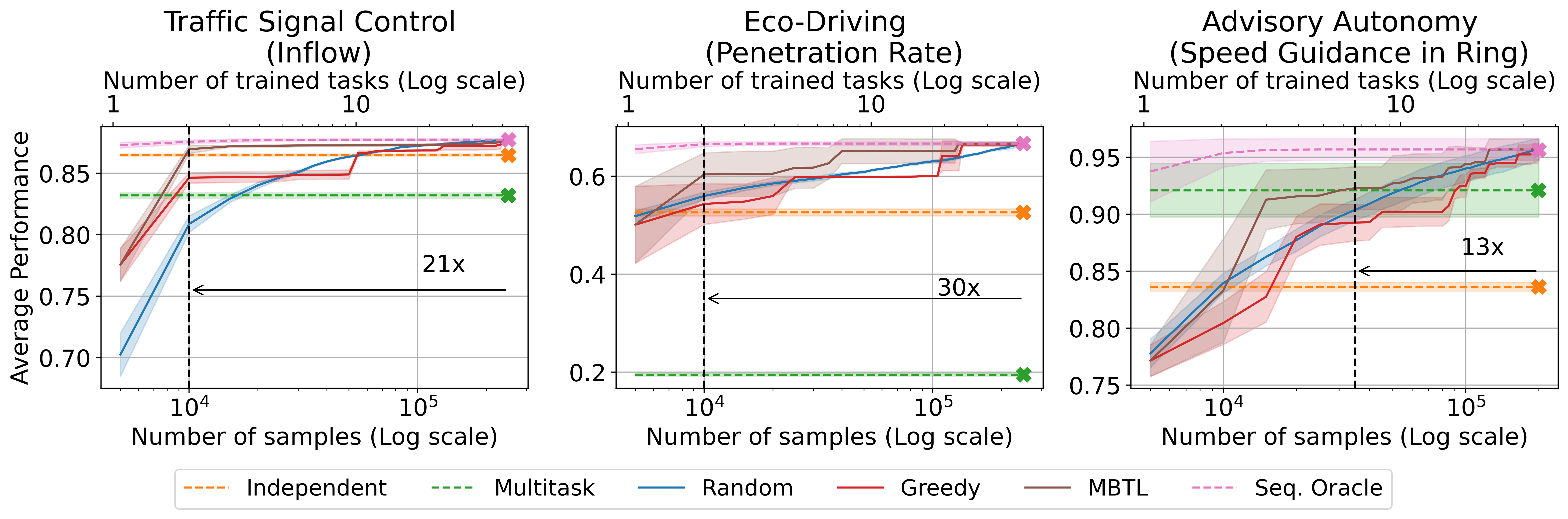}
    \caption{\cwwwinline{\textbf{Traffic CMDP results.} Method comparison of normalized performance over $N$ tasks. MBTL efficiently selects source training tasks. The black dotted line indicates the first training step within MBTL that exceeds both independent and multi-task baselines, \jhhedit{with up to 30x fewer samples.}}
    }
    \label{fig:result-main-traffic}
    \vspace{-15pt}
\end{figure}

\textbf{\cwwwinline{Results}.}
Table~\ref{tab:performance-traffic} and Figure~\ref{fig:result-main-traffic} \cwwwinline{summarize the results. Notably, the Oracle far outperforms the standard baselines (\jhhedit{independent and multi-task training}), indicating the potential for transfer learning and intelligent training of multiple models, respectively. MBTL rapidly approaches the Oracle within $\approx10$ transfer steps, indicating that the \jhhedit{GP} effectively models the training performance \jhedit{and linear generalization gap models the generalization performance}. It is important to note that multi-task RL studies commonly consider Independent training as a strong baseline due to its avoidance of negative transfer and other training instability issues. 
Indeed, Independent training often (but not always) outperforms multi-task training in our experiments. Yet, our experiments show that MBTL outperforms both Independent and Multi-task baselines and matches their performance with \textbf{\jhhedit{up to 30x} improved sample efficiency}. Among the \jhhedit{multi-policy baselines}, MBTL \jhhedit{often} outperforms the heuristic \jhhedit{multi-policy} baselines, indicating the value of adaptively selecting source tasks based on feedback. The \jhhedit{multi-policy baselines}, \jhhedit{such as random and greedy strategy}, also generally outperform Independent and Multi-task, indicating the general value of \cwwwinline{multi-policy} training for solving CMDPs.}
More results are provided in Appendix~\ref{appsec:experiment}.

\vspace{-5pt}
\subsection{\cwwwinline{Continuous} control benchmark \cwwwinline{experiments}}\label{sec:control}
\vspace{-5pt}

\cwwwinline{To probe the generality of MBTL,} we utilize context-extended versions of standard RL environments from CARL benchmark library \cite{benjamins_contextualize_2023} to \cwwwinline{evaluate} our methods under varied contexts.
For the Cartpole \jhedit{benchmark}, we \cwwwinline{considered} CMDPs with varying cart mass, pole length, and pole mass. In Pendulum, we vary the timestep duration, pendulum length, and pendulum mass. The BipedalWalker was tested under varying friction, gravity, and scale. In HalfCheetah \jhedit{domain}, we manipulated friction, gravity, and stiffness parameters. These variations critically influence the dynamics and physics of the environments. The range of context variations was selected by scaling the default values specified in CARL from 0.1 to 10 times \jhedit{($N=100$)}, enabling a comprehensive analysis of transfer learning under drastically different conditions. We provide more \jhhedit{experimental} details in Appendix~\ref{appsec:experiment}.

\begin{table}[t]
  \caption{Comparative performance of different methods on continuous control \jhedit{CMDP}s \jhhhedit{($K=15$)}}
  \label{tab:performance-control}
  \Large
  \resizebox{0.99\textwidth}{!}{
  \begin{tabular}{@{}ccccccccc@{}}
    \toprule
    \multicolumn{2}{c}{\textbf{Benchmark (CMDP)}}&\multicolumn{2}{c}{\textbf{Baselines}}&\multicolumn{2}{c}{\textbf{\jhhedit{Multi-policy Baselines}}}&\multicolumn{1}{c}{\textbf{MBTL}}&\multicolumn{1}{c}{\textbf{Oracle}}\\
    \midrule
    \midrule
    \textbf{Domain}     & \textbf{Context Variation}      & \textbf{Independent} & \textbf{Multi-task} & \textbf{Random}  & \textbf{Greedy} & \textbf{Ours} & \textbf{Sequential}\\
    \midrule
    \multicolumn{2}{c}{\textbf{Number of Trained Models}}&$N$&$1$&$K$&$K$&$K$&$N$\\
    \midrule
    \textbf{Pendulum} & Length & 0.7383 & 0.6830 & 0.7607 & \textbf{0.7774} & \textbf{0.7749} & 0.8073 \\
    \textbf{Pendulum} & Mass & 0.6237 & 0.5793 & 0.6647 & \textbf{0.6887} & \textbf{0.6933} & 0.7168 \\
    \textbf{Pendulum} & Timestep & 0.8135 & 0.7247 & \textbf{0.8331} & \textbf{0.8497} & \textbf{0.8310} & 0.8880 \\
    \midrule
    \textbf{Cartpole} & Mass of Cart & \textbf{0.9466} & 0.7153 & 0.8961 & 0.8299 & 0.9154 & 0.9998 \\
    \textbf{Cartpole} & Length of Pole & 0.9110 & 0.5441 & 0.9497 & 0.9424 & \textbf{0.9717} & 0.9995 \\
    \textbf{Cartpole} & Mass of Pole & 0.9560 & 0.6073 & 0.9870 & \textbf{0.9916} & \textbf{0.9941} & 1.0000 \\
    \midrule
    \textbf{BipedalWalker} & Gravity & 0.9281 & 0.7898 & 0.9654 & \textbf{0.9656} & \textbf{0.9669} & 0.9721 \\
    \textbf{BipedalWalker} & Friction & 0.9317 & 0.9051 & \textbf{0.9739} & \textbf{0.9738} & 0.9714 & 0.9779 \\
    \textbf{BipedalWalker} & Scale & 0.8694 & 0.7452 & \textbf{0.8910} & \textbf{0.8990} & \textbf{0.8864} & 0.9155 \\
    \midrule
    \textbf{HalfCheetah} & Gravity & 0.6679 & 0.6292 & 0.9086 & 0.9089 & \textbf{0.9308} & 0.9544 \\
    \textbf{HalfCheetah} & Friction & 0.6693 & 0.7242 & \textbf{0.9314} & \textbf{0.9184} & \textbf{0.9404} & 0.9663 \\
    \textbf{HalfCheetah} & Stiffness & 0.6561 & 0.7007 & \textbf{0.9191} & \textbf{0.9295} & \textbf{0.9214} & 0.9677 \\
    \midrule
    \multicolumn{2}{c}{\textbf{Average}} & 0.8093 & 0.6957 & 0.8901 & 0.8896 & 0.8998 & 0.9304 \\ \bottomrule
    \end{tabular}}
    \scriptsize{\textdagger Higher the better. Bold values represent the highest value(s) within the statistically significant range for each CMDP, excluding the oracle. Detailed results with variance for each method are provided in Appendix~\ref{appsec:table-full}.\\}
    \vspace{-15pt}
\end{table}

\begin{figure}[!ht]
    \centering
    \includegraphics[width=0.98\textwidth]{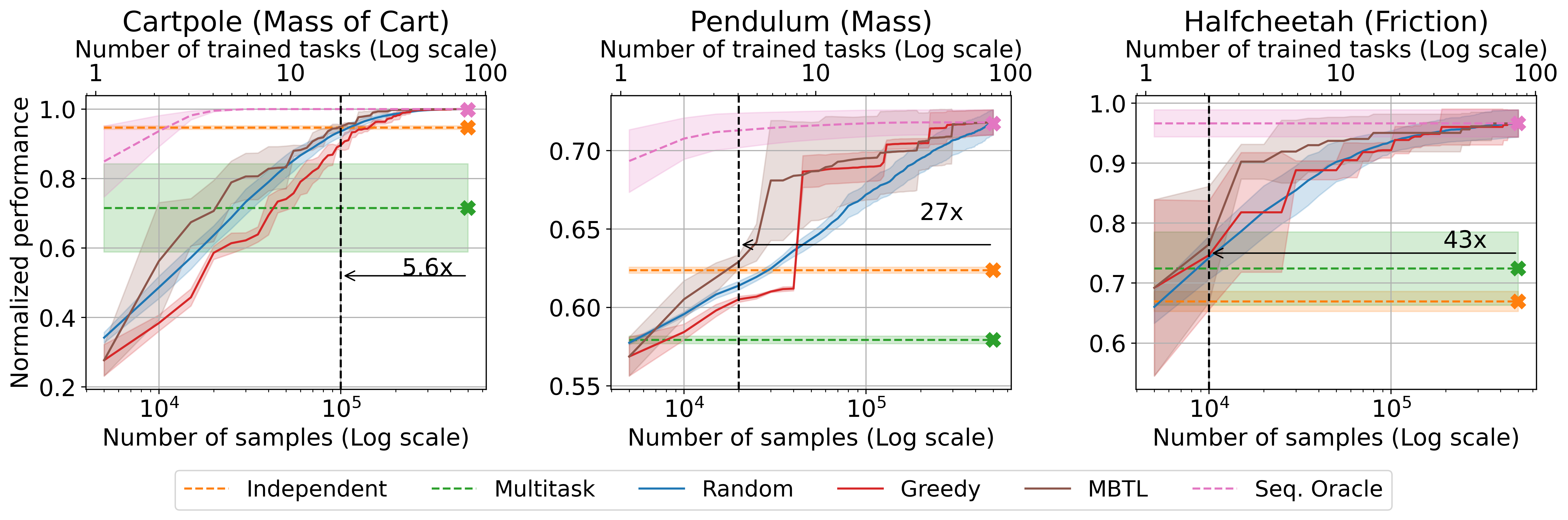}
    \caption{\cwwwinline{\textbf{Continuous control CMDP results.} Method comparison of normalized performance over $N$ tasks. The black dotted line indicates the first training step within MBTL that exceeds both independent and multi-task baselines, \jhhedit{with up to 43x improved sample efficiency.}}}
    \label{fig:result-main-control}
    \vspace{-10pt}
\end{figure}

\begin{figure}[!t]
    \centering
    \includegraphics[width=0.999\textwidth]{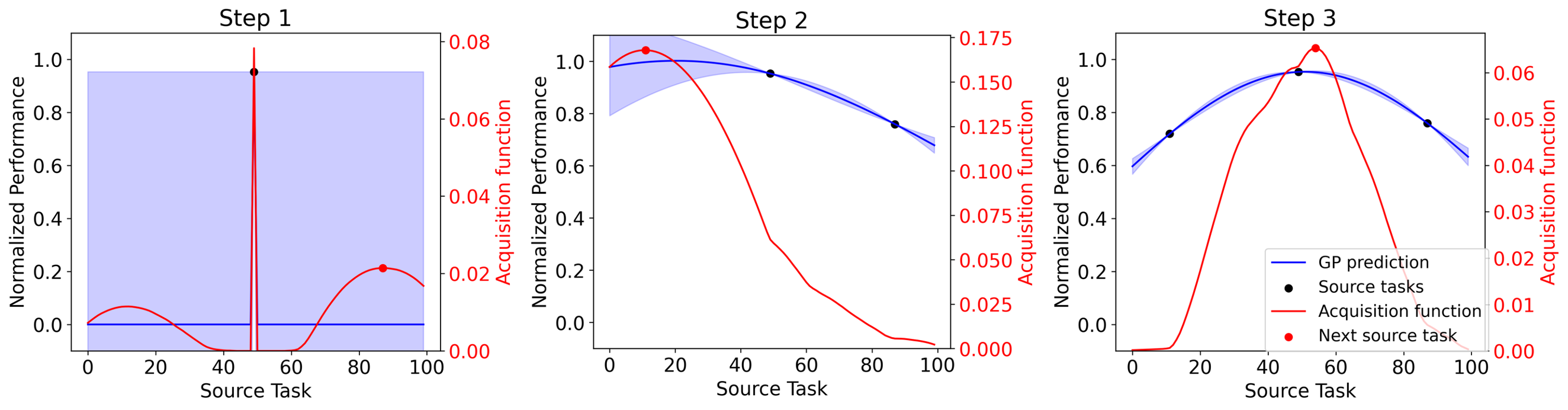}
    \caption{\jhedit{The GP sequentially updates \cwwwinline{estimates of the performance function} (blue) based on \cwwwinline{previously} trained models. \cwwwinline{Then}, MBTL selects the next source task that maximizes the acquisition function (red). (CMDP: Pendulum (Time step))}.}
    \label{fig:gp-pendulum}
    \vspace{-10pt}
\end{figure}

\textbf{\cwwwinline{Results}.} 
\cwwwinline{The results summarized in} Table~\ref{tab:performance-control} \jhedit{demonstrate sample efficiency and competitive performance \jhhedit{of multi-policy training including MBTL} across diverse control domains, often closely trailing the Oracle only with a small number of \jhhedit{trained policies}. \cwwwinline{\jhhedit{Figure~\ref{fig:result-main-control} shows that w}ith the exception of \jhhedit{a few context variations}, MBTL generally shows superior performance.}} 
\jhedit{Specifically, Figure~\ref{fig:gp-pendulum} illustrates the detailed process of how MBTL utilizes GP for performance estimation and chooses the next source task that maximizes the acquisition function that evaluates the expected improvement of generalized performance.} 
\jhedit{MBTL achieves comparable performance to multi-task or independent baselines with \jhhhedit{up to \textbf{43x} fewer samples}, highlighting its effectiveness in reducing training requirements.}

\begin{figure}[!t]
    \centering
     \includegraphics[width=0.98\linewidth]{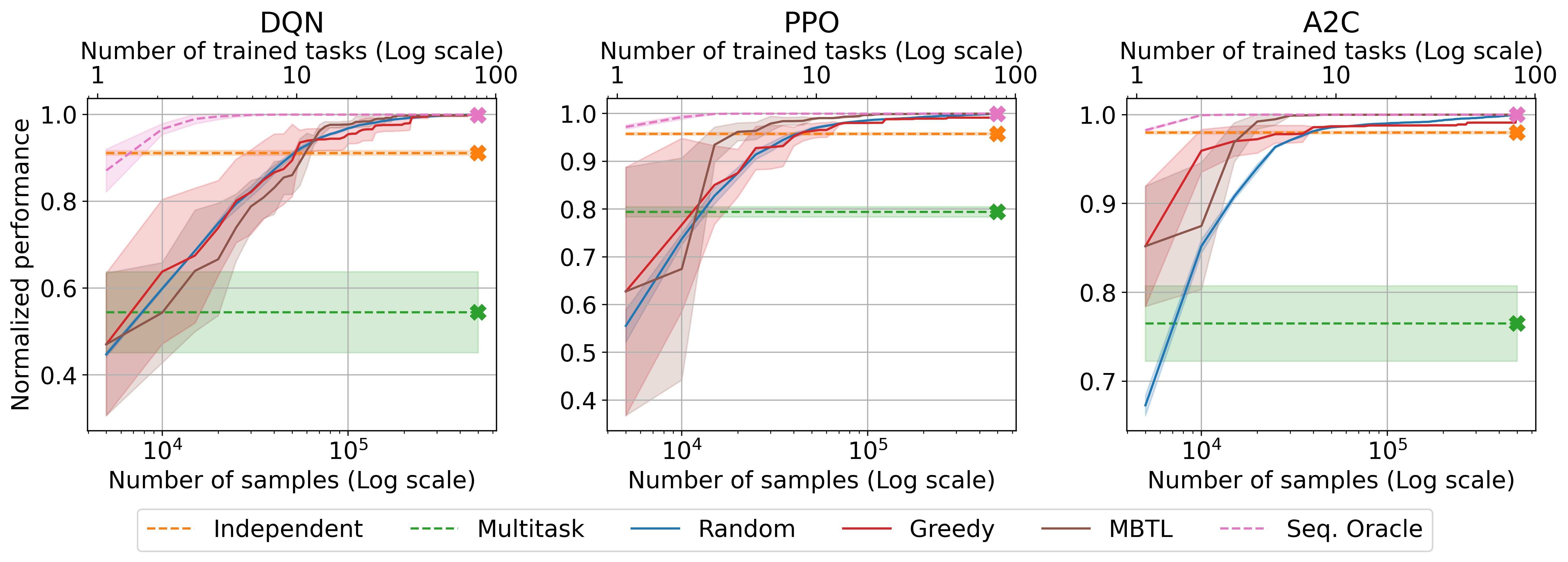}
     \captionof{figure}{\jhedit{Sensitivity analysis on the DRL algorithm \cwwwinline{underlying MBTL} (DQN, PPO, and A2C), tested on Cartpole with \cwwwinline{varying} length of pole}. MBTL \cwwwinline{remains effective}.}
     \vspace{-10pt}
    \label{fig:result-ablation-alg}
\end{figure}

\begin{wrapfigure}[12]{r}{.53\textwidth}
    \centering
    \vspace{-25pt}
     \includegraphics[width=\linewidth]{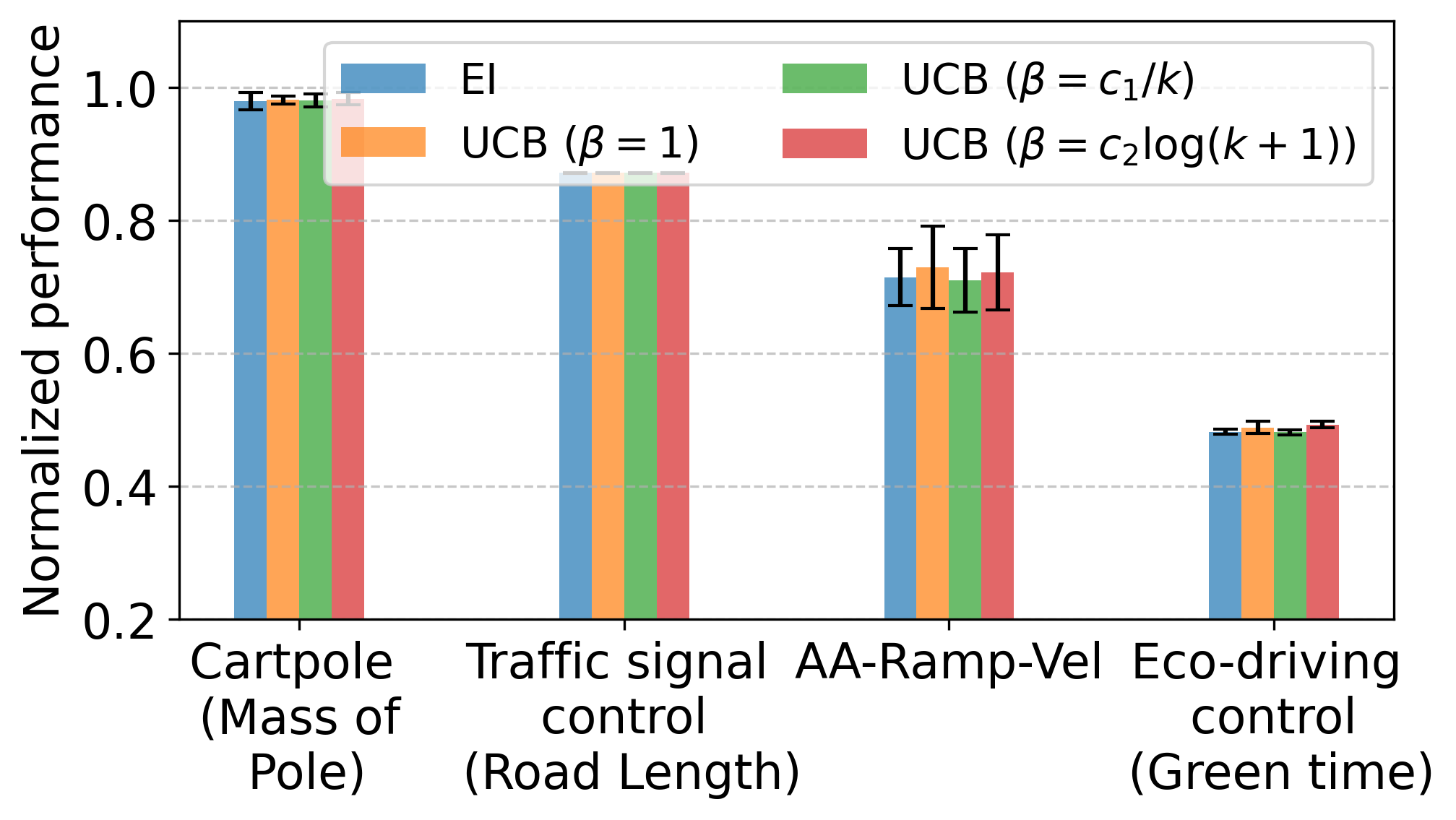}
     \captionof{figure}{\jhedit{Sensitivity analysis} on acquisition functions.}
     \vspace{-20pt}
    \label{fig:result-ablation-gp}
\end{wrapfigure}

\subsubsection{\jhedit{Sensitivity analysis}}

\textbf{DRL algorithms.}
Figure~\ref{fig:result-ablation-alg} shows \jhedit{that MBTL \cwwwinline{remains effective with different underlying} DRL algorithms—DQN, PPO, and Advantage Actor-Critic (A2C)~\cite{mnih_asynchronous_2016}—used for single-task training.}

\textbf{Acquisition functions.} Figure~\ref{fig:result-ablation-gp} assesses the role of acquisition functions in Bayesian optimization. 
While expected improvement (EI) focuses on promising marginal gains beyond the current best, UCB utilizes both mean and variance for balancing exploration and exploitation. \cwwwinline{Overall, we find that MBTL is not particularly sensitive to the choice of optimism representation in the acquisition function, which indicates that MBTL has a weak dependence on hyperparameters.}

\section{Related work}\label{sec:lit} 
\jhedit{\textbf{Contextual Reinforcement Learning.}} Robustness and generalization challenges in DRL are generally addressed by a few common techniques in the literature. The broader umbrella of such methods falls under CRL\jhedit{, which} utilizes side information about the problem variations to improve the generalization and robustness. In particular, CRL formalizes generalization in DRL using \jhedit{CMDP}s \cite{hallak_contextual_2015, modi_markov_2018, benjamins_contextualize_2023}, which incorporate context-dependent dynamics, rewards, and initial state distributions into the formalism of MDPs. 
The contexts of CMDPs are not always visible during training~\cite{kirk_survey_2023}. When they are visible, they can be directly used as side information by conditioning the policy on them \cite{sodhani2021multi}. 
\jhhedit{In this paper, we focus on a scenario where the learner can choose which context-MDP to train on. This contrasts with other CRL works that assume context-MDPs arrive from a fixed distribution.}

\jhhedit{\textbf{Multi-task training.}} 
\jhhedit{Multi-task methods can help address CRL by exploiting shared structure across tasks. Prior work has leveraged techniques such as policy sketches for task decomposition \cite{andreas_modular_2017} and distilled policies that capture common behaviors \cite{teh_distral_2017}. However, a key limitation arises when the context is unobserved, effectively transforming the CMDP into a partially observable setting \cite{kirk_survey_2023, cobbe2020leveraging}, which complicates multi-task training. Another challenge is negative transfer, wherein training on tasks that are too dissimilar leads to instability or outright failure \cite{kang_learning_2011, standley_which_2020, sun_adashare_2020}. Although more recent multi-task approaches such as MOORE \cite{hendawy_multi_2023} and PaCo \cite{sun_paco_2022} have shown promise, they often focus on discrete task sets and are thus less suited to CRL, where tasks span a broad continuum of contexts. In this work, we include multi-task learning as a baseline to benchmark our methods.}

\jhedit{\textbf{Zero-shot transfer and policy reuse.}} 
\jhhedit{Zero-shot transfer—where models trained for one environment directly perform in new, unseen settings without additional training \cite{kirk_survey_2023}—is an important strategy in CRL settings.}
\jhhedit{For solving CMDP problems, prior works attempted to utilize zero-shot transfer to solve CMDP problems by approximation on RL algorithm and hypernetworks that maps from parameterized \jhedit{CMDP} to a family of near-optimal solutions \cite{rezaei-shoshtari_hypernetworks_2023}.}
\citet{sinapov_learning_2015} use meta-data to learn inter-task transferability to learn the expected benefit of transfer given a source-target task pair. \citet{bao_information-theoretic_2019} propose a metric for evaluating transferability based on information-theoretic feature representations across tasks.
\jhhedit{Taken together, these approaches highlight the importance of policy reuse, where efficiently selecting or adjusting a pre-trained policy accelerates learning and improves robustness in new contexts.}

\jhedit{\textbf{Source task selection.}} \jhedit{In the context of transfer learning, selecting appropriate source tasks is crucial. \citet{li_optimal_2018} proposes an optimal online method for dynamically selecting the most relevant single source policy in reinforcement learning.}
Beyond RL, \citet{meiseles_source_2020} emphasizes structural alignment in time-series source models to prevent performance degradation, while \citet{poth_pre_2021} finds that selecting aligned intermediate tasks in natural language processing boosts transfer effectiveness. \jhhedit{Building upon these insights, we formulate the source task selection problem for CRL, enabling zero-shot transfer by estimating training performance online and leveraging structural generalization across context variations.}

\section{Conclusion}\label{sec:concl}
\cwwwinline{This study introduces a method called Model-based Transfer Learning (MBTL), which layers on top of existing RL methods to effectively solve CMDPs.
Rather than independent or multi-task training, which trains $N$ or 1 models, respectively, MBTL intelligently selects an intermediate number of models to train.
MBTL has two key components: an explicit model of the generalization gap and \jhedit{a Gaussian process} component to estimate training performance. MBTL achieves \jhhedit{up to 43x} improved sample efficiency on standard and real-world benchmarks. Furthermore, MBTL achieves sublinear regret in the number of training tasks.}
A \textbf{limitation} is that MBTL \cwwwinline{is designed for a single-dimensional} context variation \jhhedit{with a reliance on the explicit similarity of context variables}. \cwwwinline{Promising directions of future work include studying high-dimensional context spaces and \jhhedit{formalizing task similarity}, as well as the development of new real-world CMDP benchmarks}.

\begin{ack}
The authors acknowledge the MIT SuperCloud and Lincoln Laboratory Supercomputing Center for providing HPC resources that have contributed to the research results reported within this paper. This work was supported by the National Science Foundation (NSF) CAREER award (\#2239566), the Kwanjeong Educational Foundation Ph.D. scholarship program, and an Amazon Robotics Ph.D. Fellowship. The authors would like to thank the anonymous reviewers for their valuable feedback.
\end{ack}

% \section*{References}

{
\small
\bibliography{reference}
\medskip
}
%%%%%%%%%%%%%%%%%%%%%%%%%%%%%%%%%%%%%%%%%%%%%%%%%%%%%%%%%%%%

\newpage
\appendix
\section{Appendix}
\localtableofcontents

\clearpage
\subsection{Notation}\label{appsec:notation}
Table~\ref{tab:notation} describes the notation used in this paper.

\begin{table}[ht]
\centering
\caption{Notation used in the problem formulation}
\label{tab:notation}
\begin{tabular}{cc}
\hline
\textbf{Symbol} & \textbf{Description}                                \\ \hline
$x$             & Source task ($x \in X$)                             \\
$x'$            & Target task ($x' \in X$)                            \\
$\jhedit{\pi_x}$             & \jhedit{Trained policy from source task ($x \in X$)}                \\
$x_{k}$       & Selected source task at transfer step $k$ ($k=1,...,K$) \\
$\mathcal{M}_x$ & Contextual MDP parameterized by $x$                 \\
$J(\jhedit{\pi_x,}x)$          & Performance of task $\mathcal{M}_x$               \\
$\jhedit{J(\pi_x,x')}$ & Generalization performance (source: $x$ (or $\mathbf{x}$), target: $x'$) \\
$\Delta {J}(\jhedit{\pi_x,}x')$ & Generalization gap (source: $x$, target: $x'$) \\
$V(x';\jhedit{\pi_x})$          & Expected generalization performance of source model $x$ evaluated on all $x' \in X$ \\ \hline
\end{tabular}
\end{table}

Figure~\ref{fig:illust-for-u} helps understand the discrepancy between the observed generalized performance and the predicted one.
\Cref{fig:illust-for-V_pred} illustrates how to calculate the marginal improvement of expected generalized performance (\jhhedit{$\hat{V}(x;\pi_{1:k-1}) - V(x_{1:k-1})$}).
\begin{figure}[!ht]
    \centering
    \begin{tikzpicture}[domain=0:5]
        \draw[very thin,color=gray] (-0.1,-0.1) grid (4.9,4.1);
        
        \draw[->] (-0.2,0) -- (5.2,0) node[below=0.4, pos=0.5] {Target task $x'$};
        \draw[->] (0,-0.2) -- (0,5.2) node[left] {$U$};
        \draw[->] (5,-0.2) -- (5,5.2) node[] {};

        \draw[black,dotted,thick] (2.5,4) -- (5,4) node[right] {$J(\jhedit{\pi_{x_1},}x_1)$};
        \draw[black,dotted] (2.5,0) node[below] {$x_1$} -- (2.5,4.9);
        \draw[color=black]    plot (\x,{3*sin(2*\x/3.14 r)+1})             node[right] {$J(\jhedit{\pi_{x_1},x'})$};
        \draw[black,dashed,thick] (0,2) -- (2.5,4);
        \draw[black,dashed,thick] (2.5,4) -- (5,2) node[right] {$\hat{J}\jhedit{(\pi_{x_1},x')}$};
    \end{tikzpicture}
    \caption{Illustration of the discrepancy between observed ($J$) and predicted ($\hat{J}$) generalized performance after training on source task $x_1$ and attempting zero-shot transfer to $x'$.}
    \label{fig:illust-for-u}
\end{figure}

\begin{figure}
    \centering
    \begin{tikzpicture}[domain=0:5]
        \draw[very thin,color=gray] (-0.1,-0.1) grid (4.9,4.2);
        
        \draw[->] (-0.2,0) -- (5.2,0) node[below=0.4, pos=0.5] {Target task $x'$};
        \draw[->] (0,-0.2) -- (0,5.2) node[left] {$U$};
        \draw[->] (5,-0.2) -- (5,5.2) node[] {};

        \draw[black,dotted,thick] (2.5,4) -- (5,4) node[right] {$J(\jhedit{\pi_{1},}x_1)$};
        \draw[black,dotted] (2.5,0) node[below] {$x_1$} -- (2.5,4);
        \draw[black,dotted,thick] (0.9,3.5) -- (0,3.5) node[left] {$\hat{J}(\jhedit{\pi_x,}x)$};
        \draw[black,dotted] (0.9,0) node[below] {$x$} -- (0.9,3.5);
        \fill[pink!30] (0,1) -- (0,2) -- (0.9,3.5) -- (1.15,3) -- cycle;
        \draw[color=black]    plot (\x,{3*sin(2*\x/3.14 r)+1})             node[right] {$J(\jhedit{\pi_{1},x'})$};
        \draw[black,dashed,thick] (0,2) node[left] {$\hat{J}\jhedit{(\pi_{x},x')}$} -- (0.9,3.5);
        \draw[black,dashed,thick] (0.9,3.5) -- (3,0);
        \draw[red, ultra thick, domain=1.2:5, samples=100] plot (\x, {3*sin(2*\x/3.14 r)+1});
        \draw[red, ultra thick]    (0,2) -- (0.9,3.5);
        \draw[red, ultra thick]    (0.9,3.5) -- (1.2,3);
        \draw[black] (4.1,2.6) -- (4.7, 3) node[right] {$\max \left(\hat{J}\jhedit{(\pi_{x},x')},J(\jhedit{\pi_{1},x'}) \right)$};
        \draw[black] (0.5,2.1) -- (-0, 2.7) node[left] {\jhhedit{$\hat{V}(x;\pi_1) - V(x_1)$}};
    \end{tikzpicture}
    \caption{Step for choosing $x_2$ that maximizes the estimated marginal improvement ($\hat{V}(x;\pi_1) - V(x_1)$). $\hat{V}(x;\pi_1)$ corresponds to the red area under the red line and $V(x_1)$ as the area under $J(\jhedit{\pi_{1},x'})$.}
    \label{fig:illust-for-V_pred}
\end{figure}

\subsection{Model-Based Transfer Learning (MBTL) Algorithm}\label{appsec:alg}

\begin{tcolorbox}[colback=white!10!white, colframe=white!20!black, title=Model-based Transfer Learning (MBTL)] 

\begin{algorithmic}[1]
\State \textbf{Input:} \jhedit{CMDP}s $\mathcal{M}_x$, Task (context) set $X$, Training budget $K$
\State \textit{Initialize} : $J, V=0\ \forall x\in X$, $\fpi=\{\}$, $k=1$
\While{$k\leq K$}
    \State \textcolor{purple}{\% Estimate training performance}
    \State $\mu, \sigma \gets \mathcal{GP}(\mathbb{E}[\jhedit{{J}}(\jhedit{\pi_x,}x)], k(x, \jhedit{\tilde{x}})))$
    \State \textcolor{purple}{\% Calculate marginal generalized performance and acquisition function}
    \State Calculate $a(x;x_{1:k-1})$ with Eq.~\ref{eqn:acquisition}
    \State \textcolor{purple}{\% Select the next training task}
    \State $x_k=\arg\max_x a(x;x_{1:k-1})$
    \State $\pi_{k} \gets \textbf{Train}(\mathcal{M}_{x_k})$
    \State $\fpi \gets \fpi\cup\{\pi_{k}\}$
    \State $k \gets k+1$
\EndWhile
\State Zero-shot transfer and calculate generalization performance $V(x_1,...,x_k)$
\State \textbf{Output:} Set of \cwwwinline{policies} $\fpi$ and generalization performance $V$
\end{algorithmic}
\label{alg:mbtl}
\end{tcolorbox}

\subsection{Theoretical analysis}\label{appsec:theory}
\jhhhedit{This section provides detailed proofs of the regret bounds introduced in Theorem~\ref{theorem:regret-delta-beta-log-x-k}, Corollary~\ref{cor:regret-x-k-logk}, and Corollary~\ref{cor:regret-pseudo-x-k} from the main text. Our analysis adapts key results from \cite{srinivas_information-theoretic_2012} to settings where the search space is restricted at each iteration.}

\subsubsection{Proof of Theorem~\ref{theorem:regret-delta-beta-log-x-k}}\label{appsec:proof-log-x-k}
\regretdeltaxk*
\begin{proof}
    \jhhhedit{We begin by introducing two lemmas (Lemmas~\ref{lemma:regret_x_k} and \ref{lemma:regret-k}) that extend the results in \cite{srinivas_information-theoretic_2012} to handle the restricted search space $X_k \subseteq X$ at each iteration.}
    \begin{lemma}
        \label{lemma:regret_x_k}
        For $t\geq 1$, if $|f(x)-\mu_{k-1}(x)|\leq \beta_k^{1/2}\sigma_{k-1}(x) \quad \forall x \in X_k$, then the regret $r_t$ is bounded by $2|X_k|\beta_k^{1/2}\sigma_{k-1}(x)/|X|$.
    \end{lemma}
    
    \begin{lemma}
        \label{lemma:regret-k}
        Setting $\delta' \in (0,1)$, $\beta_k=2 \log(|X| \pi^2 k^2 / 6\delta')$, and $C_1:=\frac{8}{\log(1+\sigma^{-2})}\geq8\sigma^2$, we have $Pr\left[\sum_{k=1}^K r_k\left(\frac{|X|}{|X_k|}\right)^2\leq C_1 \beta_K \gamma_K \quad \forall K\geq 1\right]\geq1-\delta'$.
    \end{lemma}
    
    \jhhhedit{Using Lemma~\ref{lemma:regret_x_k}, we can bound each $r_t$ in terms of the restricted search space $X_k$. Then, applying Lemma~5.3 from \cite{srinivas_information-theoretic_2012} (which controls the deviation of the GP-UCB estimator) together with Lemma~\ref{lemma:regret-k} (which sums these instantaneous regrets under restricted search spaces), and finally invoking the Cauchy--Schwarz inequality, we derive a bound on the cumulative regret. Specifically, with probability at least $1 - \delta'$, we have:}
    \begin{equation}
        R_K=\sum_{k=1}^K r_k\leq\sqrt{\sum_{k=1}^Kr_k\left(\frac{|X|}{|X_k|}\right)^2 \sum_{k=1}^K\left(\frac{|X_k|}{|X|}\right)^2}\leq\sqrt{C_1\beta_K\gamma_K\sum_{k=1}^K\left(\frac{|X_k|}{|X|}\right)^2}.
    \end{equation}
\end{proof}

\subsubsection{Proof of Corollary~\ref{cor:regret-x-k-logk}}\label{appsec:proof-regret-logk}
\regretlogk*
\begin{proof}
    Recall that $\sum_{k=1}^K\frac{1}{k}\leq \log K$.
    
    Calculating the sum of squares for the reduced segments, we have:
    \begin{equation}
        \sum_{k=1}^K|X_k|^2=\sum_{k=1}^K\frac{1}{k}|X|^2\leq|X|^2\log K
    \end{equation}
    The cumulative regret can be bounded as below:
    \begin{equation}
        R_K=\sum_{k=1}^K r_k\leq\sqrt{C_1\beta_K\gamma_K\sum_{k=1}^K\left(\frac{|X_k|}{|X|}\right)^2}\leq\sqrt{C_1\beta_K\gamma_K\log K}.
    \end{equation}
\end{proof}

\subsubsection{Proof of Corollary~\ref{cor:regret-pseudo-x-k}}\label{appsec:proof-regret-ps}
\regretps*
\begin{proof}
    Calculating the sum of squares for the reduced segments, we have:
    \begin{equation}
        \sum_{k=1}^K|X_k|^2=\sum_{k=1}^K2^{-2\lfloor\log_2k\rfloor}|X|^2\leq\sum_{k=1}^K\frac{1}{k^2}|X|^2\leq\frac{\pi^2}{6}|X|^2
    \end{equation}
    The cumulative regret can be bounded as below:
    \begin{equation}
        R_K=\sum_{k=1}^K r_k\leq\sqrt{C_1\beta_K\gamma_K\sum_{k=1}^K\left(\frac{|X_k|}{|X|}\right)^2}\leq\sqrt{\frac{C_1\beta_K\gamma_K\pi^2}{6}}.
    \end{equation}
\end{proof}

\clearpage
\subsection{Experiment details}\label{appsec:experiment}
\subsubsection{Details about Gaussian process (GP) Regresstion}\label{appsec:detail-gp}
\jhhedit{We use the \texttt{GaussianProcessRegressor} implementation from \texttt{scikit-learn}, which follows Algorithm~2.1 of \cite{williams_gaussian_2006}. Specifically, we construct a kernel by multiplying a constant kernel 
\[
C(\theta) = 1.0, \quad \theta \in (10^{-3}, 10^{3}),
\]
by a radial basis function (RBF) kernel 
\[
k_{\mathrm{RBF}}(\mathbf{x},\mathbf{x}'; \ell) = \exp\Bigl(-\frac{\|\mathbf{x} - \mathbf{x}'\|^2}{2\ell^2}\Bigr)
\]
with an initial length scale $\ell = 1.0$ (constrained to lie in the range $[10^{-2}, 10^2]$).
To determine the hyperparameters, we begin by generating synthetic data that aligns with our modeling assumptions, including constant training performance and a linear generalization gap, while introducing noise to degrade generalization performance by up to 10\%, sampled from a uniform distribution. We vary the GP hyperparameters, including noise standard deviation over the set $\{0.001, 0.01, 0.1, 1\}$, the number of restarts for the optimizer over $\{5, 6, \dots, 15\}$, and explore several kernel configurations on the synthetic data. We then select the hyperparameter configuration that maximizes the average predictive performance. Specifically, we choose a noise standard deviation of $\sigma = 0.001$ and perform $15$ random restarts of the hyperparameter optimizer to reduce the risk of convergence to poor local minima.
We use the same GP hyperparameter configuration across all experiments and benchmarks.
}

\jhdelete{In our study, we conducted hyperparameter tuning experiments for Gaussian Process (GP) regression to optimize its performance. Specifically, we varied the noise standard deviation over the set $\{0.001, 0.01, 0.1, 1\}$ and the number of restarts for the optimizer over the set $\{5, 6, 7, 8, 9, 10, 11, 12, 13, 14, 15\}$ with several kernel configurations. We set noise standard deviation as $0.1$, the number of restarts for the optimizer over the set as $9$, and the kernel as the combination of a constant kernel (C) with the bound from 0.001 to 1000 and a radial basis function kernel (RBF) with a length scale ranging from 0.01 to 100. The noise level was set to the square of the noise standard deviation ($0.1^2$). }

\subsubsection{\jhhedit{Accuracy of generalization gap assumption}}
\jhhedit{In Figure~\ref{fig:gen_gap_corr}, we report the Pearson correlation between the observed generalization gap and the estimated gap under our linear assumption (Assumption~\ref{assume:linear-generalization}). Each histogram shows how strongly the two measures align across various tasks in standard control (blue) and traffic (red) benchmarks. Many tasks cluster around moderate positive correlations (0.3–0.5), suggesting that a linear function of context similarity can reasonably capture the gap in most scenarios. However, certain tasks—such as Eco-driving—exhibit higher correlations (above 0.6), whereas others—such as HalfCheetah—are closer to 0, indicating that the assumption holds more effectively in some domains than in others.}
\begin{figure}[H]
    \centering
    \includegraphics[width=0.7\textwidth]{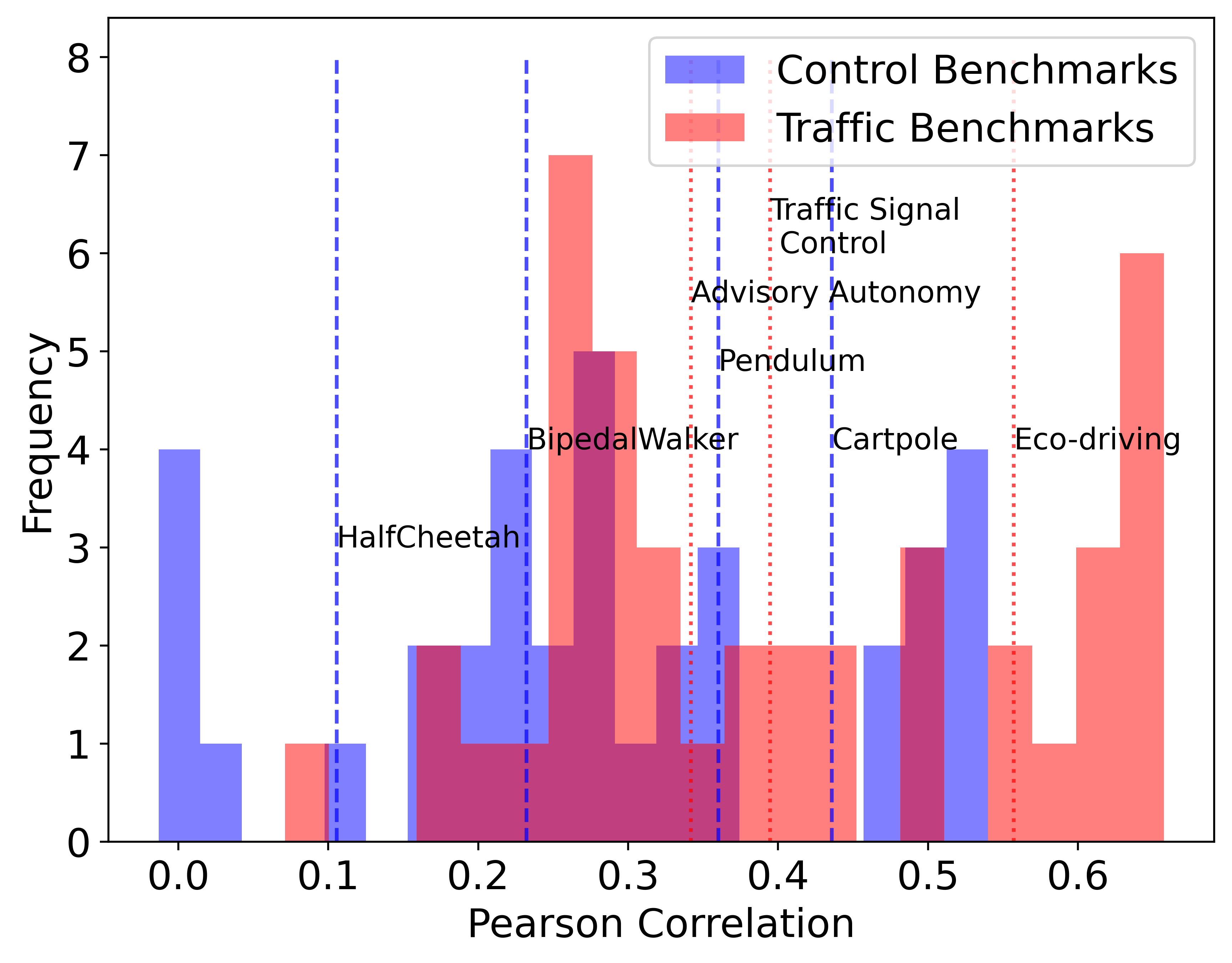}
    \caption{\jhhedit{\textbf{Accuracy of linear generalization gap assumption.} Pearson correlation analysis on the observed generalization gap and the estimated gap under our linear assumption.}}
    \label{fig:gen_gap_corr}
\end{figure}

\clearpage
\subsubsection{Results of table with standard deviation}\label{appsec:table-full}
\begin{table}[!ht]
  \caption{Comparative performance of different methods on context-variant traffic and control CMDPs (\jhhhedit{$K=15$})}
  \label{tab:performance}
  \Large
  \resizebox{0.999\textwidth}{!}{
  \begin{tabular}{ccccccccc}
    \toprule
    \multicolumn{2}{c}{\textbf{Benchmark (CMDP)}} & \multicolumn{2}{c}{\textbf{Baselines}} & \multicolumn{2}{c}{\textbf{Multi-policy Baselines}} & \multicolumn{1}{c}{\textbf{MBTL}} & \multicolumn{1}{c}{\textbf{Oracle}}\\
    \midrule
    \midrule
    \textbf{Domain} & \textbf{Context Variation} & \textbf{Independent} & \textbf{Multi-task} & \textbf{Random} & \textbf{Greedy} & \textbf{Ours} & \textbf{Sequential}\\
    \midrule

    \textbf{Traffic Signal} & Road Length & \begin{tabular}[c]{@{}c@{}}\textbf{0.9409}\\(0.0002)\end{tabular} & \begin{tabular}[c]{@{}c@{}}0.8242\\(0.0659)\end{tabular} & \begin{tabular}[c]{@{}c@{}}0.9366\\(0.0009)\end{tabular} & \begin{tabular}[c]{@{}c@{}}0.9349\\(0.0021)\end{tabular} & \begin{tabular}[c]{@{}c@{}}\textbf{0.9409}\\(0.0005)\end{tabular} & \begin{tabular}[c]{@{}c@{}}0.9432\\(0.0001)\end{tabular} \\
    \textbf{Traffic Signal} & Inflow & \begin{tabular}[c]{@{}c@{}}0.8646\\(0.0009)\end{tabular} & \begin{tabular}[c]{@{}c@{}}0.8319\\(0.0049)\end{tabular} & \begin{tabular}[c]{@{}c@{}}0.8699\\(0.0011)\end{tabular} & \begin{tabular}[c]{@{}c@{}}0.8682\\(0.0008)\end{tabular} & \begin{tabular}[c]{@{}c@{}}\textbf{0.8729}\\(0.0010)\end{tabular} & \begin{tabular}[c]{@{}c@{}}0.8773\\(0.0009)\end{tabular} \\
    \textbf{Traffic Signal} & Speed Limit & \begin{tabular}[c]{@{}c@{}}0.8857\\(0.0005)\end{tabular} & \begin{tabular}[c]{@{}c@{}}0.6083\\(0.0493)\end{tabular} & \begin{tabular}[c]{@{}c@{}}\textbf{0.8872}\\(0.0002)\end{tabular} & \begin{tabular}[c]{@{}c@{}}\textbf{0.8874}\\(0.0004)\end{tabular} & \begin{tabular}[c]{@{}c@{}}0.8866\\(0.0003)\end{tabular} & \begin{tabular}[c]{@{}c@{}}0.8877\\(0.0003)\end{tabular} \\
    \midrule
    \textbf{Eco-Driving} & Penetration Rate & \begin{tabular}[c]{@{}c@{}}0.5260\\(0.0087)\end{tabular} & \begin{tabular}[c]{@{}c@{}}0.1945\\(0.0070)\end{tabular} & \begin{tabular}[c]{@{}c@{}}0.6212\\(0.0041)\end{tabular} & \begin{tabular}[c]{@{}c@{}}0.5992\\(0.0007)\end{tabular} & \begin{tabular}[c]{@{}c@{}}\textbf{0.6519}\\(0.0301)\end{tabular} & \begin{tabular}[c]{@{}c@{}}0.6668\\(0.0046)\end{tabular} \\
    \textbf{Eco-Driving} & Inflow & \begin{tabular}[c]{@{}c@{}}0.4061\\(0.0094)\end{tabular} & \begin{tabular}[c]{@{}c@{}}0.2229\\(0.0012)\end{tabular} & \begin{tabular}[c]{@{}c@{}}0.5077\\(0.0114)\end{tabular} & \begin{tabular}[c]{@{}c@{}}\textbf{0.5299}\\(0.0456)\end{tabular} & \begin{tabular}[c]{@{}c@{}}\textbf{0.5356}\\(0.0125)\end{tabular} & \begin{tabular}[c]{@{}c@{}}0.5531\\(0.0095)\end{tabular} \\
    \textbf{Eco-Driving} & Green Phase & \begin{tabular}[c]{@{}c@{}}0.3850\\(0.0063)\end{tabular} & \begin{tabular}[c]{@{}c@{}}0.4228\\(0.0225)\end{tabular} & \begin{tabular}[c]{@{}c@{}}0.4724\\(0.0069)\end{tabular} & \begin{tabular}[c]{@{}c@{}}0.4678\\(0.0147)\end{tabular} & \begin{tabular}[c]{@{}c@{}}\textbf{0.4932}\\(0.0164)\end{tabular} & \begin{tabular}[c]{@{}c@{}}0.5058\\(0.0047)\end{tabular} \\
    \midrule
    \textbf{AA-Ring-Acc} & Hold Duration & \begin{tabular}[c]{@{}c@{}}0.8362\\(0.0048)\end{tabular} & \begin{tabular}[c]{@{}c@{}}\textbf{0.9219}\\(0.0381)\end{tabular} & \begin{tabular}[c]{@{}c@{}}\textbf{0.9307}\\(0.0118)\end{tabular} & \begin{tabular}[c]{@{}c@{}}0.9021\\(0.0154)\end{tabular} & \begin{tabular}[c]{@{}c@{}}\textbf{0.9329}\\(0.0250)\end{tabular} & \begin{tabular}[c]{@{}c@{}}0.9567\\(0.0116)\end{tabular} \\
    \midrule
    \textbf{AA-Ring-Vel} & Hold Duration & \begin{tabular}[c]{@{}c@{}}0.9589\\(0.0096)\end{tabular} & \begin{tabular}[c]{@{}c@{}}0.9688\\(0.0145)\end{tabular} & \begin{tabular}[c]{@{}c@{}}\textbf{0.9820}\\(0.0001)\end{tabular} & \begin{tabular}[c]{@{}c@{}}\textbf{0.9819}\\(0.0001)\end{tabular} & \begin{tabular}[c]{@{}c@{}}\textbf{0.9820}\\(0.0002)\end{tabular} & \begin{tabular}[c]{@{}c@{}}0.9822\\(0.0000)\end{tabular} \\
    \midrule
    \textbf{AA-Ramp-Acc} & Hold Duration & \begin{tabular}[c]{@{}c@{}}0.4276\\(0.0066)\end{tabular} & \begin{tabular}[c]{@{}c@{}}0.5374\\(0.1478)\end{tabular} & \begin{tabular}[c]{@{}c@{}}\textbf{0.6599}\\(0.0250)\end{tabular} & \begin{tabular}[c]{@{}c@{}}\textbf{0.6570}\\(0.0810)\end{tabular} & \begin{tabular}[c]{@{}c@{}}\textbf{0.6282}\\(0.0414)\end{tabular} & \begin{tabular}[c]{@{}c@{}}0.7120\\(0.0468)\end{tabular} \\
    \midrule
    \textbf{AA-Ramp-Vel} & Hold Duration & \begin{tabular}[c]{@{}c@{}}0.5473\\(0.0222)\end{tabular} & \begin{tabular}[c]{@{}c@{}}0.5257\\(0.0121)\end{tabular} & \begin{tabular}[c]{@{}c@{}}\textbf{0.7210}\\(0.0535)\end{tabular} & \begin{tabular}[c]{@{}c@{}}0.6461\\(0.0791)\end{tabular} & \begin{tabular}[c]{@{}c@{}}\textbf{0.7426}\\(0.0604)\end{tabular} & \begin{tabular}[c]{@{}c@{}}0.7691\\(0.0576)\end{tabular} \\
    \midrule
    \textbf{Pendulum} & Length & \begin{tabular}[c]{@{}c@{}}0.7383\\(0.0034)\end{tabular} & \begin{tabular}[c]{@{}c@{}}0.6830\\(0.0010)\end{tabular} & \begin{tabular}[c]{@{}c@{}}0.7607\\(0.0072)\end{tabular} & \begin{tabular}[c]{@{}c@{}}\textbf{0.7774}\\(0.0041)\end{tabular} & \begin{tabular}[c]{@{}c@{}}\textbf{0.7749}\\(0.0143)\end{tabular} & \begin{tabular}[c]{@{}c@{}}0.8073\\(0.0104)\end{tabular} \\
    \textbf{Pendulum} & Mass & \begin{tabular}[c]{@{}c@{}}0.6237\\(0.0023)\end{tabular} & \begin{tabular}[c]{@{}c@{}}0.5793\\(0.0051)\end{tabular} & \begin{tabular}[c]{@{}c@{}}0.6647\\(0.0065)\end{tabular} & \begin{tabular}[c]{@{}c@{}}\textbf{0.6887}\\(0.0116)\end{tabular} & \begin{tabular}[c]{@{}c@{}}\textbf{0.6933}\\(0.0346)\end{tabular} & \begin{tabular}[c]{@{}c@{}}0.7168\\(0.0107)\end{tabular} \\
    \textbf{Pendulum} & Timestep & \begin{tabular}[c]{@{}c@{}}0.8135\\(0.0103)\end{tabular} & \begin{tabular}[c]{@{}c@{}}0.7247\\(0.0597)\end{tabular} & \begin{tabular}[c]{@{}c@{}}\textbf{0.8331}\\(0.0084)\end{tabular} & \begin{tabular}[c]{@{}c@{}}\textbf{0.8497}\\(0.0322)\end{tabular} & \begin{tabular}[c]{@{}c@{}}\textbf{0.8310}\\(0.0238)\end{tabular} & \begin{tabular}[c]{@{}c@{}}0.8880\\(0.0199)\end{tabular} \\
    \midrule
    \textbf{Cartpole} & Mass of Cart & \begin{tabular}[c]{@{}c@{}}\textbf{0.9466}\\(0.0065)\end{tabular} & \begin{tabular}[c]{@{}c@{}}0.7153\\(0.2688)\end{tabular} & \begin{tabular}[c]{@{}c@{}}0.8961\\(0.0214)\end{tabular} & \begin{tabular}[c]{@{}c@{}}0.8299\\(0.0392)\end{tabular} & \begin{tabular}[c]{@{}c@{}}0.9154\\(0.0294)\end{tabular} & \begin{tabular}[c]{@{}c@{}}0.9998\\(0.0003)\end{tabular} \\
    \textbf{Cartpole} & Length of Pole & \begin{tabular}[c]{@{}c@{}}0.9110\\(0.0065)\end{tabular} & \begin{tabular}[c]{@{}c@{}}0.5441\\(0.1977)\end{tabular} & \begin{tabular}[c]{@{}c@{}}0.9497\\(0.0044)\end{tabular} & \begin{tabular}[c]{@{}c@{}}0.9424\\(0.0310)\end{tabular} & \begin{tabular}[c]{@{}c@{}}\textbf{0.9717}\\(0.0148)\end{tabular} & \begin{tabular}[c]{@{}c@{}}0.9995\\(0.0007)\end{tabular} \\
    \textbf{Cartpole} & Mass of Pole & \begin{tabular}[c]{@{}c@{}}0.9560\\(0.0128)\end{tabular} & \begin{tabular}[c]{@{}c@{}}0.6073\\(0.1161)\end{tabular} & \begin{tabular}[c]{@{}c@{}}0.9870\\(0.0050)\end{tabular} & \begin{tabular}[c]{@{}c@{}}\textbf{0.9916}\\(0.0030)\end{tabular} & \begin{tabular}[c]{@{}c@{}}\textbf{0.9941}\\(0.0083)\end{tabular} & \begin{tabular}[c]{@{}c@{}}1.0000\\(0.0000)\end{tabular} \\
    \midrule
    \textbf{BipedalWalker} & Gravity & \begin{tabular}[c]{@{}c@{}}0.9281\\(0.0034)\end{tabular} & \begin{tabular}[c]{@{}c@{}}0.7898\\(0.1136)\end{tabular} & \begin{tabular}[c]{@{}c@{}}0.9654\\(0.0004)\end{tabular} & \begin{tabular}[c]{@{}c@{}}\textbf{0.9656}\\(0.0021)\end{tabular} & \begin{tabular}[c]{@{}c@{}}\textbf{0.9669}\\(0.0011)\end{tabular} & \begin{tabular}[c]{@{}c@{}}0.9721\\(0.0008)\end{tabular} \\
    \textbf{BipedalWalker} & Friction & \begin{tabular}[c]{@{}c@{}}0.9317\\(0.0074)\end{tabular} & \begin{tabular}[c]{@{}c@{}}0.9051\\(0.0900)\end{tabular} & \begin{tabular}[c]{@{}c@{}}\textbf{0.9739}\\(0.0003)\end{tabular} & \begin{tabular}[c]{@{}c@{}}\textbf{0.9738}\\(0.0013)\end{tabular} & \begin{tabular}[c]{@{}c@{}}0.9714\\(0.0024)\end{tabular} & \begin{tabular}[c]{@{}c@{}}0.9779\\(0.0012)\end{tabular} \\
    \textbf{BipedalWalker} & Scale & \begin{tabular}[c]{@{}c@{}}0.8694\\(0.0087)\end{tabular} & \begin{tabular}[c]{@{}c@{}}0.7452\\(0.1148)\end{tabular} & \begin{tabular}[c]{@{}c@{}}\textbf{0.8910}\\(0.0079)\end{tabular} & \begin{tabular}[c]{@{}c@{}}\textbf{0.8990}\\(0.0135)\end{tabular} & \begin{tabular}[c]{@{}c@{}}\textbf{0.8864}\\(0.0159)\end{tabular} & \begin{tabular}[c]{@{}c@{}}0.9155\\(0.0023)\end{tabular} \\
    \midrule
    \textbf{HalfCheetah} & Gravity & \begin{tabular}[c]{@{}c@{}}0.6679\\(0.0162)\end{tabular} & \begin{tabular}[c]{@{}c@{}}0.6292\\(0.0317)\end{tabular} & \begin{tabular}[c]{@{}c@{}}0.9086\\(0.0078)\end{tabular} & \begin{tabular}[c]{@{}c@{}}0.9089\\(0.0235)\end{tabular} & \begin{tabular}[c]{@{}c@{}}\textbf{0.9308}\\(0.0203)\end{tabular} & \begin{tabular}[c]{@{}c@{}}0.9544\\(0.0221)\end{tabular} \\
    \textbf{HalfCheetah} & Friction & \begin{tabular}[c]{@{}c@{}}0.6693\\(0.0203)\end{tabular} & \begin{tabular}[c]{@{}c@{}}0.7242\\(0.1293)\end{tabular} & \begin{tabular}[c]{@{}c@{}}\textbf{0.9314}\\(0.0175)\end{tabular} & \begin{tabular}[c]{@{}c@{}}\textbf{0.9184}\\(0.0184)\end{tabular} & \begin{tabular}[c]{@{}c@{}}\textbf{0.9404}\\(0.0460)\end{tabular} & \begin{tabular}[c]{@{}c@{}}0.9663\\(0.0276)\end{tabular} \\
    \textbf{HalfCheetah} & Stiffness & \begin{tabular}[c]{@{}c@{}}0.6561\\(0.0101)\end{tabular} & \begin{tabular}[c]{@{}c@{}}0.7007\\(0.1379)\end{tabular} & \begin{tabular}[c]{@{}c@{}}\textbf{0.9191}\\(0.0100)\end{tabular} & \begin{tabular}[c]{@{}c@{}}\textbf{0.9295}\\(0.0169)\end{tabular} & \begin{tabular}[c]{@{}c@{}}\textbf{0.9214}\\(0.0164)\end{tabular} & \begin{tabular}[c]{@{}c@{}}0.9677\\(0.0287)\end{tabular} \\
    \midrule
    \end{tabular}}
    \scriptsize{* \textit{Note}: Bold values represent the highest value(s) within the statistically significant range for each task, excluding the oracle. Standard deviation across multiple runs in the parenthesis.\\}
    \scriptsize{\textdaggerdbl AA: Advisory autonomy tasks, Ring: Single lane ring, Ramp: Highway ramp, Acc: Acceleration guidance, Vel: Speed guidance.}
\end{table}

\clearpage

\subsubsection{Detailed sample complexity comparison results}\label{appsec:table-sample-efficiency}
\jhedit{Table~\ref{tab:sample-complexity} presents a comparison of sample complexity required for MBTL to perform as good as the best generalization performance of baselines (independent training and multi-task training) across various tasks in the \jhedit{CMDP}. Each row lists a different domain, the specific context variation applied (e.g., changes in physical properties or environmental parameters), and two key values: $k^*$ and $N$, where $k^*$ represents the number of models required by MBTL to reach a performance level comparable to the baseline. This value is shown as a range (e.g., $[3, 5, 3]$), indicating results from three random seeds.
$N$ represents the total number of contexts. The value $\frac{N}{k^*}$ helps represent the sample efficiency of MBTL.}

\begin{table}[!h]
  \caption{\jhhhedit{Sample complexity comparison to baseline performance on \jhedit{CMDP} tasks}}
  \label{tab:sample-complexity}
  \normalsize
  \centering
  \begin{tabular}{@{}cccccc@{}}
    \toprule
    \textbf{Task} & \textbf{Variation} & $k^*$ & $N$ & $k^*$ average & $N/k^*$ average \\ \midrule
    \textbf{Pendulum} & \textbf{Length} & [3, 5, 3] & 100 & 3.67 & 27.27 \\
    \textbf{Pendulum} & \textbf{Mass} & [4, 4, 3] & 100 & 3.67 & 27.27 \\
    \textbf{Pendulum} & \textbf{Timestep} & [8, 9, 10] & 100 & 9 & 11.11 \\ \midrule
    \textbf{Cartpole} & \textbf{Mass of Cart} & [18, 14, 22] & 100 & 18 & 5.56 \\
    \textbf{Cartpole} & \textbf{Length of Pole} & [13, 12, 10] & 100 & 11.67 & 8.57 \\
    \textbf{Cartpole} & \textbf{Mass of Pole} & [5, 5, 4] & 100 & 4.67 & 21.43 \\ \midrule
    \textbf{BipedalWalker} & \textbf{Gravity} & [3, 3, 9] & 100 & 5 & 20 \\
    \textbf{BipedalWalker} & \textbf{Friction} & [2, 4, 2] & 100 & 2.67 & 37.5 \\
    \textbf{BipedalWalker} & \textbf{Scale} & [3, 13, 1] & 100 & 5.67 & 17.65 \\ \midrule
    \textbf{HalfCheetah} & \textbf{Gravity} & [2, 2, 2] & 100 & 2 & 50 \\
    \textbf{HalfCheetah} & \textbf{Friction} & [1, 3, 3] & 100 & 2.33 & 42.86 \\
    \textbf{HalfCheetah} & \textbf{Stiffness} & [1, 3, 3] & 100 & 2.33 & 42.86 \\ \midrule
    \textbf{AA-Ring-Acc} & \textbf{Hold Duration} & [3, 3, 4] & 40 & 3.33 & 12 \\
    \textbf{AA-Ring-Vel} & \textbf{Hold Duration} & [1, 3, 5] & 40 & 3 & 13.33 \\
    \textbf{AA-Ramp-Acc} & \textbf{Hold Duration} & [32, 3, 4] & 40 & 13 & 3.08 \\
    \textbf{AA-Ramp-Vel} & \textbf{Hold Duration} & [3, 2, 2] & 40 & 2.33 & 17.14 \\ \midrule
    \textbf{Traffic Signal} & \textbf{Road Length} & [19, 15, 10] & 50 & 14.67 & 3.41 \\
    \textbf{Traffic Signal} & \textbf{Inflow} & [3, 2, 2] & 50 & 2.33 & 21.43 \\
    \textbf{Traffic Signal} & \textbf{Speed Limit} & [13, 8, 5] & 50 & 8.67 & 5.77 \\ \midrule
    \textbf{Eco-Driving} & \textbf{Penetration Rate} & [2, 2, 1] & 50 & 1.67 & 30 \\
    \textbf{Eco-Driving} & \textbf{Inflow} & [3, 1, 1] & 50 & 1.67 & 30 \\
    \textbf{Eco-Driving} & \textbf{Green Phase} & [2, 2, 3] & 50 & 2.33 & 21.43 \\ \midrule
    \bottomrule
    \end{tabular}
\end{table}

\subsubsection{Details about traffic signal control \jhhedit{benchmark}}\label{appsec:detail-traffic-signal}

Most traffic lights operate on fixed schedules, but adaptive traffic signal control using DRL can optimize the traffic flow using real-time information on the traffic \cite{chu_multi-agent_2020, li_traffic_2016}, though challenges persist in generalizing across various intersection configurations \cite{jayawardana_impact_2022}.

Figure~\ref{fig:traffic-signal-network} showcases the layout of traffic networks used in a traffic signal control task with several lanes and a signalized intersection in the middle.
The state space represents the presence of vehicles in discretized lane cells along the incoming roads. Actions determine which lane gets the green phase of the traffic signal, and rewards are based on changes in cumulative stopped time, the period when speed is zero. 
The global objective is to minimize the average waiting times at the intersection. Different configurations of intersections (e.g., road length, inflow, speed limits) are modeled to represent varying real-world conditions.
\jhhedit{We vary factors such as road length, inflow rate, and speed limits from 0.1 to 5 times; by default, the road length is 750 meters, the flow rate is 500 vehicles per hour, and the speed limit is 13.89 m/s.}
\begin{figure}[H]
    \centering
    \includegraphics[width=0.8\textwidth]{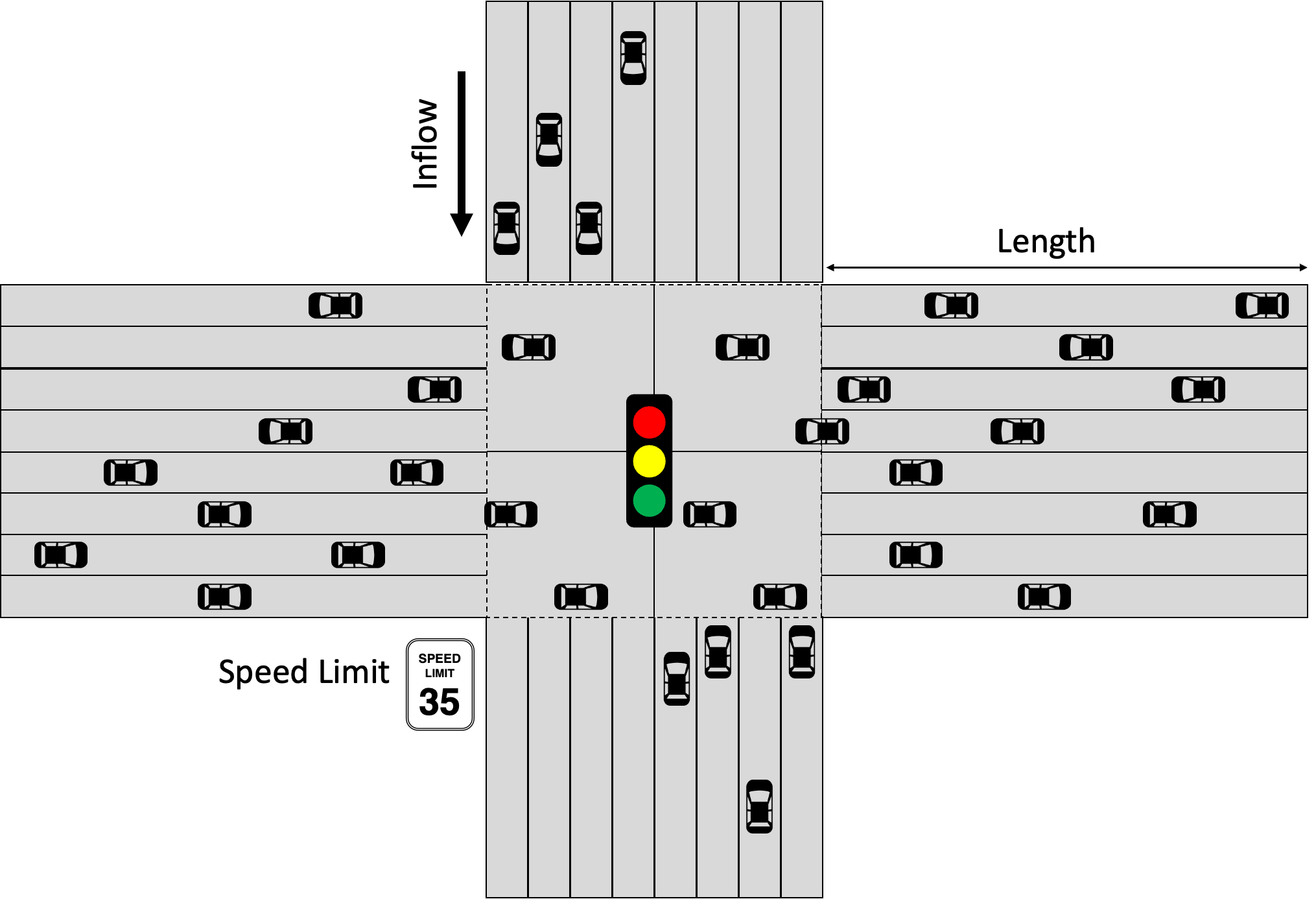}
    \caption{Illustration of the traffic networks in traffic signal control task.}
    \label{fig:traffic-signal-network}
\end{figure}

\paragraph{Training configuration}
We used the microscopic traffic simulation called Simulation of Urban MObility (SUMO) \cite{SUMO2018} v.1.16.0. \jhhhedit{For reinforcement learning, the Deep Q-Network (DQN) algorithm was employed with a neural network architecture comprising four hidden layers, each with 400 units. The learning rate was set to 0.001, and training was conducted over 800 epochs. The discount factor ($\gamma$) was configured at 0.75 to balance short-term and long-term rewards effectively.}
All experiments are done on a distributed computing cluster equipped with 48 Intel Xeon Platinum 8260 CPUs.
\paragraph{License} \jhhedit{Traffic signal control benchmark falls under MIT License.}

\paragraph{\jhhedit{Potential of multi-policy training and zero-shot transfer}}
\begin{figure}[H]
    \centering
    \includegraphics[width=0.99\textwidth]{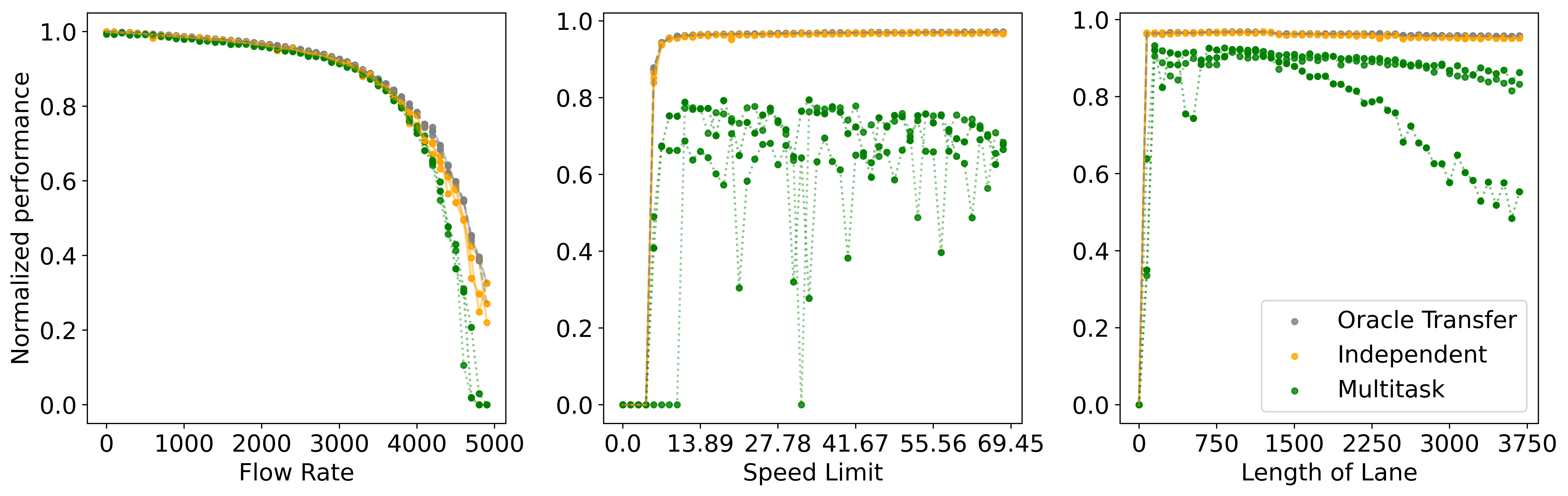}
    \caption{\jhhedit{Normalized performance of three DRL-based methods—Oracle Transfer (gray), independent training (orange), and multi-task training (green)—under different traffic-signal benchmarks: flow rate (left), speed limit (middle), and lane length (right). While independent and multi-task training approaches exhibit higher variance and reduced asymptotic performance, Oracle Transfer benefits from zero-shot transfer with full information, demonstrating more stable and generally higher performance.}}
    \label{fig:gap-traffic-signal}
\end{figure}
\jhhedit{Figure~\ref{fig:gap-traffic-signal} shows how each approach adapts to variations in flow rate, speed limit, and lane length for a traffic signal control benchmark. The y-axis shows normalized performance, with higher values indicating better control policies. Oracle Transfer consistently achieves superior performance across these different settings, owing to its ability to leverage full task information in a zero-shot transfer manner. By contrast, independent and multi-task training exhibit more pronounced performance drops and greater instability when faced with shifts in problem parameters, underscoring the challenges of generalizing policies in traditional DRL approaches.}

\paragraph{Transferability heatmap}

Figure~\ref{fig:heatmap-traffic-signal} presents heatmaps of transferability for different traffic signal control tasks, each varying a specific aspect: inflow, speed limit, and road length. The heatmaps display the effectiveness of strategy transfer from each source task (vertical axis) to each target task (horizontal axis). In terms of inflow variation, transferability drops when transferring from tasks with lower vehicle inflow to those with higher inflow. In speed limit variation, the transferability shows uniform effectiveness, suggesting less sensitivity to these changes. In road length variation, distinct blocks of high transferability indicate that different road lengths may require significantly tailored strategies.

\begin{figure}[H]
    \centering
    \begin{subfigure}[b]{0.32\textwidth}
        \includegraphics[width=\textwidth]{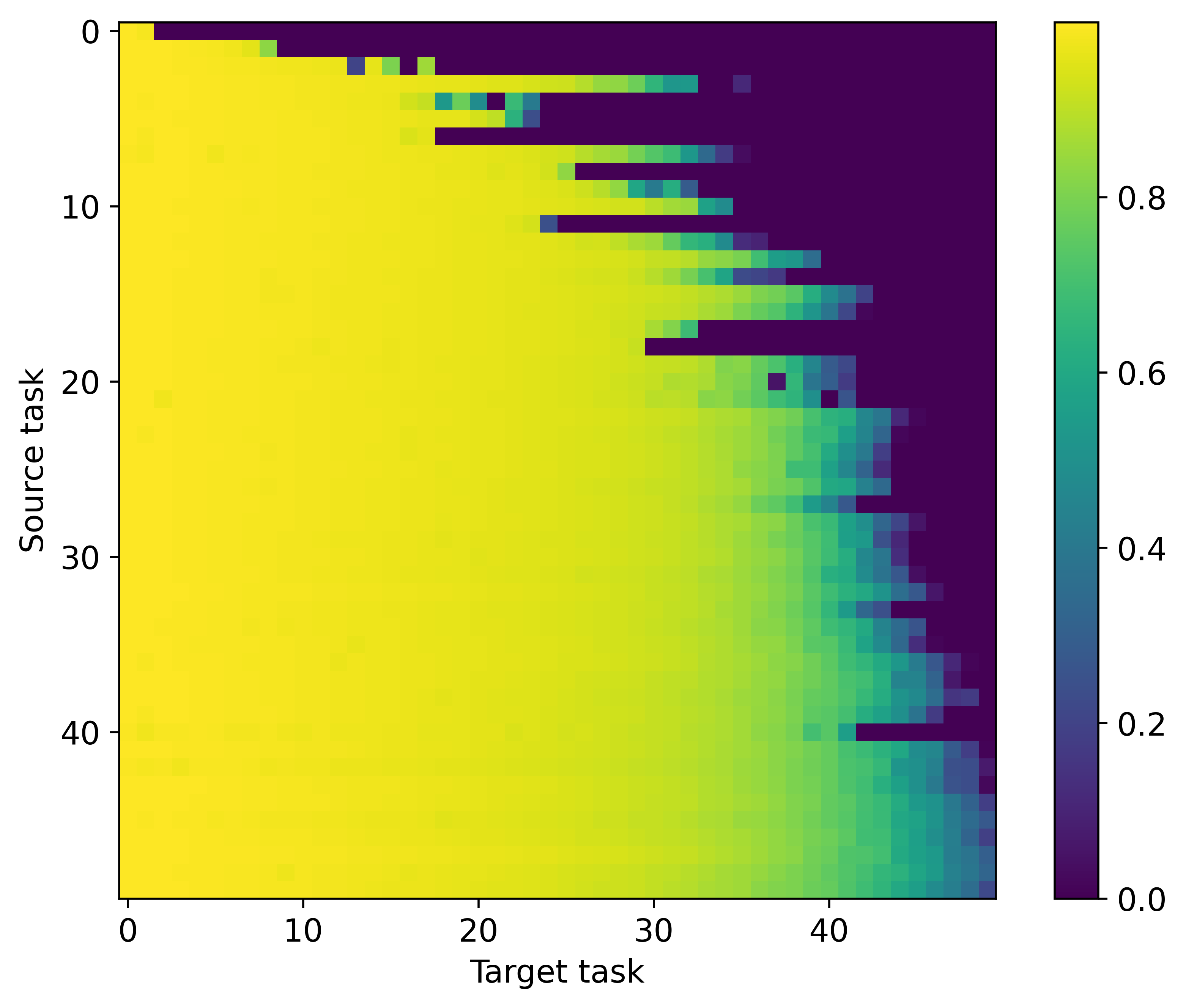}
        \caption{Inflow variation}
        \label{fig:heatmap-traffic-signal-flow}
    \end{subfigure}
    \hfill 
    \begin{subfigure}[b]{0.32\textwidth}
        \includegraphics[width=\textwidth]{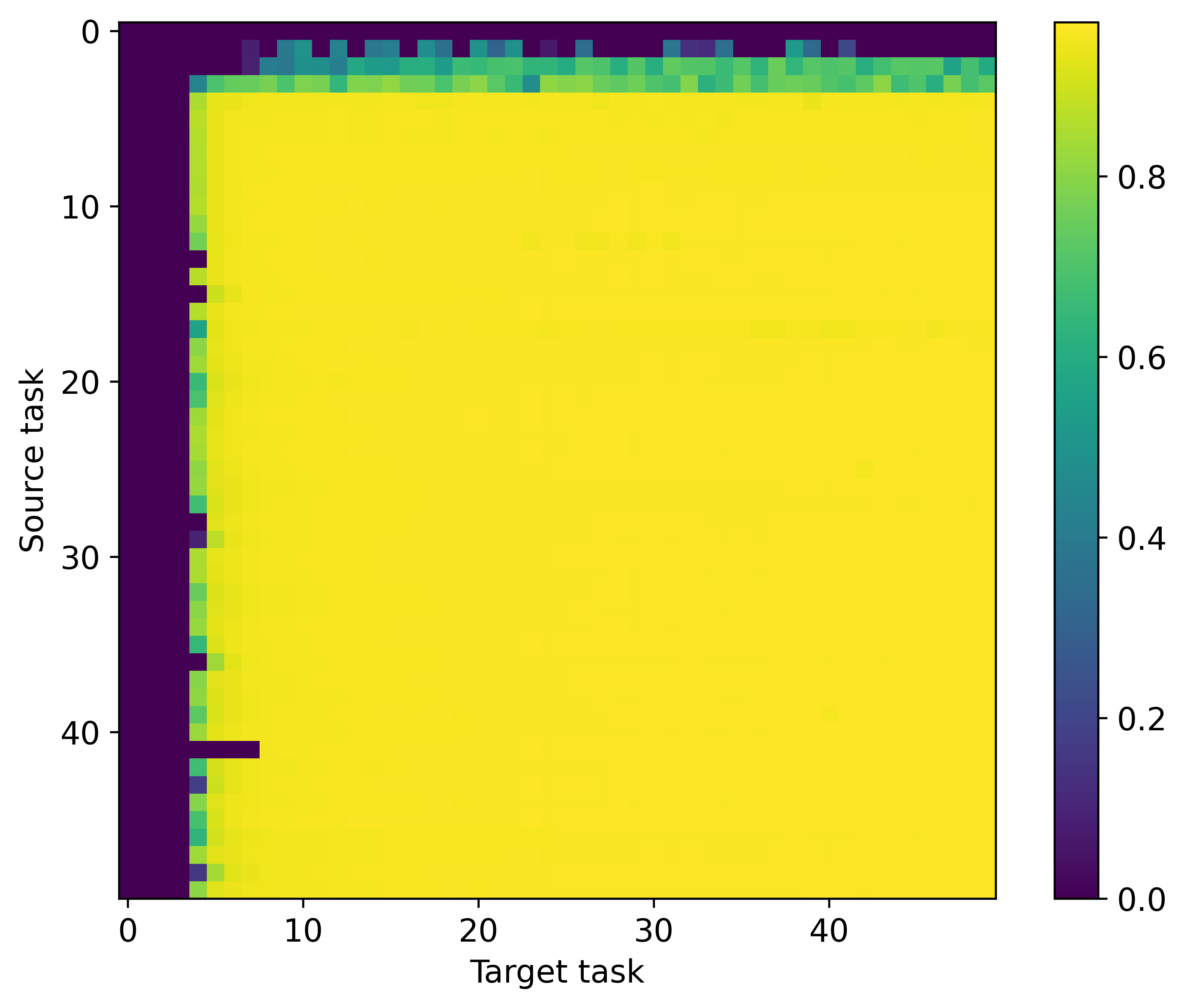}
        \caption{Speed limit variation}
        \label{fig:heatmap-traffic-signal-speed}
    \end{subfigure}
    \hfill 
    \begin{subfigure}[b]{0.32\textwidth}
        \includegraphics[width=\textwidth]{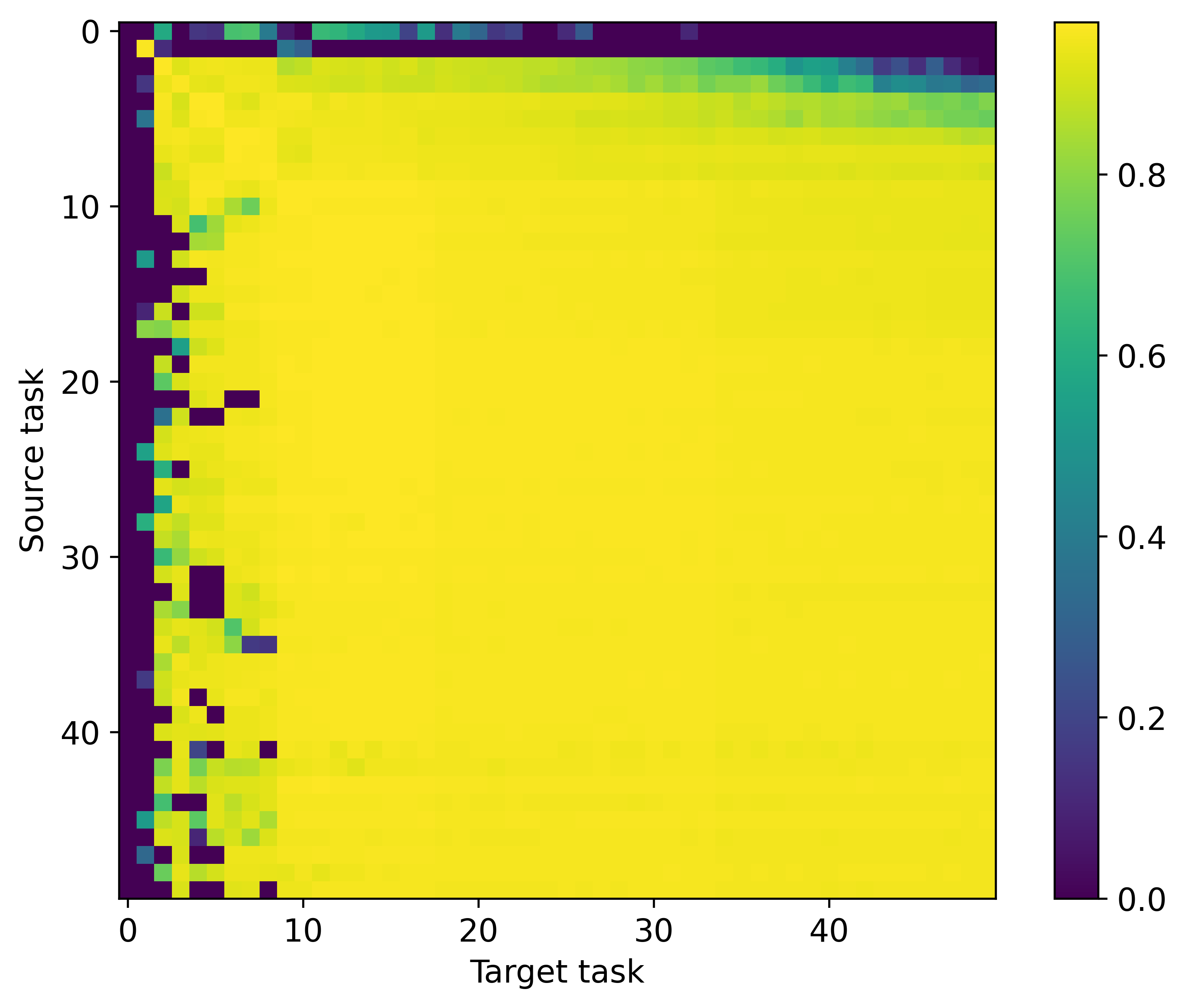}
        \caption{Road length variation}
        \label{fig:heatmap-traffic-signal-length}
    \end{subfigure}
    \caption{Examples of transferability heatmap for traffic signal control.}
    \label{fig:heatmap-traffic-signal}
\end{figure}

\paragraph{Results}

Figure~\ref{fig:result-traffic-signal} illustrates the normalized generalized performance across various traffic control tasks: inflow, speed limit, and road length. The plots display how different strategies adapt with increasing transfer steps:

\begin{itemize}
  \item \textbf{Inflow:} Performance improves as the number of transfer steps increases, with MBTL strategy consistently achieving the highest scores, demonstrating their effectiveness in adapting to changes in inflow conditions.
  \item \textbf{Speed Limit:} Here, performance levels are relatively stable across all strategies except for the multi-task training.
  \item \textbf{Road Length:} There is a general upward trend in performance for all strategies, particularly for MBTL, indicating robustness in adapting to different road lengths.
\end{itemize}

This data suggests that MBTL and Oracle are particularly effective across varying conditions, maintaining higher levels of performance adaptability.

\begin{figure}[H]
    \centering
    \includegraphics[width=0.999\textwidth]{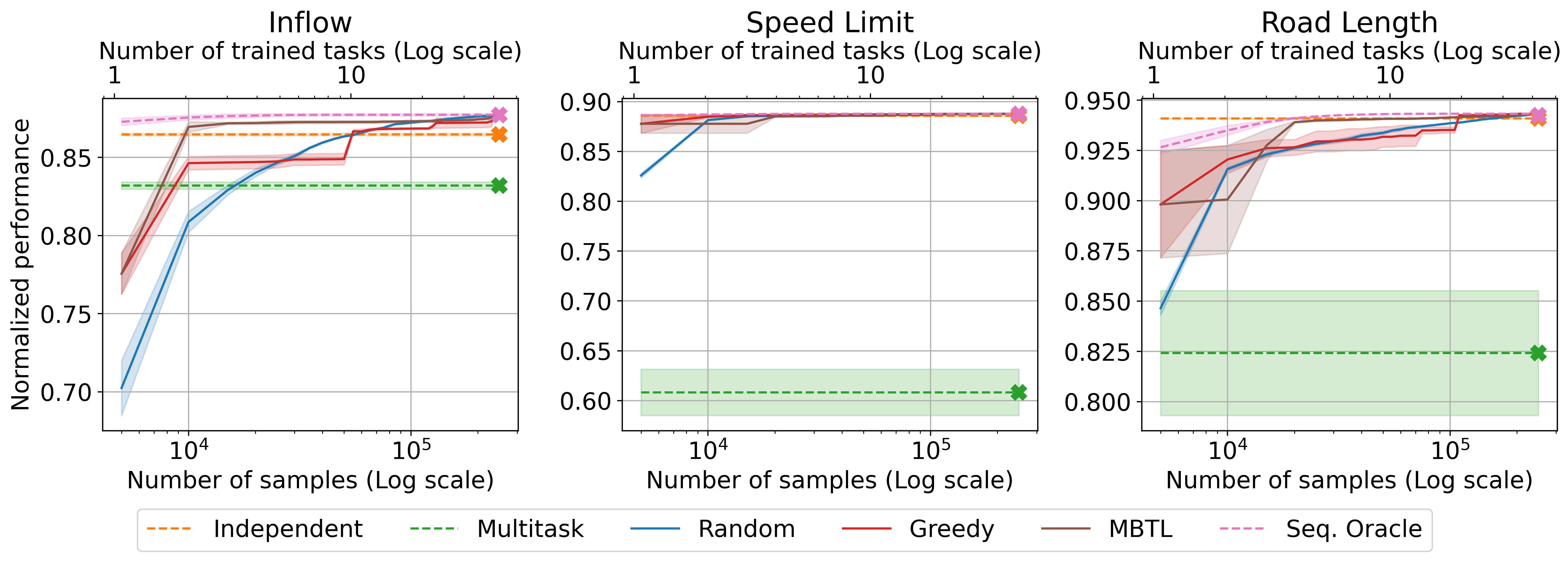}
    \caption{Comparison of normalized generalized performance of all target tasks: Traffic signal control.}
    \label{fig:result-traffic-signal}
\end{figure}

\clearpage
\subsubsection{Details about eco-driving control \jhhedit{benchmark}}\label{appsec:detail-eco-driving}
Given the significant portion of greenhouse gas emissions in the United States coming from the transportation sector \cite{us_epa_sources_2023}, eco-driving behaviors are critical for climate change mitigation. Deep reinforcement learning-based eco-driving strategies have been developed \cite{guo_hybrid_2021, wegener_automated_2021,jayawardana_learning_2022,jayawardana_mitigating_2024} but still have some issues of difficulties in generalization. We also extend to various intersection configurations with different traffic inflow rates, penetration rates of eco-driving systems, and durations of green phases at static traffic signals to optimize vehicle behaviors for reduced emissions.

Figure~\ref{fig:eco-driving-network} illustrates the traffic road network used in the eco-driving control task. The road network is depicted as traffic flowing vertically and horizontally, crossing the static phase traffic signal. There are both guided and default vehicles in the system. 
The state space includes the speed and position of the ego vehicle, the leading vehicle, and the following vehicles, supplemented by the current traffic signal phase and relevant context features, including lane length and green phase durations. The action space specifically focuses on the ego vehicle's acceleration control. The reward mechanism is designed to optimize the driving strategy by balancing the average speed of the vehicles against penalties for emissions, thereby promoting eco-friendly driving behaviors within the traffic system.
\jhhhedit{The traffic simulation used a default inflow of 400 vehicles per hour, a CAV penetration rate of 0.2, and a green phase time of 35 seconds to simulate realistic urban traffic conditions, with parameters varied from 0.1 to 5 times for CMDP.}
\begin{figure}[H]
    \centering
    \includegraphics[width=0.8\textwidth]{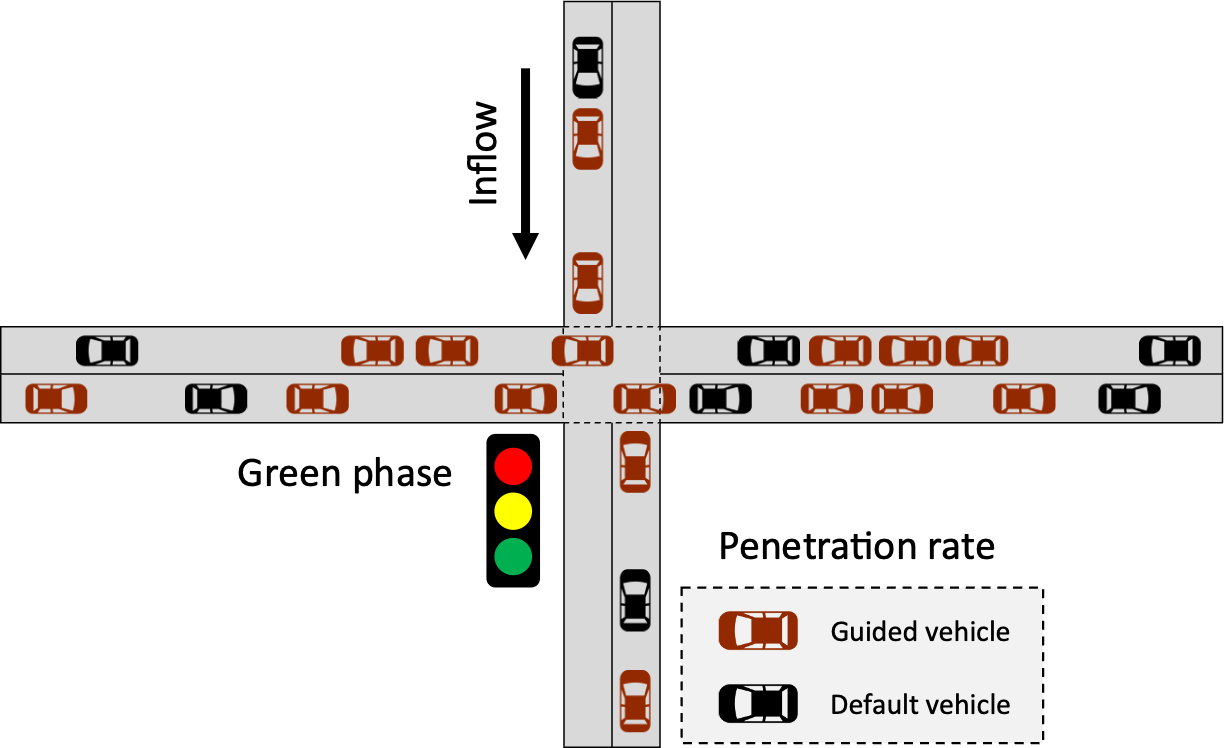}
    \caption{Illustration of the traffic networks in eco-driving control task.}
    \label{fig:eco-driving-network}
\end{figure}

\paragraph{Training configuration}
We also used the microscopic traffic simulation called Simulation of Urban MObility (SUMO) \cite{SUMO2018} v.1.16.0 and PPO for RL algorithm \cite{schulman_proximal_2017}. \jhhhedit{The Proximal Policy Optimization (PPO) algorithm was configured with a policy clipping parameter of 0.03 and an initial KL divergence coefficient of 0.1, targeting a KL divergence of 0.02 during training. The value function was clipped with a value clip of 3, and the value loss coefficient was set to 1. Entropy regularization was applied with a coefficient of 0.005 to encourage exploration. Gradient updates were performed over 10 steps per epoch, with training spanning 5000 epochs and 10 episodes per epoch. Each episode had a horizon of 1500 steps, and mini-batches of size 40 were used. The Adam optimizer was employed with a learning rate of 0.0001, weight decay of 0.97, and betas set to (0.9, 0.999). A neural network with four hidden layers, each 256 units wide, used the tanh activation function and orthogonal weight initialization. The simulation warmup steps were set to 50, and the simulation step size was 0.5 seconds. The discount factor ($\gamma$) was 0.99.
For detailed experimental details and RL hyperparameter configurations, please refer to \cite{jayawardana_mitigating_2024}.}

\paragraph{License} \jhhedit{Eco-driving benchmark falls under MIT License \cite{jayawardana_intersectionzoo_2024}.}

\paragraph{\jhhedit{Potential of multi-policy training and zero-shot transfer}}
\jhhedit{Figure~\ref{fig:gap} shows how each RL training paradigm adapts to variations in green phase time, penetration rate, and inflow rate in the eco-driving control benchmark. Oracle Transfer remains the strongest method across all configurations, benefiting from zero-shot transfer. independent training shows unstable performance across different task variations, performance, while multi-task training lags behind. Overall, the trends highlight the advantage of leveraging zero-shot transfer in traffic CMDPs.}

\paragraph{Transferability heatmap}
Figure~\ref{fig:heatmap-eco-driving} displays heatmaps for the eco-driving control task, with each heatmap varying an aspect such as green phase, inflow, and penetration rate. These visuals illustrate the transferability of strategies from source tasks (vertical axis) to target tasks (horizontal axis), highlighting the impact of traffic light phases, vehicle inflow, and the proportion of guided vehicles on strategy effectiveness. Notably, longer green phases correlate with improved performance and transferability. For inflow variations, reduced inflow typically yields better outcomes. However, variations in the penetration rate of guided vehicles show minimal impact on performance differences.

\begin{figure}[!h]
    \centering
    \begin{subfigure}[b]{0.32\textwidth}
        \includegraphics[width=\textwidth]{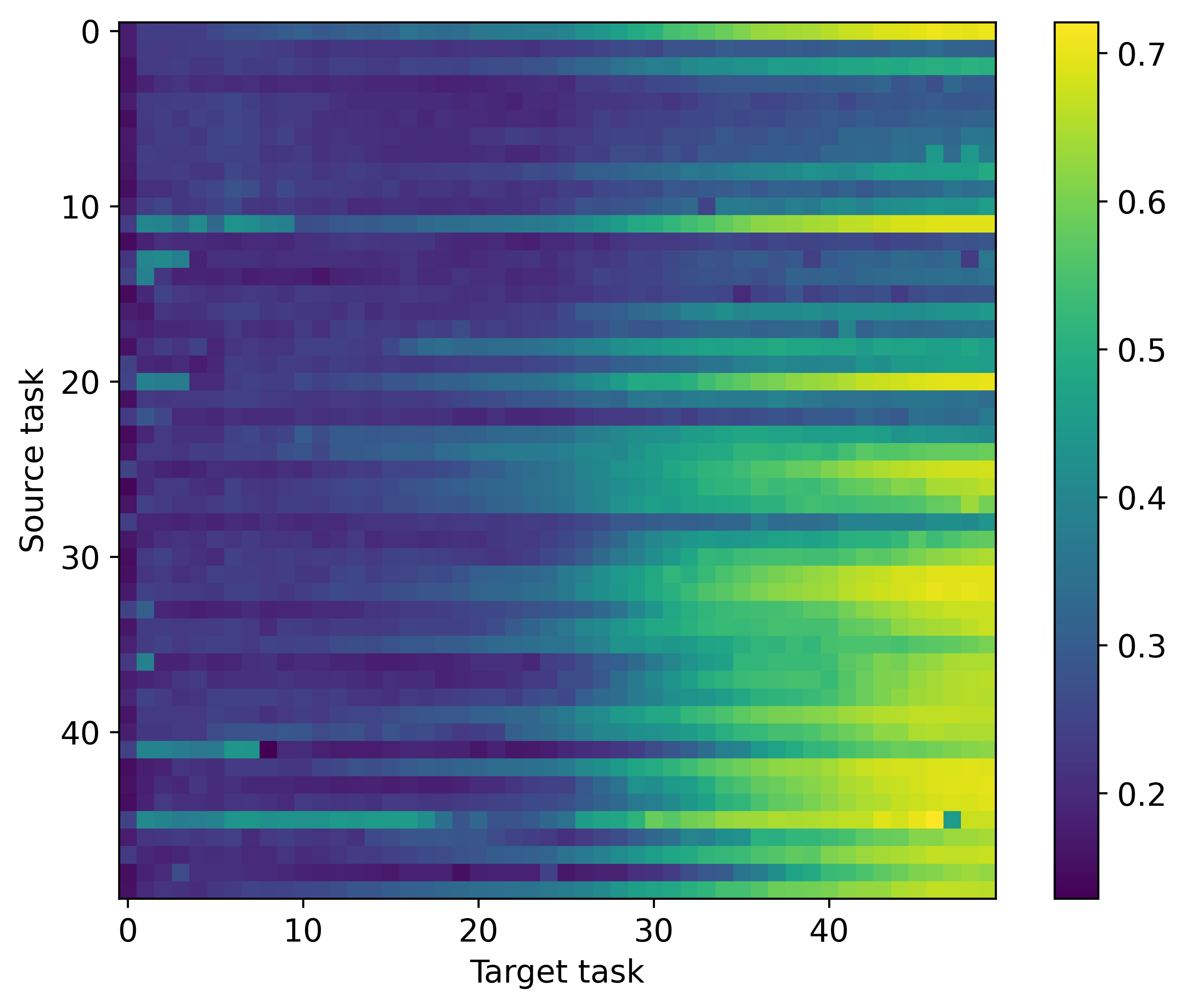}
        \caption{Green phase variation}
        \label{fig:heatmap-eco-driving-green}
    \end{subfigure}
    \hfill 
    \begin{subfigure}[b]{0.32\textwidth}
        \includegraphics[width=\textwidth]{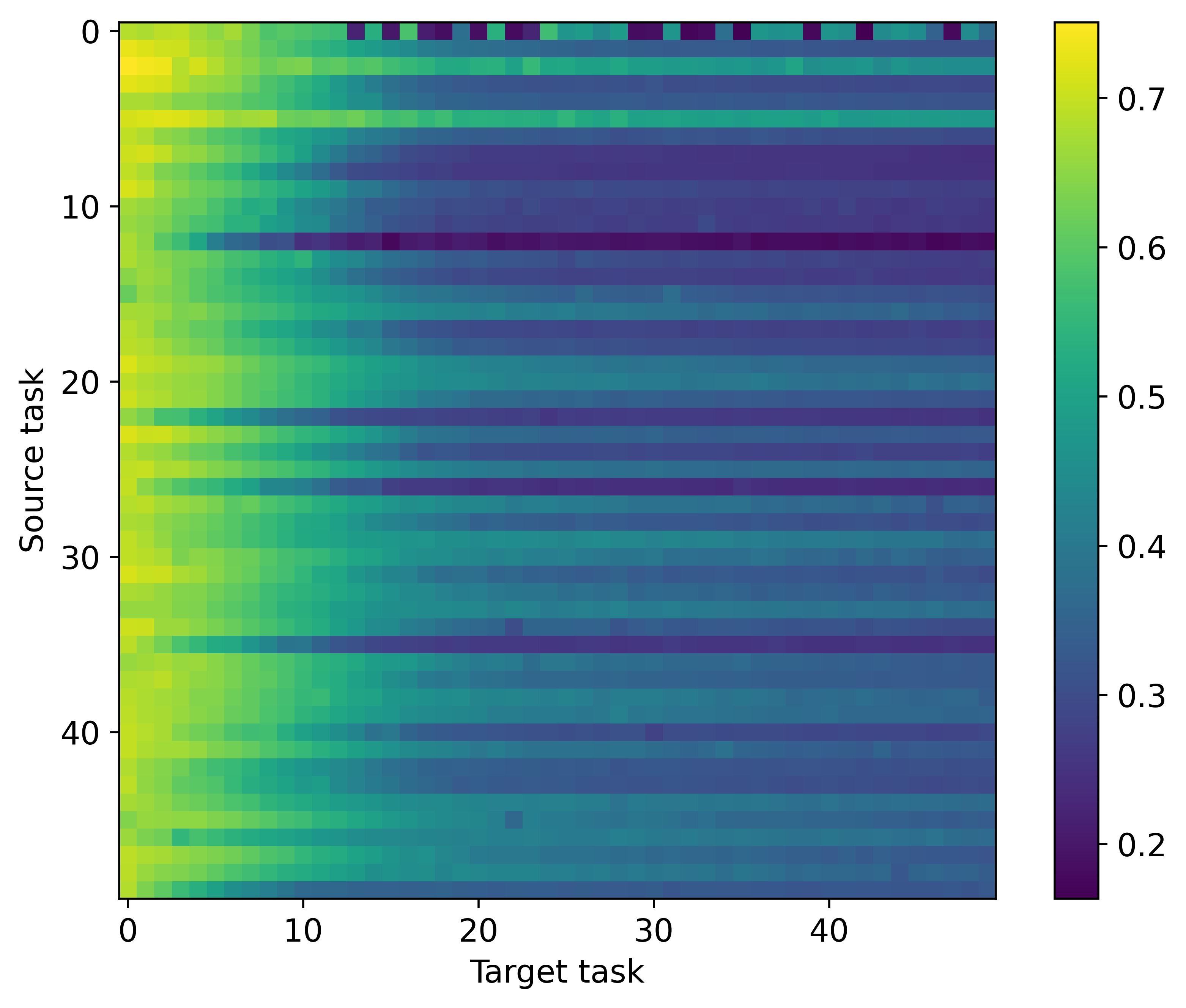}
        \caption{Inflow variation}
        \label{fig:heatmap-eco-driving-inflow}
    \end{subfigure}
    \hfill 
    \begin{subfigure}[b]{0.32\textwidth}
        \includegraphics[width=\textwidth]{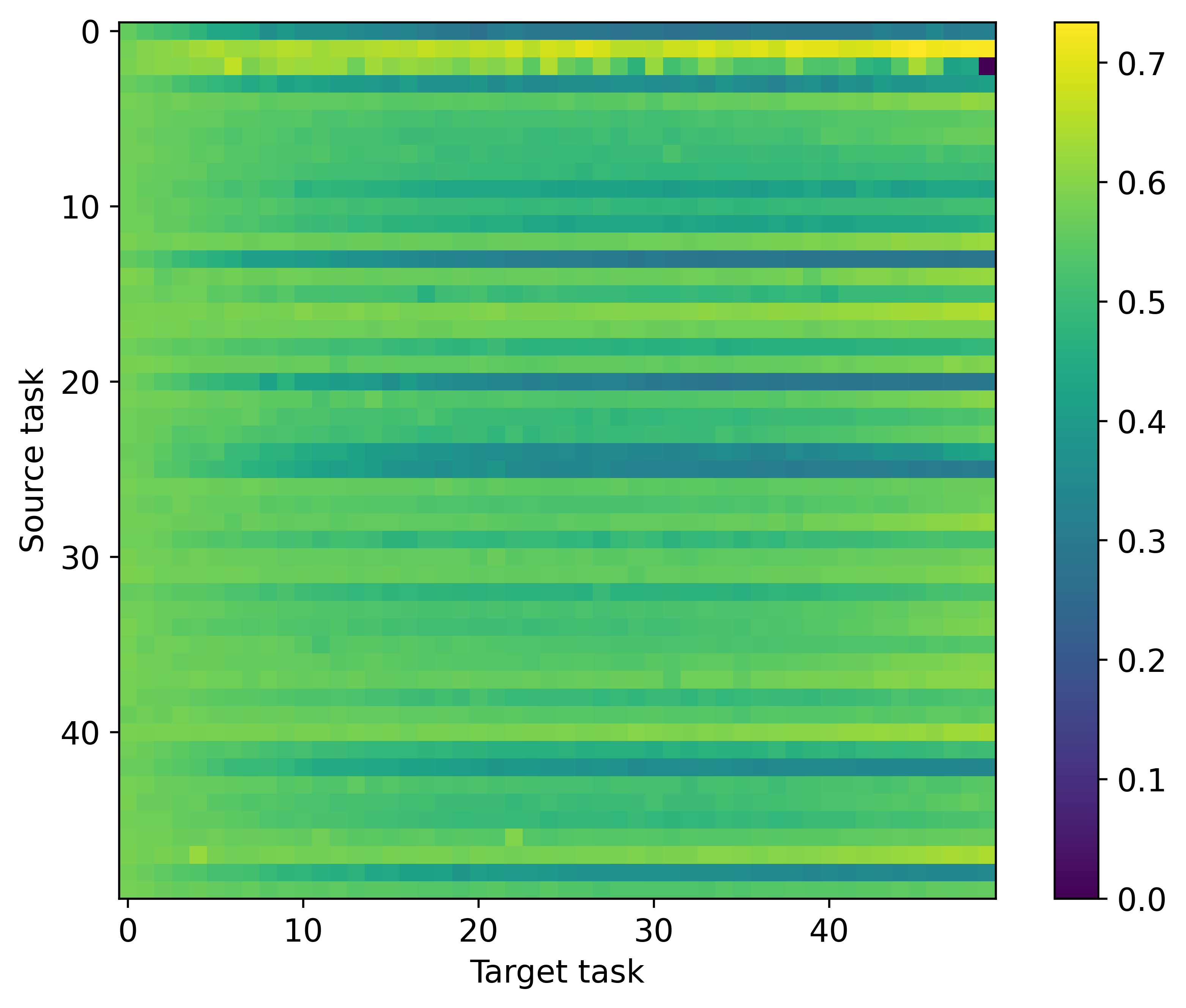}
        \caption{Penetration rate variation}
        \label{fig:heatmap-eco-driving-penrate}
    \end{subfigure}
    \caption{Examples of transferability heatmap for eco-driving control.}
    \label{fig:heatmap-eco-driving}
\end{figure}

\paragraph{Results}

Figure~\ref{fig:result-eco-driving} illustrates the normalized generalized performance across variants of eco-driving control tasks, specifically looking at variations in green phase time, inflow, and penetration rate. The graphs depict performance enhancement over transfer steps for different strategies. Notably, MBTL consistently demonstrates superior performance across all variations, indicating robust adaptability to changing task parameters.

\begin{figure}[!h]
    \centering
    \includegraphics[width=0.999\textwidth]{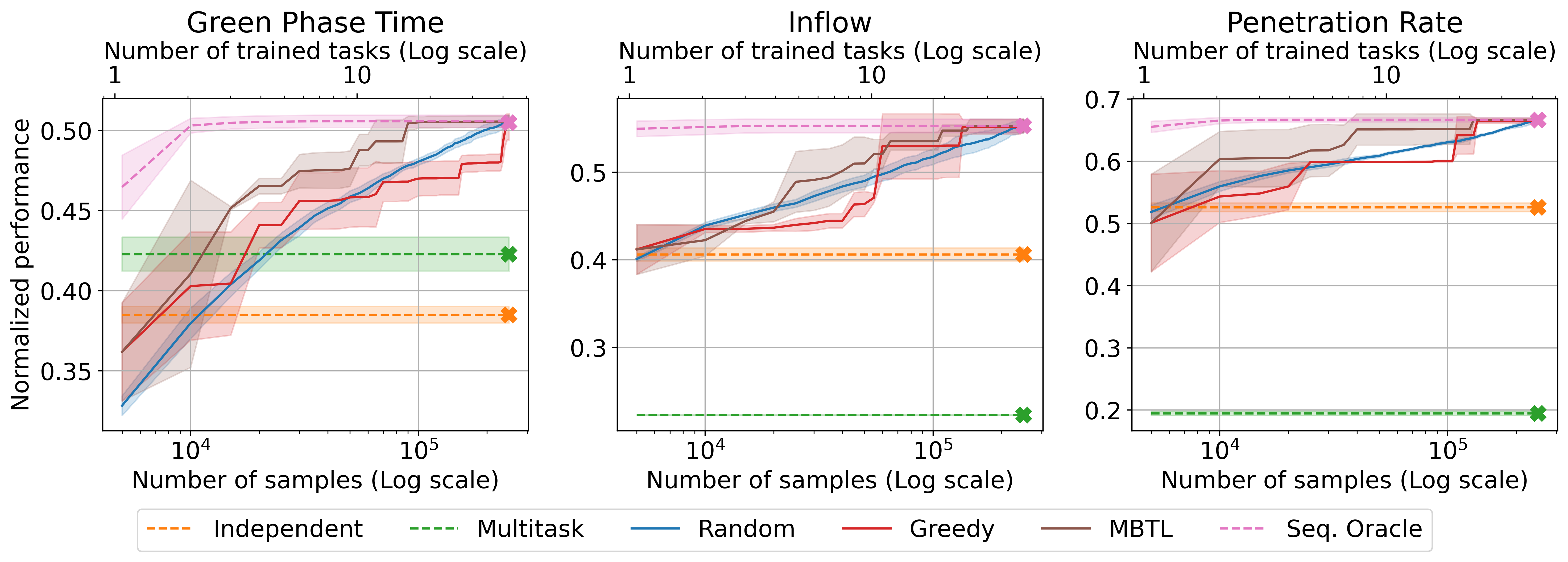}
    \caption{Comparison of normalized generalized performance of all target tasks: Eco-driving control.}
    \label{fig:result-eco-driving}
\end{figure}

% \clearpage
\subsubsection{Details about advisory autonomy \jhhedit{benchmark}}\label{appsec:detail-advisory-autonomy}

Advisory autonomy involves a real-time speed advisory system that enables human drivers to emulate the system-level performance of autonomous vehicles in mixed autonomy systems \cite{sridhar_piecewise_2021, cho_temporal_2023, hasan_cooperative_2024}. Instead of direct and instantaneous control, human drivers receive periodic guidance, which varies based on road type and guidance strategy. Here, we consider the different frequencies of this periodic guidance as contextual MDPs since the zero-order hold action affects the transition function.

Figure~\ref{fig:advisory-network} illustrates two distinct traffic network configurations used in the advisory autonomy task: a single-lane ring and a highway ramp. The single-lane ring features 22 vehicles circulating the ring, with only one being actively controlled, presenting a relatively controlled environment for testing vehicle guidance systems. The highway ramp scenario introduces a more complex dynamic, where vehicles not only travel along the highway but also merge from ramps, creating potential stop-and-go traffic patterns that challenge the adaptability of autonomous guidance systems. \jhhhedit{The road network consists of a pre-merge distance of 400 m, a merge distance of 100 m, and a post-merge distance of 30 m. Traffic inflow rates were set to 2000 vehicles per hour on the highway and 300 vehicles per hour on the ramp.}

\paragraph{Problem Definition} 
In a single-lane ring scenario, the state space includes the speeds of the ego and leading vehicles, along with the headway. Vehicle dynamics incorporated acceleration and deceleration limits of 0.5 m/s$^2$. For highway ramp scenarios, additional states cover the relative positions and speeds of adjacent vehicles. Actions vary by guidance type: for acceleration guidance, the action space is continuous, ranging from $-1$ to $1$; for speed guidance, it has ten discrete actions compared to the speed limit. Rewards are based on system throughput or average speed of all vehicles in the system.

\paragraph{Context Variations} 
We explore different durations of coarse-grained guidance holds to test various levels of human compatibility, adjusting the model based on observed driver behaviors and system performance.

\paragraph{\jhhhedit{Training configuration}}
\jhhhedit{Advisory autonomy experiments utilized Trust Region Policy Optimization (TRPO) \cite{schulman_trust_2015}, with a discount factor ($\gamma$) of 0.999, a learning rate of $10^{-3}$, and the Adam optimizer configured with weight decay of 0.97 and betas (0.9, 0.999). Policies were modeled with a four-layer neural network, each with 256 units and tanh activation, using orthogonal weight initialization. The KL divergence constraint was set to 0.02, with an initial KL coefficient of 0.1 dynamically adjusted during training. Rewards were normalized and centered, and regularization penalties were applied to ensure stable and robust policy optimization.}

\paragraph{License} \jhhedit{Advisory autonomy benchmark falls under MIT License \cite{sridhar_piecewise_2021}.}

\begin{figure}[!ht]
    \centering
    \includegraphics[width=0.8\textwidth]{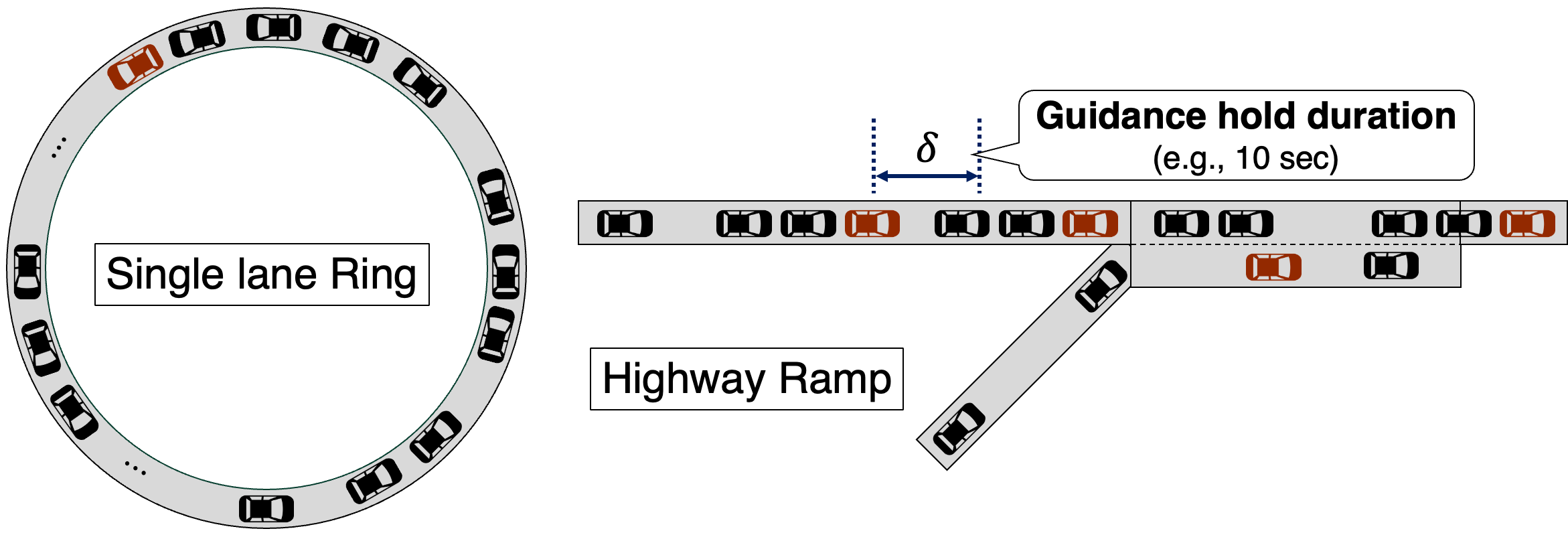}
    \caption{Illustration of the traffic networks in advisory autonomy task.}
    \label{fig:advisory-network}
\end{figure}

\paragraph{\jhhedit{Potential of multi-policy training and zero-shot transfer}}
\begin{figure}[H]
    \centering
    \includegraphics[width=0.99\textwidth]{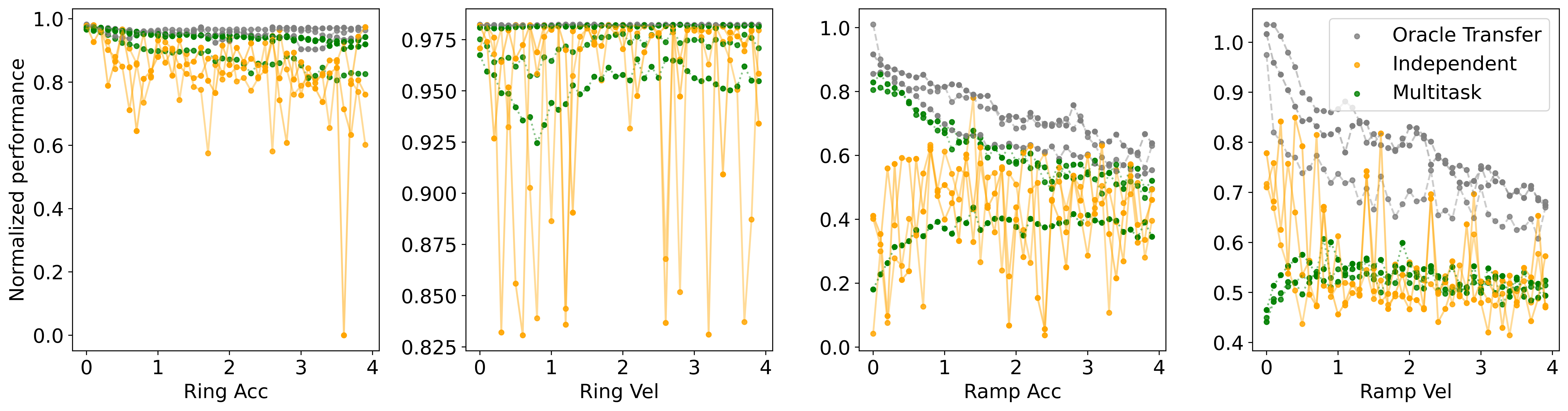}
    \caption{\jhhedit{Normalized performance of Oracle Transfer, independent training, and multi-task training under Advisory Autonomy benchmark with human compatibility task variations.}}
    \label{fig:gap-advisory}
\end{figure}
\jhhedit{Figure~\ref{fig:gap-advisory} shows that ring-road networks tend to yield higher performance and smaller gaps compared to highway ramp scenarios. In addition, independent training exhibits greater performance drop and variability due to training instability. Oracle Transfer retains clear potential improvements over other baselines.}

\paragraph{Transferability heatmap}
Figure~\ref{fig:heatmap-advisory-autonomy} showcases heatmaps of transferability for advisory autonomy tasks, each varying in specific aspects: acceleration guidance and speed guidance across a single lane ring and a highway ramp. These heatmaps demonstrate the effectiveness of strategy transfer from each source task (vertical axis) to each target task (horizontal axis), capturing how variations in task conditions influence adaptability. For acceleration guidance in a ring setup (a), transferability is generally higher among tasks with similar acceleration demands. In contrast, speed guidance on a ramp (d) reveals more variability in transferability, potentially due to the complexity of speed adjustments in ramp scenarios.

\begin{figure}[!h]
    \centering
    \begin{subfigure}[t]{0.24\textwidth}
        \includegraphics[width=\textwidth]{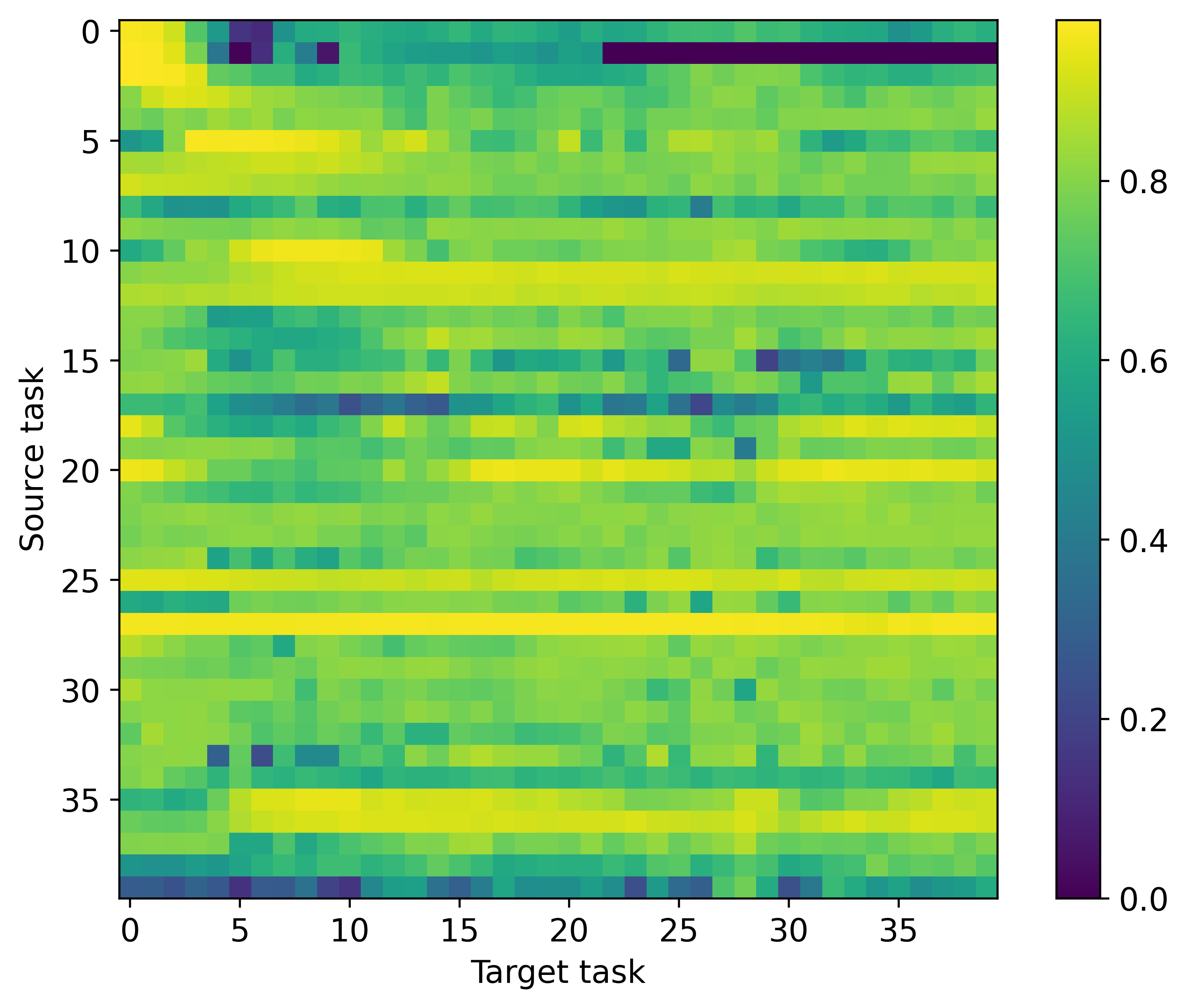}
        \caption{Ring with acceleration guide}
        \label{fig:heatmap-aa-ring-acc}
    \end{subfigure}
    \hfill 
    \begin{subfigure}[t]{0.24\textwidth}
        \includegraphics[width=\textwidth]{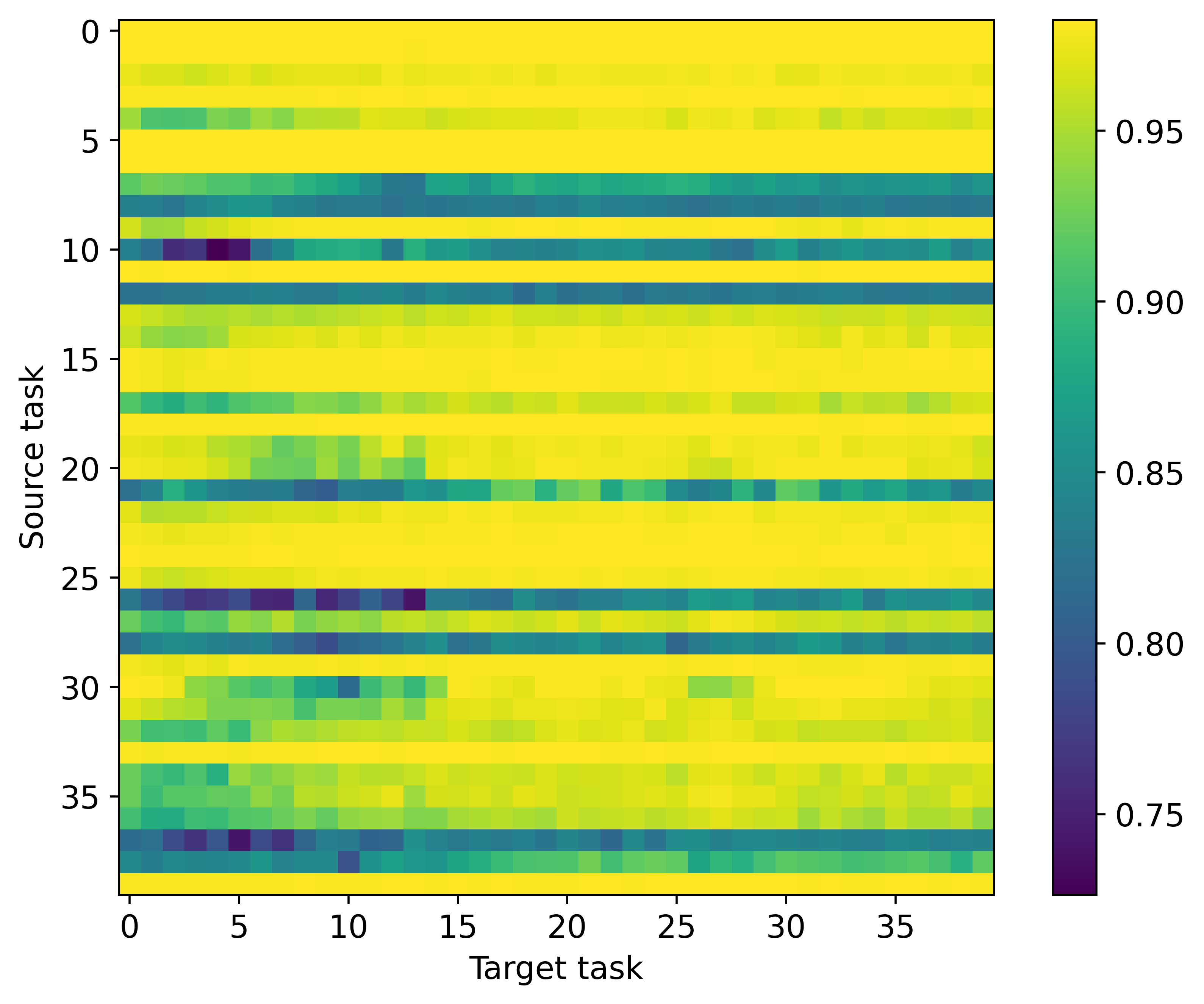}
        \caption{Ring with speed guide}
        \label{fig:heatmap-aa-ring-vel}
    \end{subfigure}
    \hfill 
    \begin{subfigure}[t]{0.24\textwidth}
        \includegraphics[width=\textwidth]{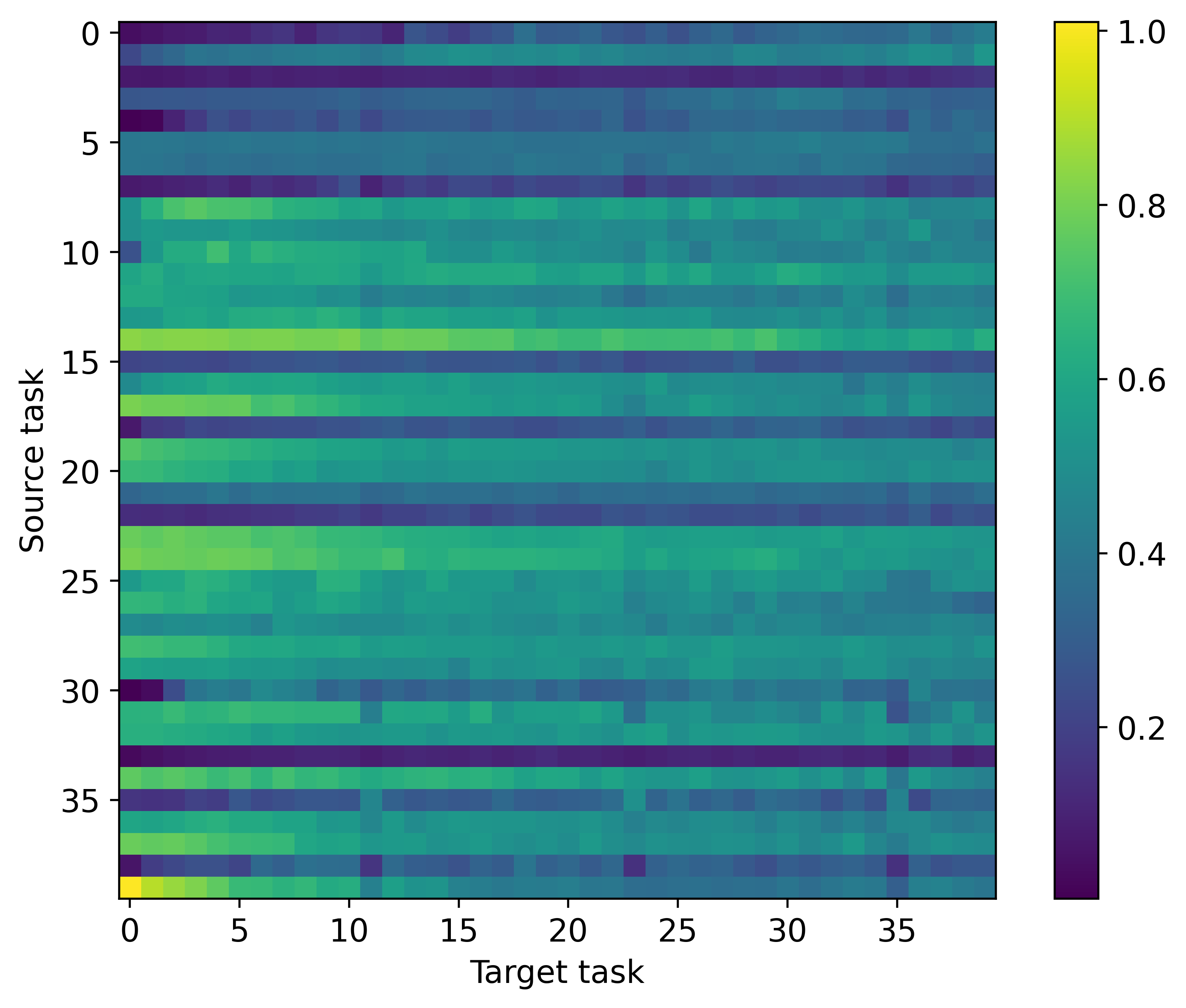}
        \caption{Ramp with acceleration guide}
        \label{fig:heatmap-aa-ramp-acc}
    \end{subfigure}
    \hfill 
    \begin{subfigure}[t]{0.24\textwidth}
        \includegraphics[width=\textwidth]{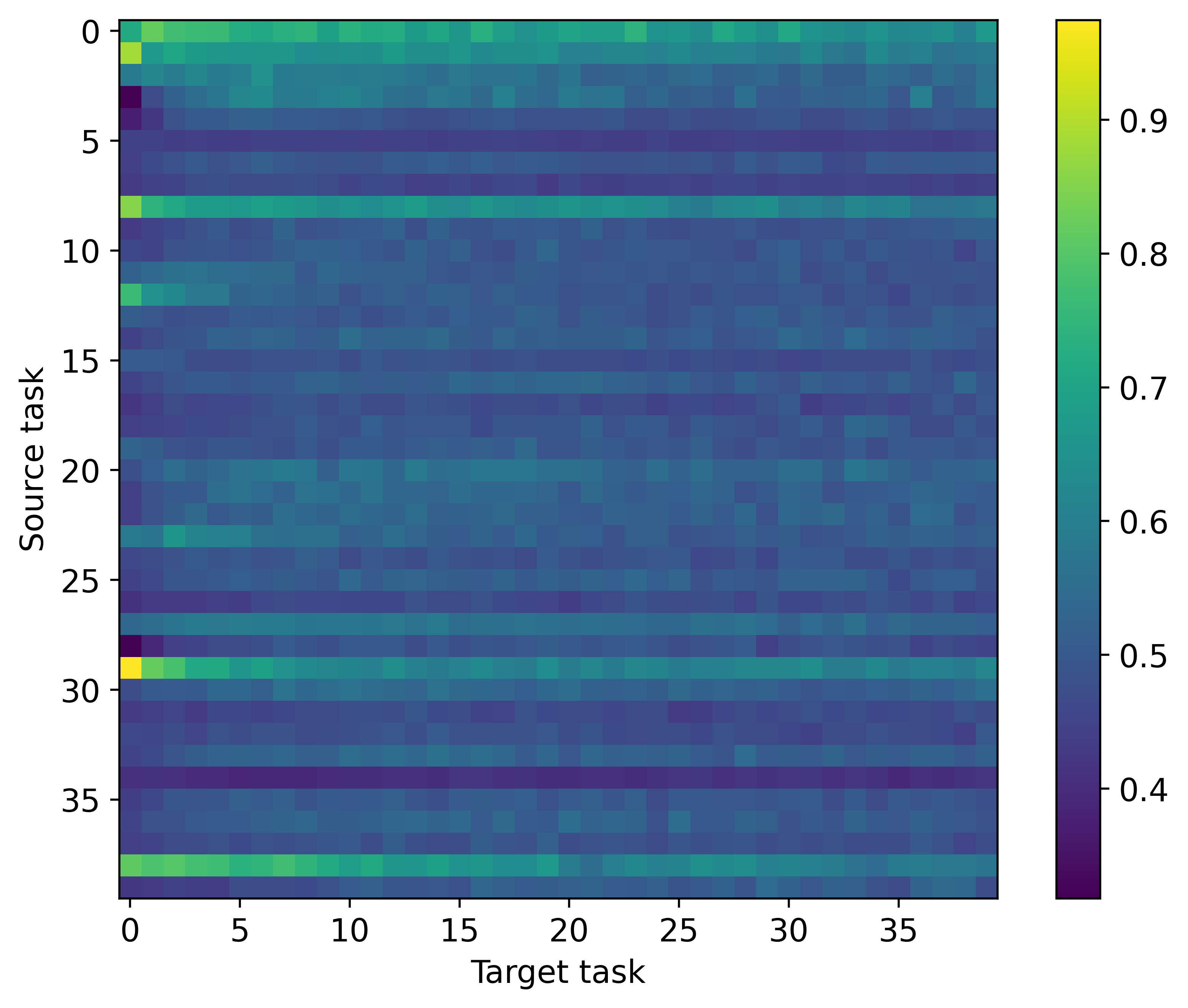}
        \caption{Ramp with speed guide}
        \label{fig:heatmap-aa-ramp-vel}
    \end{subfigure}
    \caption{Examples of transferability heatmap for advisory autonomy.}
    \label{fig:heatmap-advisory-autonomy}
\end{figure}

\paragraph{Results}

Figure~\ref{fig:result-advisory-autonomy} illustrates the comparison of normalized generalized performance for advisory autonomy tasks, specifically acceleration and speed guidance in a ring and acceleration guidance on a ramp. The graphs demonstrate that MBTL consistently exhibits higher performance across all tasks. Particularly, acceleration guidance in both ring and ramp scenarios shows significant performance improvements over transfer steps, with MBTL closely matching in some instances.

\begin{figure}[!h]
    \centering
    \includegraphics[width=0.999\textwidth]{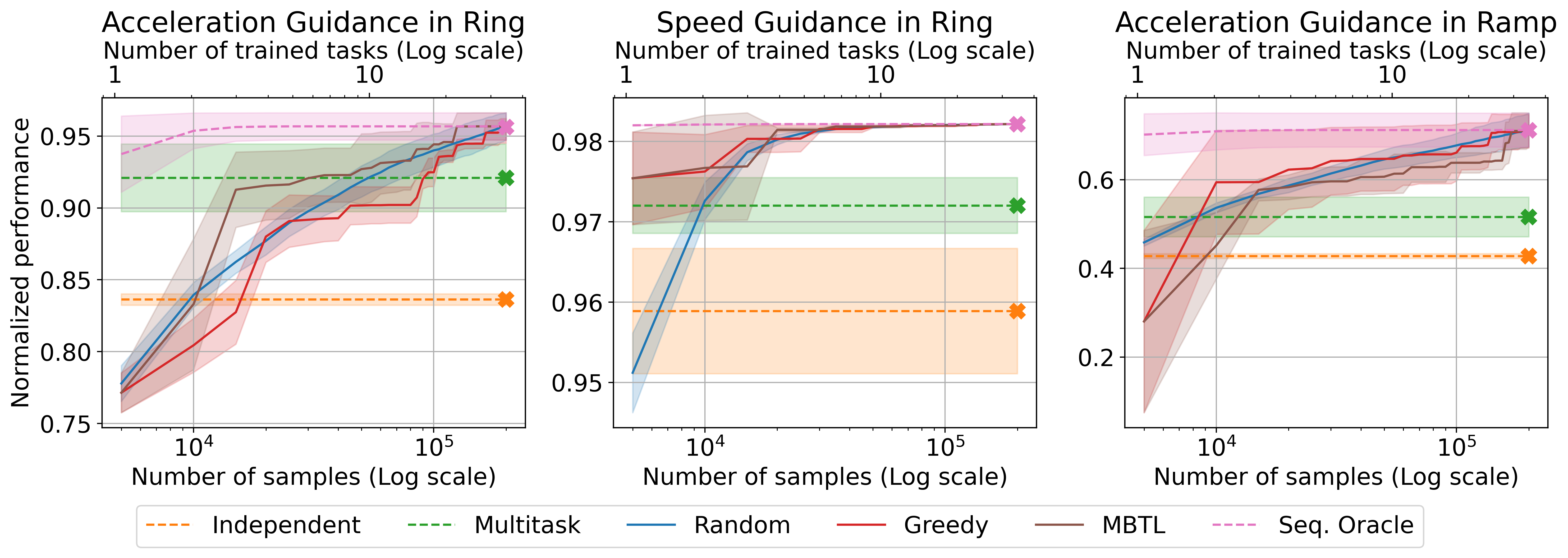}
    \caption{Comparison of normalized generalized performance of all target tasks: Advisory autonomy.}
    \label{fig:result-advisory-autonomy}
\end{figure}

\subsubsection{Details of control \jhhedit{benchmarks}}\label{appsec:detail-control}

For this experimental phase, we selected context-extended versions of standard RL environments from the CARL benchmark library, including Cartpole, Pendulum, BipedalWalker, and Halfcheetah. These environments were chosen to rigorously test the robustness and adaptability of our MBTL algorithm under varied conditions that mirror the complexity encountered in real-world scenarios.

\textbf{Context Variations:} In the Cartpole tasks, we explored \jhedit{CMDP}s with varying cart masses, pole lengths, and pole masses. For the Pendulum, the experiments involved adjusting the timestep duration, pendulum length, and pendulum mass. The BipedalWalker was tested under different settings of friction, gravity, and scale. Similarly, in the Halfcheetah tasks, we manipulated parameters such as friction, gravity, and stiffness to simulate different physical conditions. These variations critically influence the dynamics and physics of the environments, thereby presenting unique challenges that test the algorithm’s capacity to generalize from previous learning experiences without the need for extensive retraining. The range of context variations was established by scaling the default values specified in the CARL framework from 0.1 to 10 times, enabling a comprehensive examination of each model’s performance under drastically different conditions.

\paragraph{\jhhhedit{Implementation Details.}} \jhhhedit{We employed the \texttt{PPO} algorithm from \texttt{stable\_baselines3} (v1.5.0) \cite{stable-baselines3} with its default hyperparameters, including a learning rate of $3\times10^{-4}$, \texttt{n\_steps}$= 2048$, batch size $64$, discount factor $\gamma = 0.99$, GAE parameter $\lambda = 0.95$, clipping parameter $0.2$, entropy coefficient $0$, and a value function loss coefficient of $0.5$.}

\textbf{License:} CARL falls under the Apache License 2.0 as is permitted by all work that we use \cite{benjamins_contextualize_2023}.

\subsubsection{Details about Cartpole \jhhedit{benchmark}}\label{appsec:result-cartpole}

\paragraph{\jhhhedit{Environment Details.}} \jhhhedit{We utilized the CARL benchmark library’s \emph{default} environment parameters for the Cartpole task, training for five million total timesteps. Specifically, Cartpole used the default length of a pole of $0.5$, the mass of the cart of $1.0$, and the mass of a pole of $0.1$.}

\paragraph{\jhhedit{Potential of multi-policy training and zero-shot transfer}}
\jhhedit{Cartpole may be considered a simpler benchmark than traffic benchmarks, yet independent and multi-task training methods still face notable difficulty when faced with context variations. Figure~\ref{fig:gap-cartpole} shows that Oracle Transfer performs at the highest performance across problem variations, while independent training or multi-task training shows a larger variance in performance.}
\begin{figure}[H]
    \centering
    \includegraphics[width=0.99\textwidth]{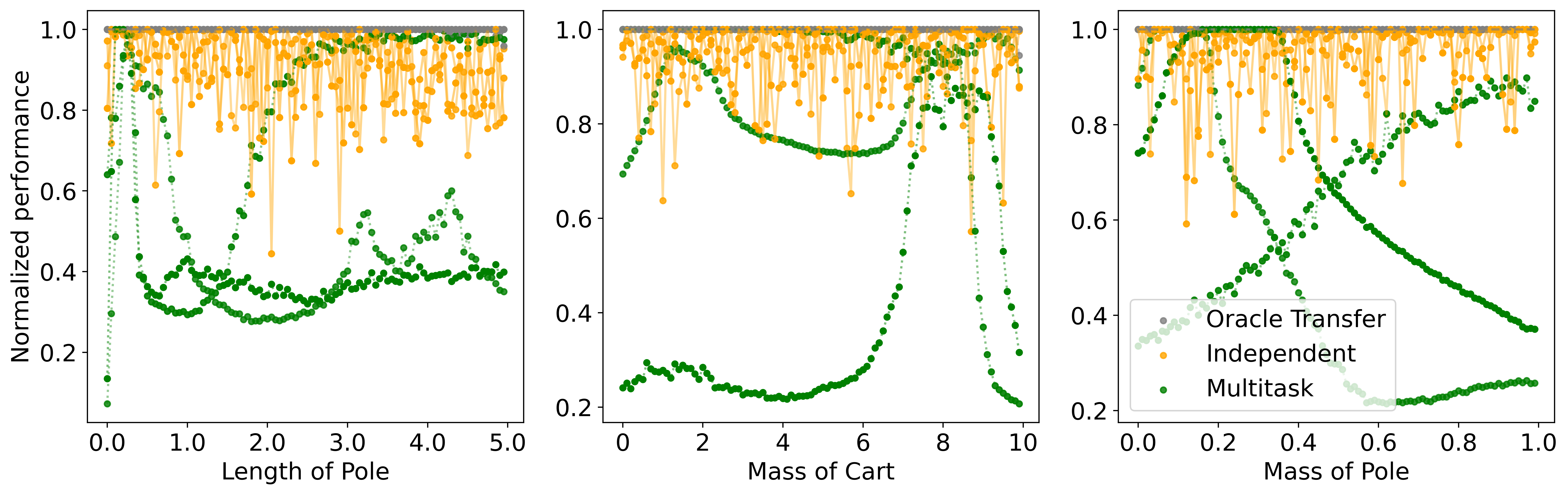}
    \caption{\jhhedit{Normalized performance of Oracle Transfer, independent training, and multi-task training in Cartpole benchmarks.}}
    \label{fig:gap-cartpole}
\end{figure}

\paragraph{Transferability heatmap}
Figure~\ref{fig:heatmap-cartpole} presents transferability heatmaps for the Cartpole task with variations in three physical properties: mass of the cart, length of the pole, and mass of the pole. Each heatmap illustrates how well strategies transfer from source tasks (vertical axis) to target tasks (horizontal axis), depicting the influence of each parameter on control strategy effectiveness. For the mass of the cart variation (a), transferability decreases as the mass difference increases. In the length of the pole variation (b), strategies are less transferable between significantly different pole lengths. Similarly, for the mass of the pole variation (c), variations show divergent transferability depending on the extent of mass change.

\begin{figure}[!h]
    \begin{subfigure}[t]{0.32\textwidth}
        \includegraphics[width=\textwidth]{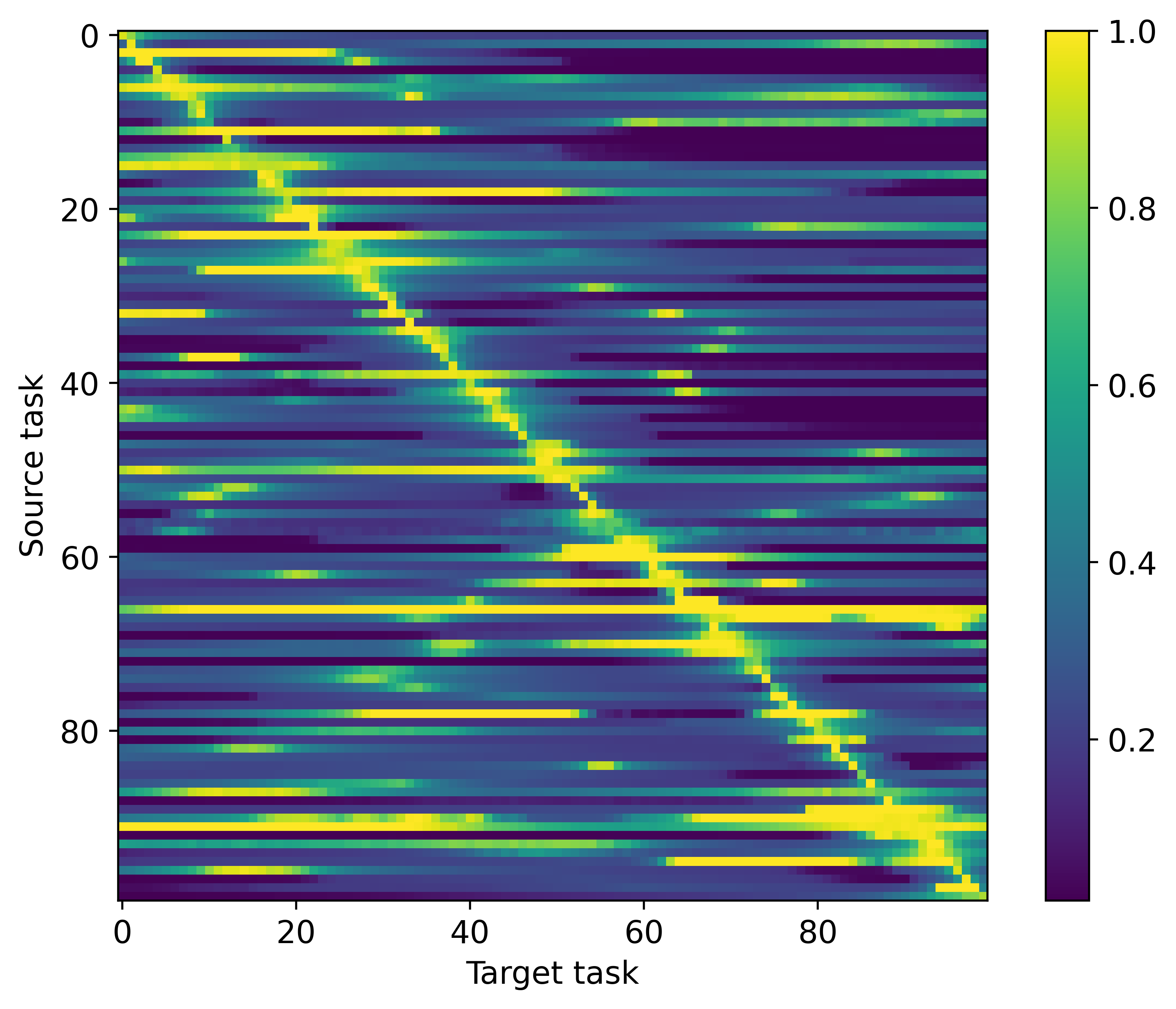}
        \caption{Mass of cart variation}
        \label{fig:heatmap-cartpole-masscart}
    \end{subfigure}
    \hfill 
    \centering
    \begin{subfigure}[t]{0.32\textwidth}
        \includegraphics[width=\textwidth]{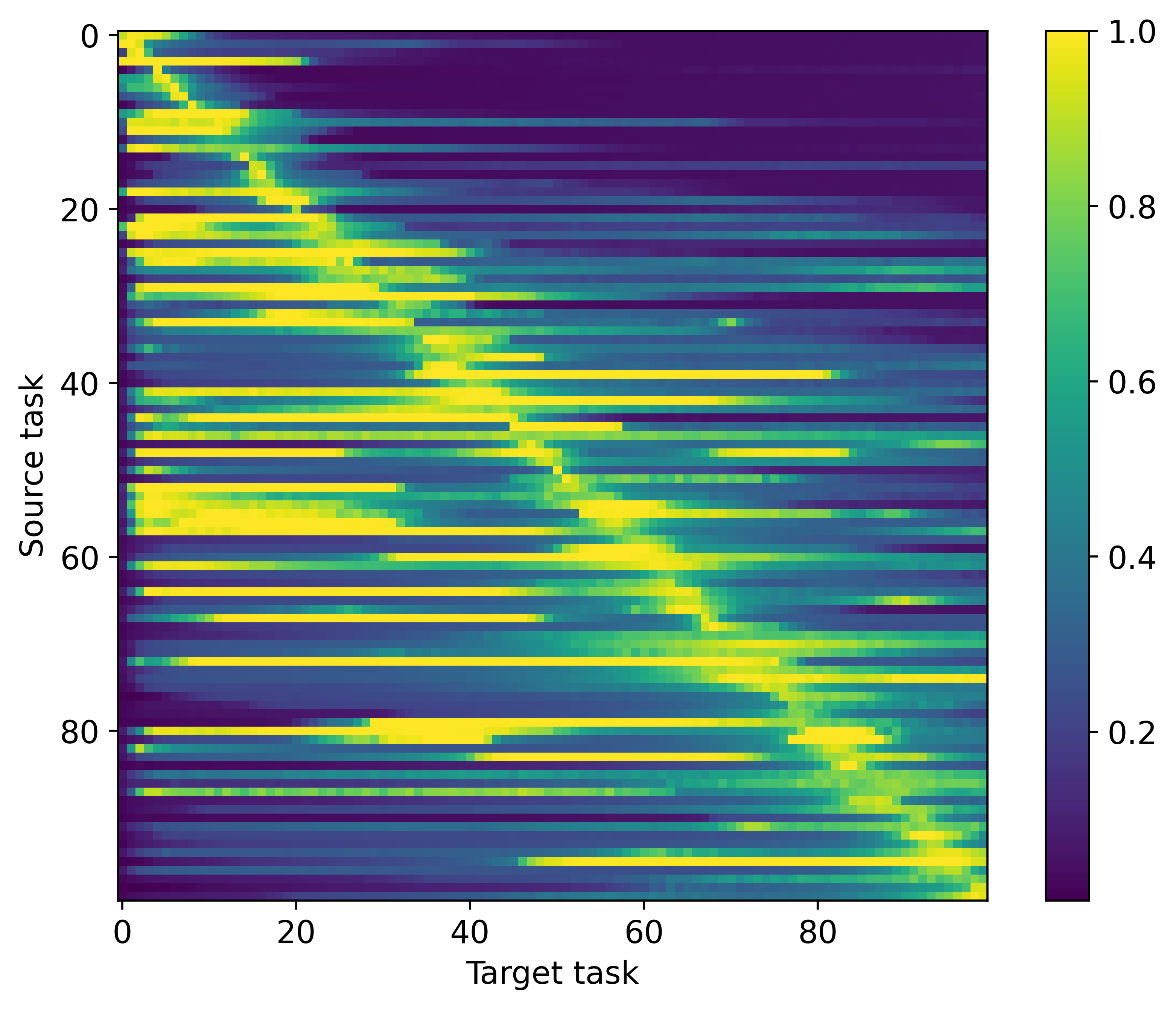}
        \caption{Length of the pole variation}
        \label{fig:heatmap-cartpole-lenpole}
    \end{subfigure}
    \hfill 
    \begin{subfigure}[t]{0.32\textwidth}
        \includegraphics[width=\textwidth]{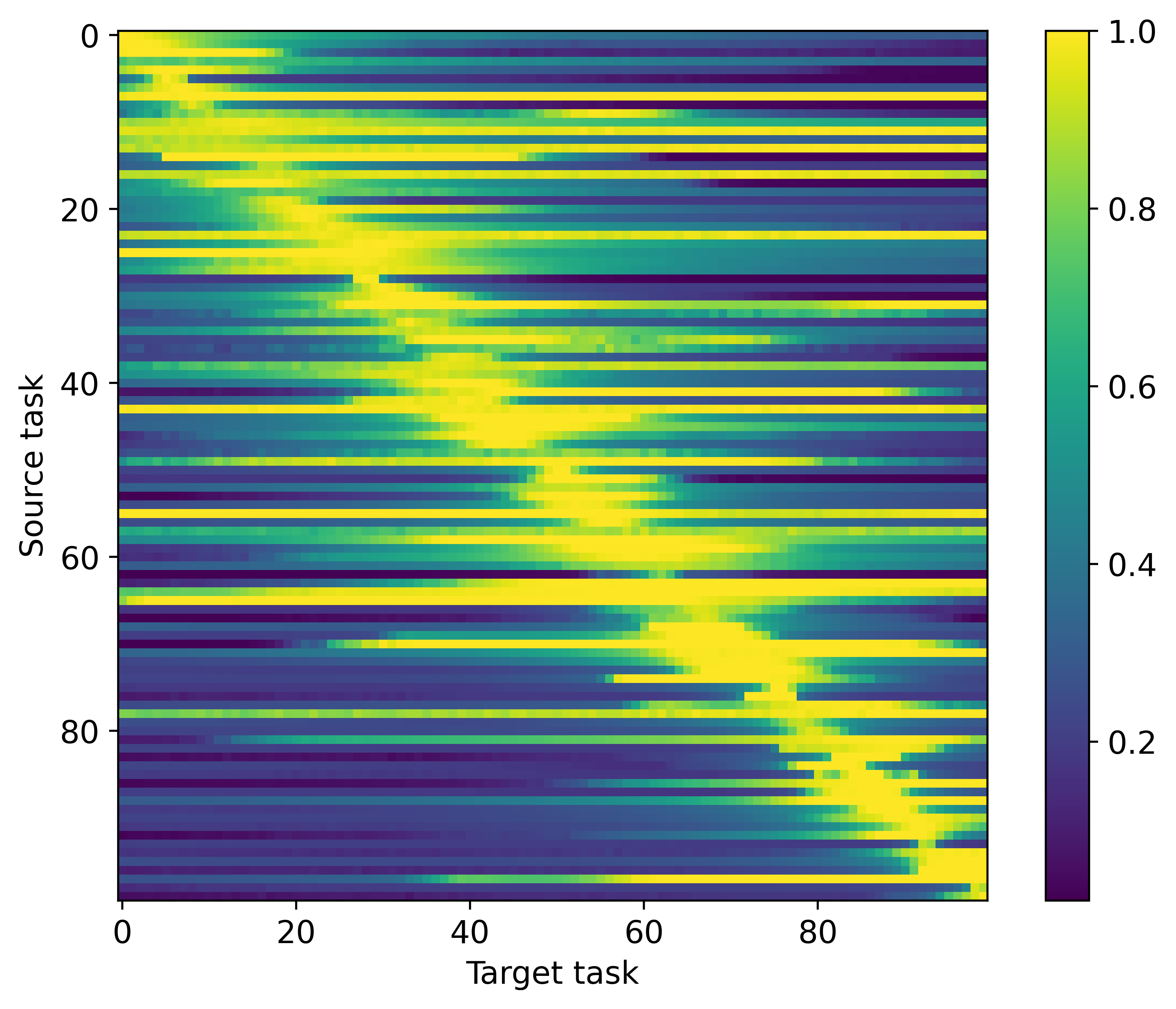}
        \caption{Mass of pole variation}
        \label{fig:heatmap-cartpole-masspole}
    \end{subfigure}
    \caption{Examples of transferability heatmap for Cartpole.}
    \label{fig:heatmap-cartpole}
\end{figure}

\paragraph{Results}

Figure~\ref{fig:result-cartpole} presents a comparison of normalized generalized performance for the Cartpole task across different strategies when varying the mass of the cart, length of the pole, and mass of the pole. 
In the mass of cart variation, performance generally increases with transfer steps, with MBTL strategies achieving the highest scores. This indicates robust adaptability to changes in cart mass. Similar trends are observed with length variation and mass of pole variation. MBTL shows close to Oracle performance.

\begin{figure}[!h]
    \centering
    \includegraphics[width=0.999\textwidth]{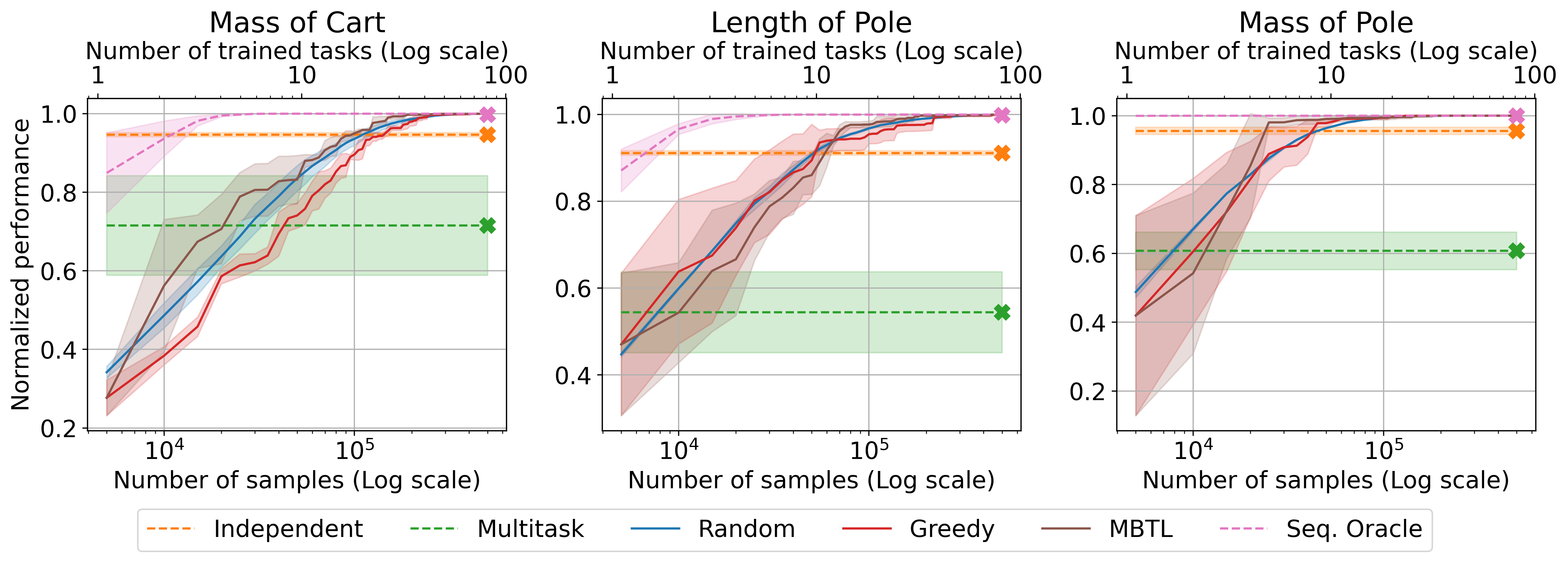}
    \caption{Comparison of normalized generalized performance of all target tasks: Cartpole.}
    \label{fig:result-cartpole}
\end{figure}

\subsubsection{Details about Pendulum \jhhedit{benchmark}}\label{appsec:result-pendulum}
\paragraph{\jhhhedit{Environment Details.}} \jhhhedit{We utilized the CARL benchmark library’s \emph{default} environment parameters for the Pendulum task, training for a million total timesteps. Specifically, we used the default length of $1.0$, the mass of $1.0$, and the simulation timestop of $0.05$.}

\paragraph{\jhhedit{Potential of multi-policy training and zero-shot transfer}}
\jhhedit{Pendulum is also one of the simplest benchmarks in classic control. In the Pendulum benchmark, we vary three key parameters: time step (left), pole length (middle), and ball mass (right). As shown in Figure~\ref{fig:gap-pendulum}, a few well-trained policies in specific contexts transfer effectively to new tasks, particularly under Oracle Transfer. For certain configurations (e.g., shorter poles, lighter balls), Independent and Oracle Transfer both excel, while multi-task struggles. These results suggest a remaining performance gap that multi-policy training and zero-shot transfer could help bridge.}
\begin{figure}[H]
    \centering
    \includegraphics[width=0.99\textwidth]{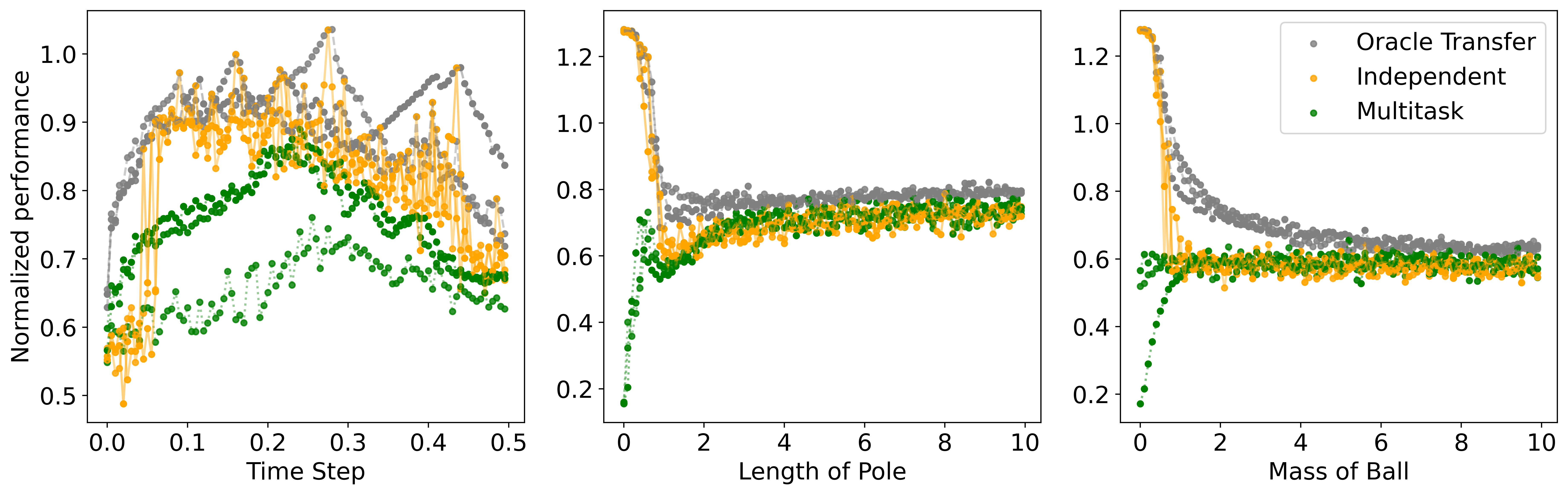}
    \caption{\jhhedit{Normalized performance of Oracle Transfer, independent training, and multi-task training in Pendulum benchmarks.}}
    \label{fig:gap-pendulum}
\end{figure}

\jhhedit{Figure~\ref{fig:gap-pendulum-link} provides a visual mapping of which trained policy is being applied for each specific context in the Pendulum benchmark. CMDP with time step variation demonstrates that the tasks are covered by a few ``good" policies nearby. In addition, contexts with shorter poles and lighter balls often gravitate toward a single high-performing policy, whereas more challenging configurations may require a specialized or distinct policy. This illustrates how Oracle Transfer can seamlessly select from a suite of learned policies, demonstrating robust zero-shot transfer and stronger adaptability.}

\begin{figure}[H]
    \centering
    \includegraphics[width=0.99\textwidth]{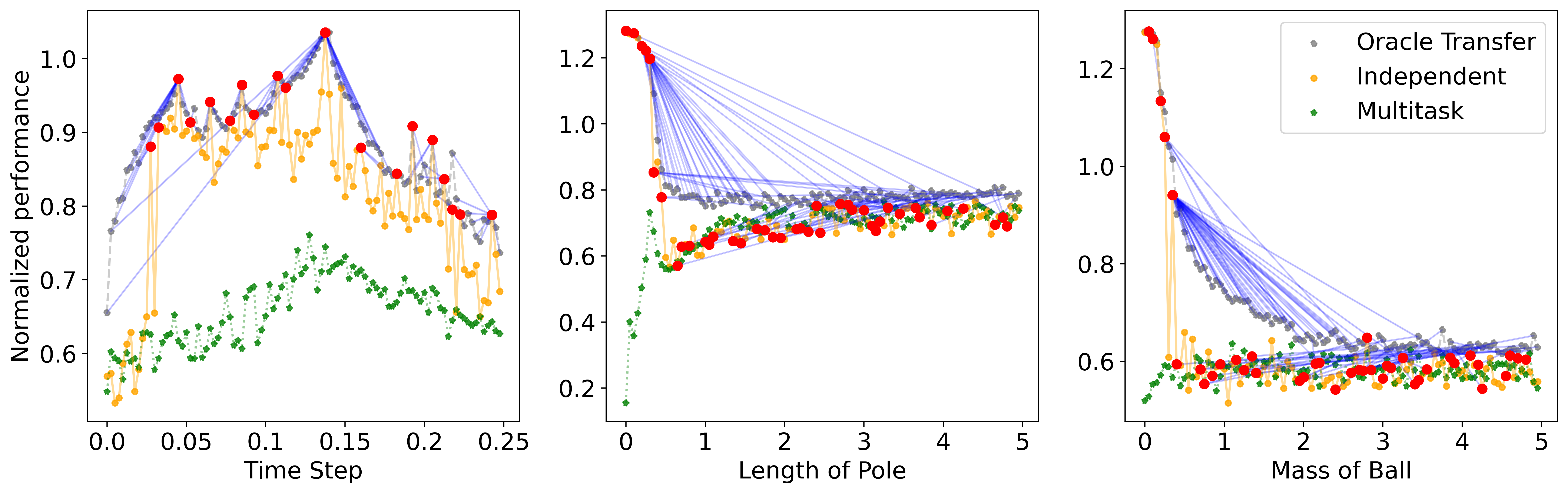}
    \caption{\jhhedit{Visulazation on which policy is used to solve specific context-MDP in Pendulum CMDP.}}
    \label{fig:gap-pendulum-link}
\end{figure}

\paragraph{Transferability heatmap}
Figure~\ref{fig:heatmap-pendulum} presents transferability heatmaps for the Pendulum task with variations in three physical properties: timestep, length of the pendulum, and mass of the pendulum. Each heatmap illustrates how effectively strategies transfer from source tasks (vertical axis) to target tasks (horizontal axis), highlighting the impact of each parameter on control strategy effectiveness. For the timestep variation (a), there appears to be high consistency in transferability across different timesteps, especially around the diagonal axis. In the length of the pendulum variation (b), transferability decreases with greater length differences. Similarly, for the mass of the pendulum variation (c), transferability shows variability dependent on the extent of mass changes.

\begin{figure}[!h]
    \begin{subfigure}[t]{0.32\textwidth}
        \includegraphics[width=\textwidth]{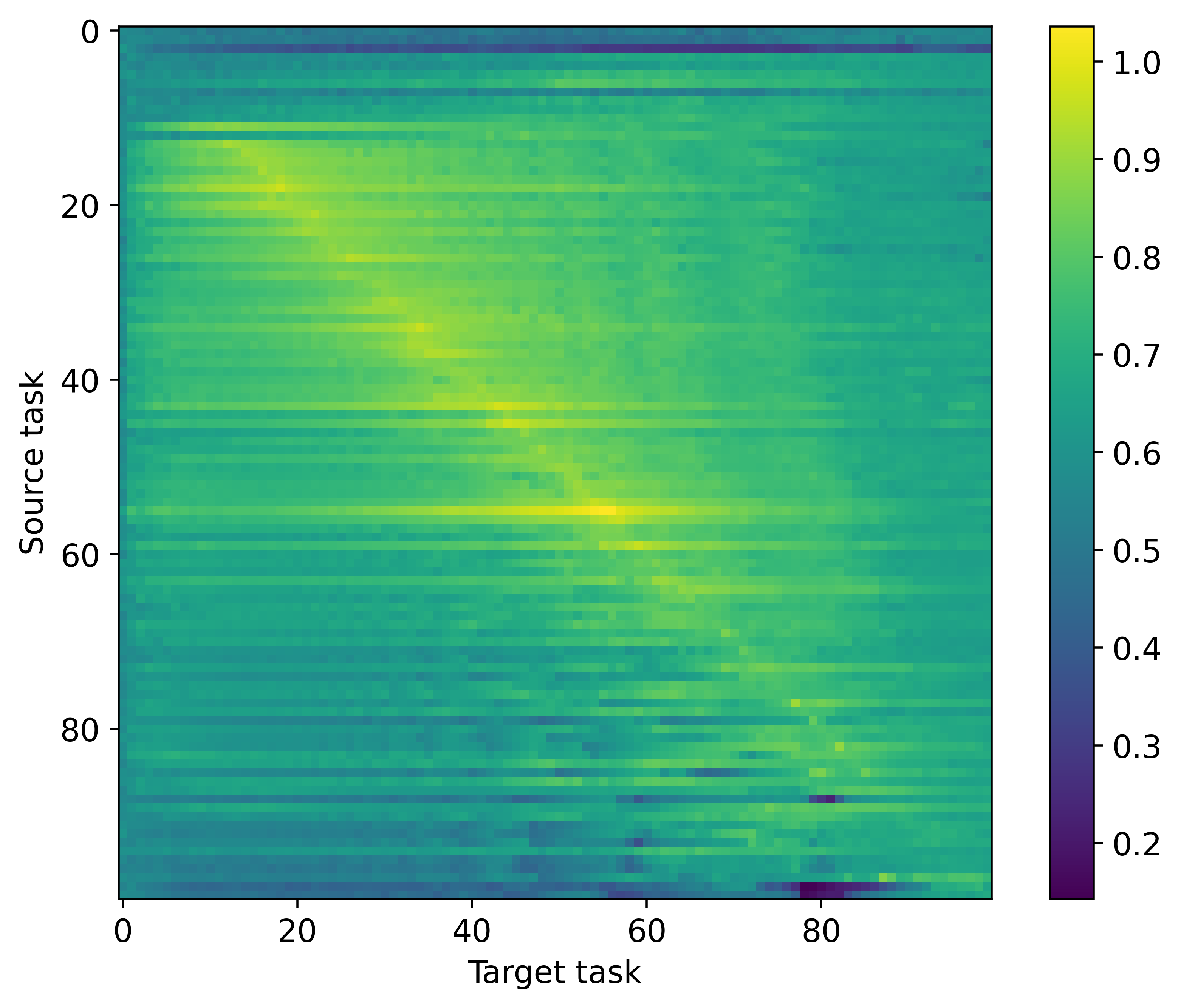}
        \caption{Timestep variation}
        \label{fig:heatmap-pendulum-dt}
    \end{subfigure}
    \hfill 
    \centering
    \begin{subfigure}[t]{0.32\textwidth}
        \includegraphics[width=\textwidth]{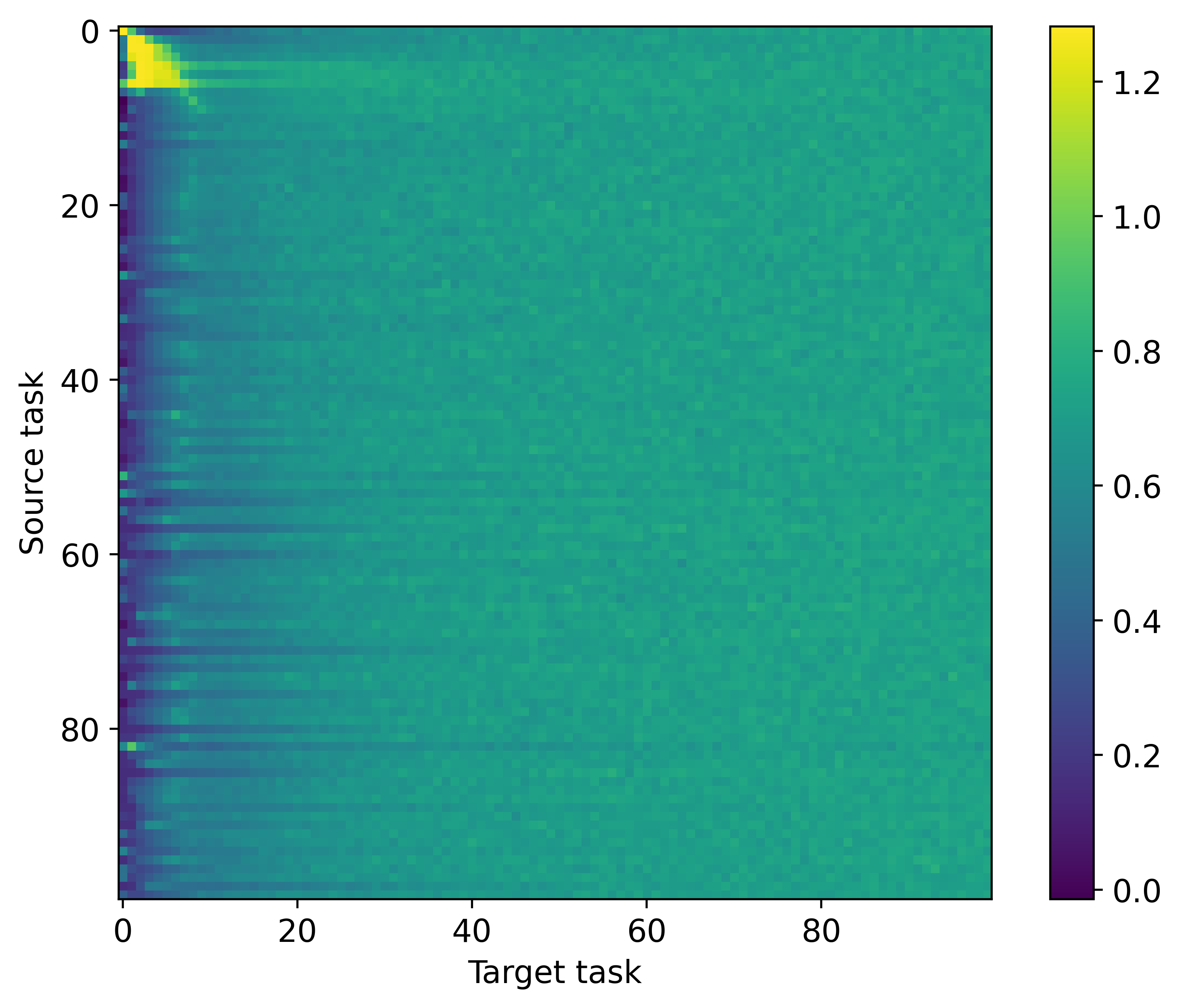}
        \caption{Length of the pendulum variation}
        \label{fig:heatmap-pendulum-l}
    \end{subfigure}
    \hfill 
    \begin{subfigure}[t]{0.32\textwidth}
        \includegraphics[width=\textwidth]{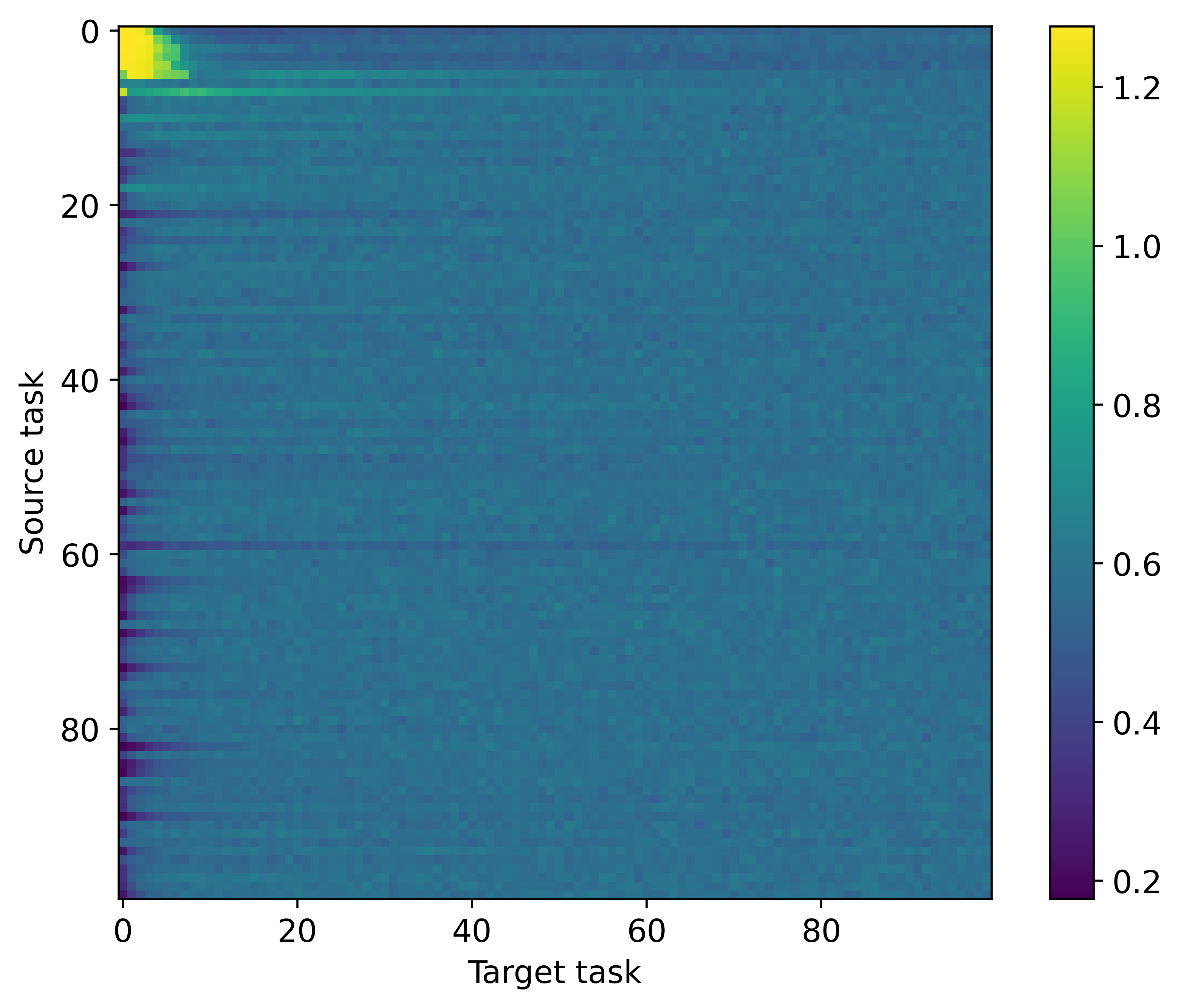}
        \caption{Mass of pendulum variation}
        \label{fig:heatmap-pendulum-m}
    \end{subfigure}
    \caption{Examples of transferability heatmap for Pendulum.}
    \label{fig:heatmap-pendulum}
\end{figure}

\paragraph{Results}

Figure~\ref{fig:result-pendulum} shows a comparison of normalized generalized performance for the Pendulum task across different strategies when varying the timestep, length of the pendulum, and mass of the pendulum. For the length of the pendulum variation and mass of the pendulum one, MBTL strategies demonstrate the highest scores, suggesting robust adaptability to changes in pendulum dynamics. MBTL shows performance close to that of the Oracle across all variations, indicating its effectiveness in handling dynamic changes in system parameters.

\begin{figure}[!h]
    \centering
    \includegraphics[width=0.999\textwidth]{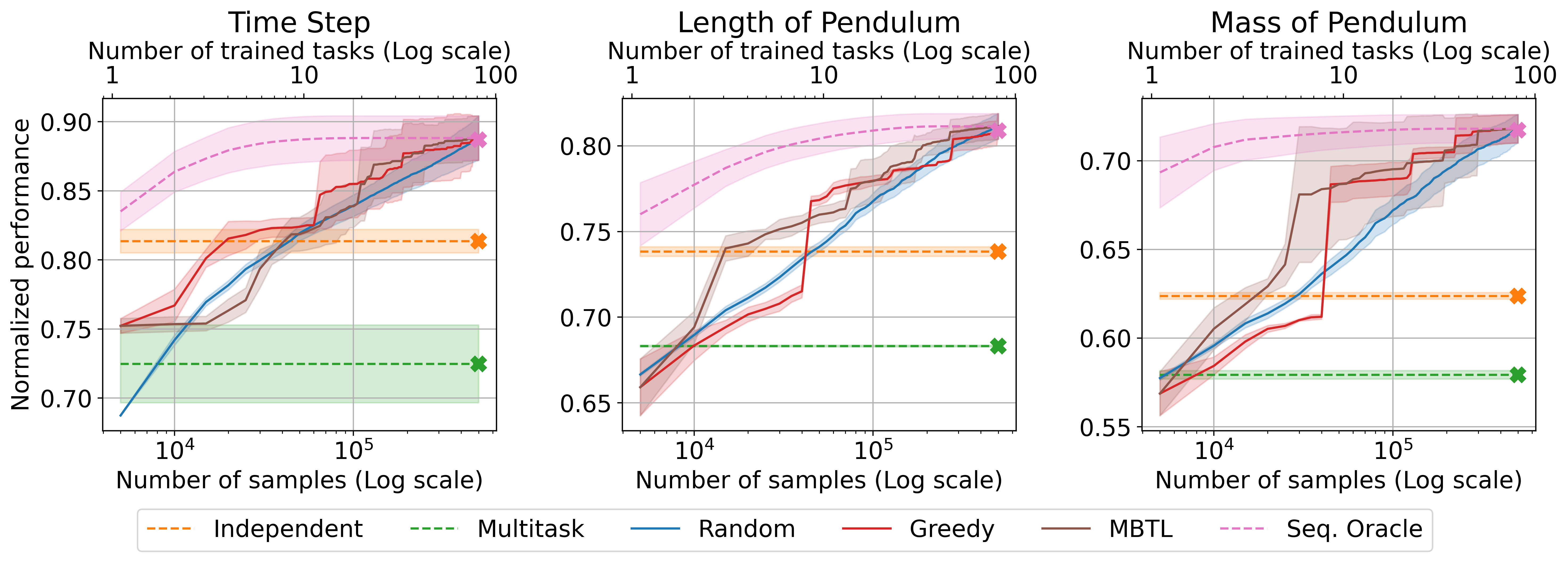}
    \caption{Comparison of normalized generalized performance of all target tasks: Pendulum.}
    \label{fig:result-pendulum}
\end{figure}

% \clearpage
\subsubsection{Details about BipedalWalker \jhhedit{benchmark}}\label{appsec:result-walker}

\paragraph{\jhhhedit{Environment Details.}} \jhhhedit{We utilized the CARL benchmark library’s \emph{default} environment parameters for the BipedalWalker task, training for five million total timesteps. Specifically, BipedalWalker used the default friction of $2.5$, scale of $30$, and gravity of $10$.}

\paragraph{\jhhedit{Potential of multi-policy training and zero-shot transfer}}
\jhhedit{Figure~\ref{fig:gap-walker} compares the performance of different RL training methods in the BipedalWalker benchmark. independent training typically performs nearly as well but suffers intermittent dips. Similarly, multi-task training experiences larger swings, occasionally collapsing to low performance in certain parameter regions. However, Oracle Transfer remains near-perfect across every setting. These patterns highlight how multi-policy training with zero-shot transfer and per-task training generally fare better than a single universal model when faced with diverse environment dynamics.}
\begin{figure}[H]
    \centering
    \includegraphics[width=0.99\textwidth]{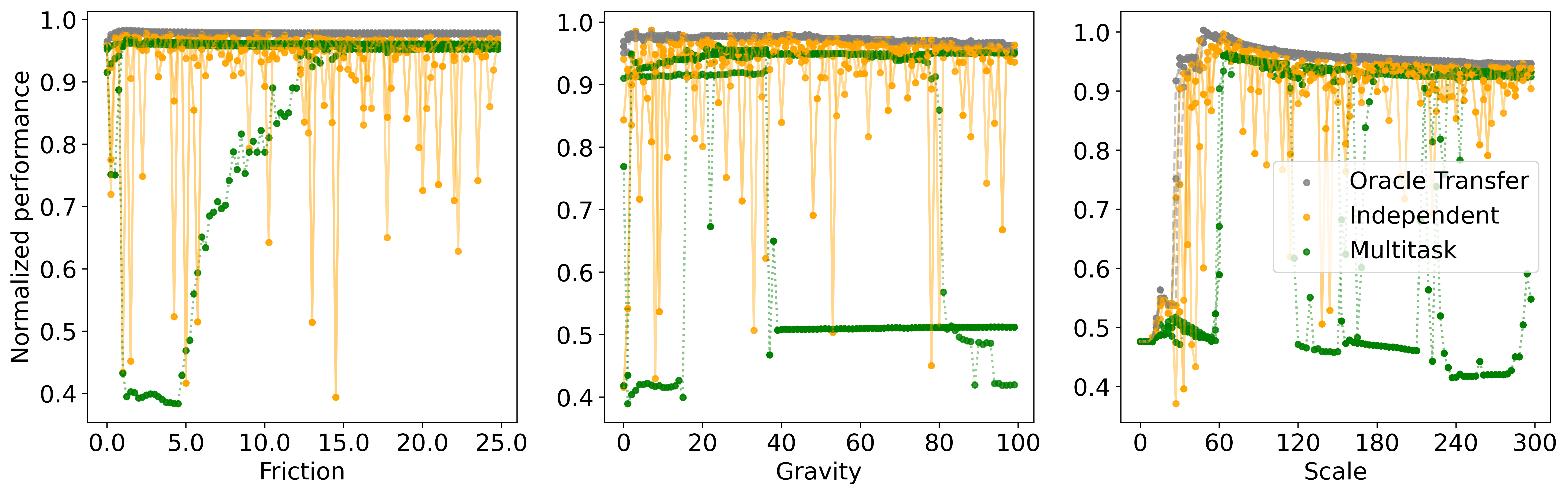}
    \caption{\jhhedit{Normalized performance of Oracle Transfer, independent training, and multi-task training in BipedalWalker benchmarks.}}
    \label{fig:gap-walker}
\end{figure}

\paragraph{Transferability heatmap}

Figure~\ref{fig:heatmap-walker} presents transferability heatmaps for the BipedalWalker task, focusing on three variations: friction, gravity, and scale. Each heatmap illustrates the effectiveness of strategy transfer from source tasks (vertical axis) to target tasks (horizontal axis), highlighting how each parameter influences control strategy adaptability. For friction variation (a), strategies show uniform transferability across different friction levels. In gravity variation (b), transferability is highly variable, suggesting that strategies need specific tuning for different gravity levels. For scale variation (c), the heatmap indicates variable transferability, reflecting the challenges of scaling control strategies.

\begin{figure}[!h]
    \begin{subfigure}[t]{0.32\textwidth}
        \includegraphics[width=\textwidth]{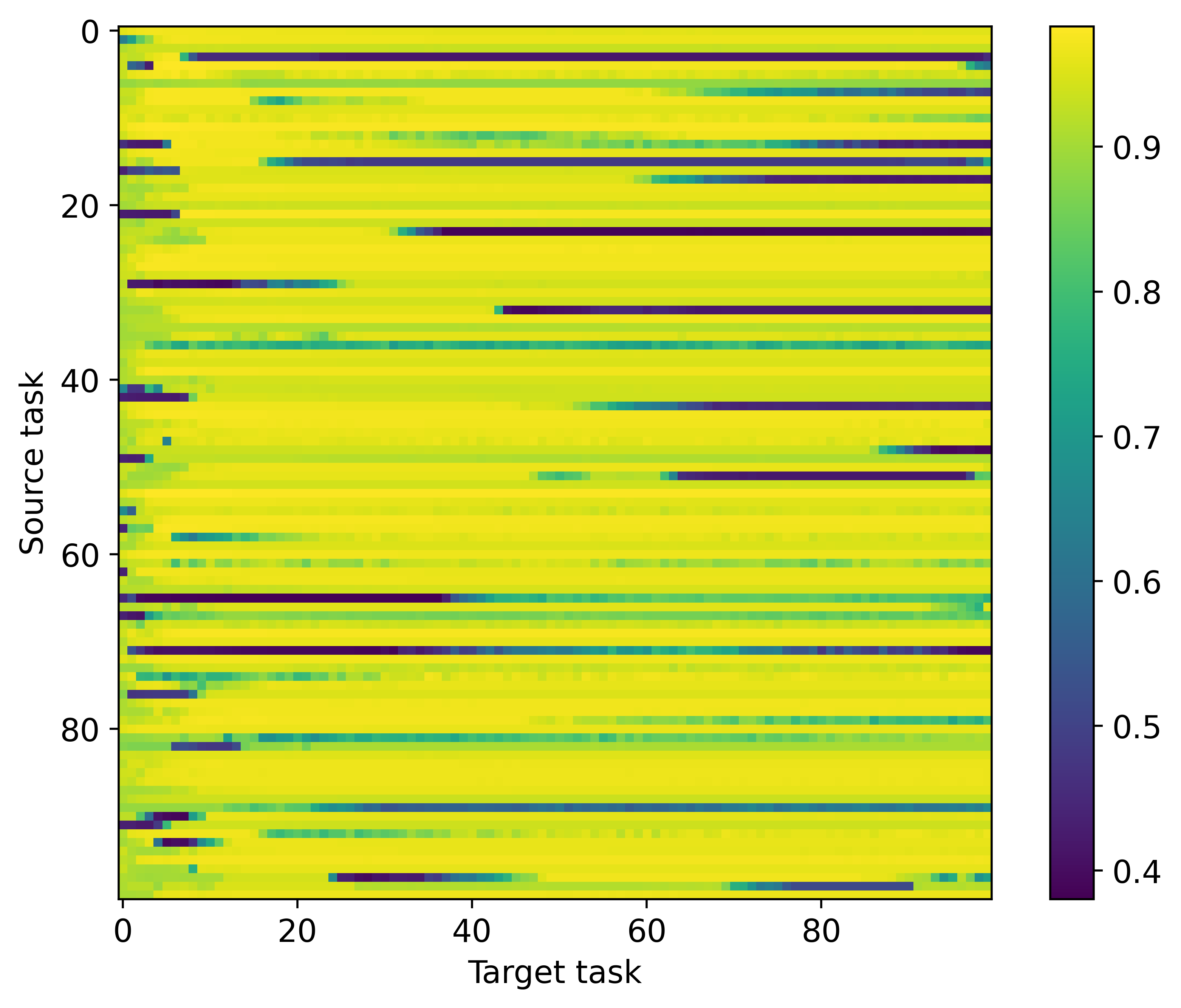}
        \caption{Friction variation}
        \label{fig:heatmap-walker-friction}
    \end{subfigure}
    \hfill 
    \centering
    \begin{subfigure}[t]{0.32\textwidth}
        \includegraphics[width=\textwidth]{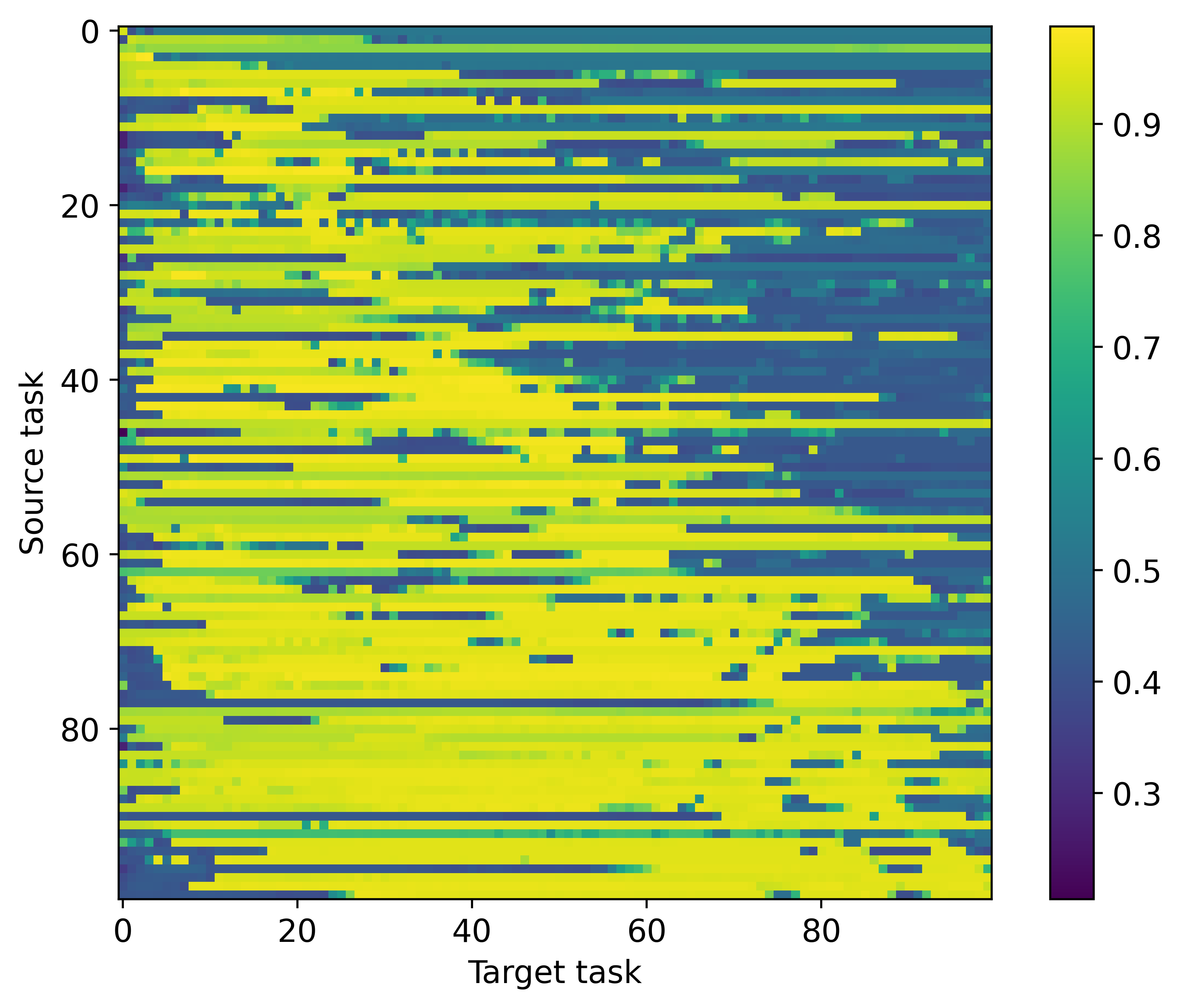}
        \caption{Gravity variation}
        \label{fig:heatmap-walker-gravity}
    \end{subfigure}
    \hfill 
    \begin{subfigure}[t]{0.32\textwidth}
        \includegraphics[width=\textwidth]{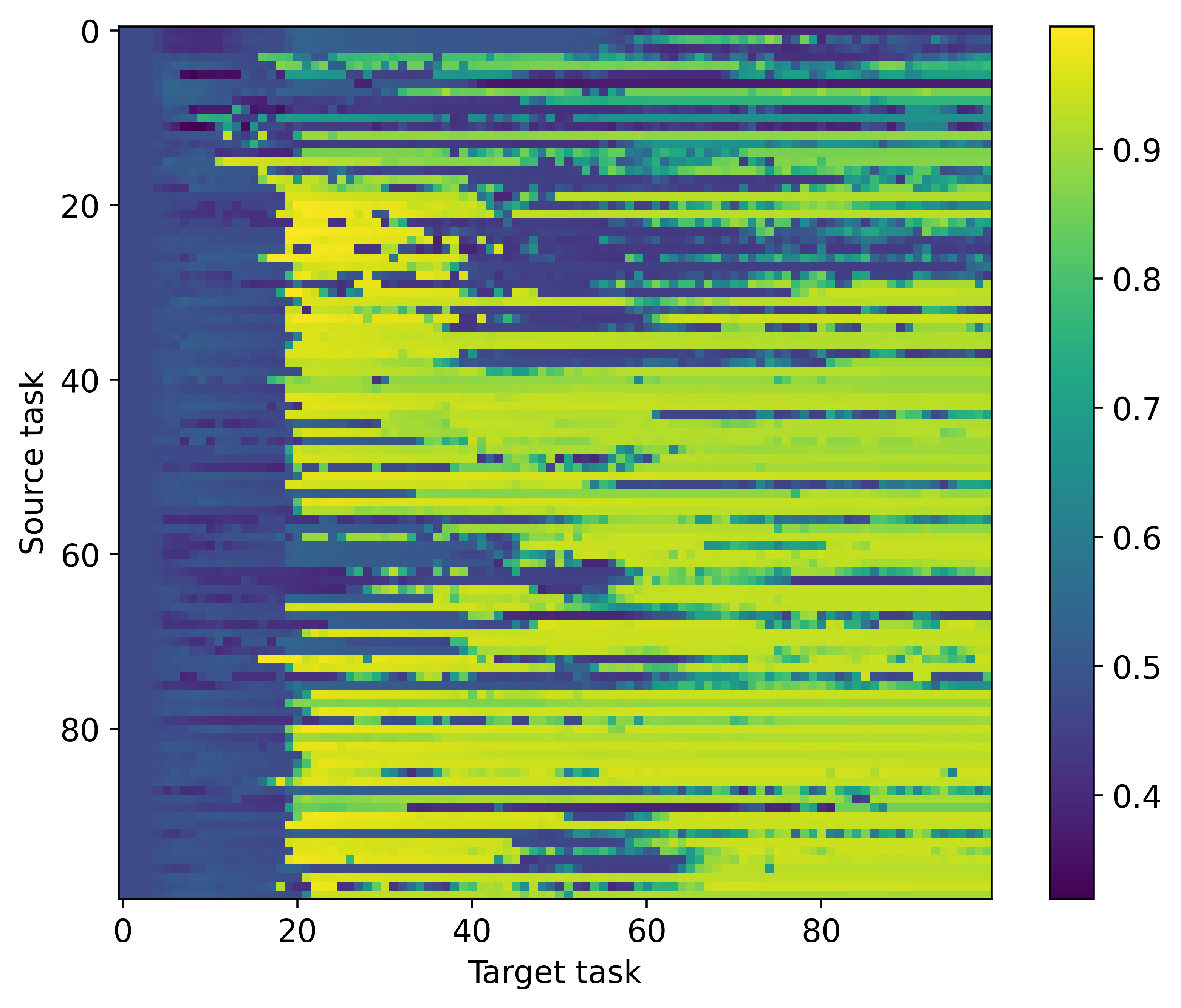}
        \caption{Scale variation}
        \label{fig:heatmap-walker-scale}
    \end{subfigure}
    \caption{Examples of transferability heatmap for BipedalWalker.}
    \label{fig:heatmap-walker}
\end{figure}

\paragraph{Results}
Figure~\ref{fig:result-walker} shows the comparison of normalized generalized performance for all variations within the BipedalWalker task. There is no huge difference in performance for all three cases, but if we look into the tabualr results in Table~\ref{tab:performance-control}, MBTL shows the highest performance across varying conditions, indicating their robustness in adapting to changes in physical parameters of the model. This suggests that these strategies are more effective in handling the complexities introduced by different frictions, gravities, and scales compared to other baselines.

\begin{figure}[!h]
    \centering
    \includegraphics[width=0.999\textwidth]{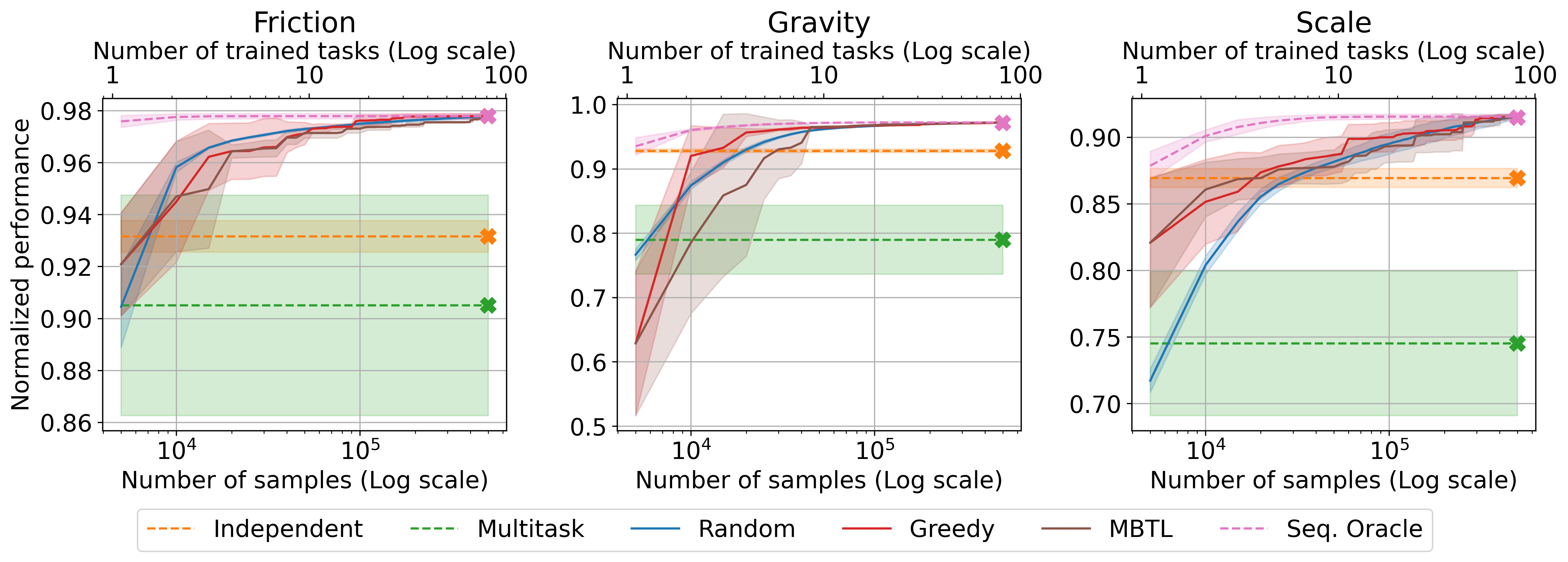}
    \caption{Comparison of normalized generalized performance of all target tasks: BipedalWalker.}
    \label{fig:result-walker}
\end{figure}

% \clearpage
\subsubsection{Details about HalfCheetah \jhhedit{benchmark}}\label{appsec:result-halfcheetah}

\paragraph{\jhhhedit{Environment Details (HalfCheetah).}} \jhhhedit{We utilized the CARL benchmark library’s \emph{default} environment parameters for the HalfCheetah task, training for five million total timesteps. Specifically, HalfCheetah used the default joint stiffness of $15000$, gravity of $9.8$, and friction of $0.6$.}

\paragraph{\jhhedit{Potential of multi-policy training and zero-shot transfer}}
\jhhedit{In this HalfCheetah benchmark, each subplot examines how policies adapt to changing friction, gravity, and stiffness (Figure~\ref{fig:gap-halfcheetah}). Oracle Transfer maintains nearly perfect scores for all parameter ranges, indicating robust zero-shot adaptability. In contrast, independent training experiences larger fluctuations, while multi-task training remains consistent yet at a lower performance plateau. The clear gap between Oracle Transfer and the other methods highlights the advantage of leveraging specialized multi-policy training solutions that effectively transfer across diverse dynamics.}
\begin{figure}[H]
    \centering
    \includegraphics[width=0.99\textwidth]{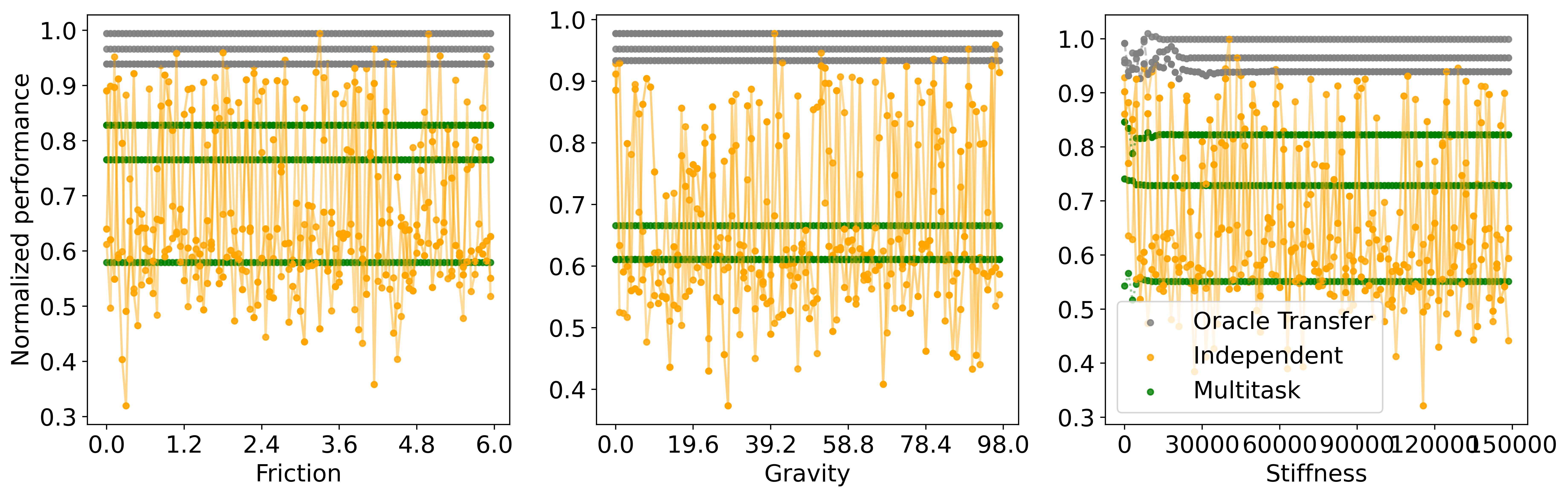}
    \caption{\jhhedit{Normalized performance of Oracle Transfer, independent training, and multi-task training in HalfCheetah benchmarks.}}
    \label{fig:gap-halfcheetah}
\end{figure}

\paragraph{Transferability heatmap}

Figure~\ref{fig:heatmap-halfcheetah} displays transferability heatmaps for the HalfCheetah task, focusing on three physical properties: friction, gravity, and stiffness. Each heatmap demonstrates the transferability of strategies from source tasks (vertical axis) to target tasks (horizontal axis). For friction variation (a), there is uniform high transferability across different friction levels, indicating that strategies are robust to changes in friction. Gravity variation (b) shows less consistent transferability, suggesting a sensitivity to gravity changes that might require adaptation of strategies. Stiffness variation (c) similarly demonstrates variable transferability, highlighting the challenges of adapting to different stiffness levels in control strategies.

\begin{figure}[!h]
    \begin{subfigure}[t]{0.32\textwidth}
        \includegraphics[width=\textwidth]{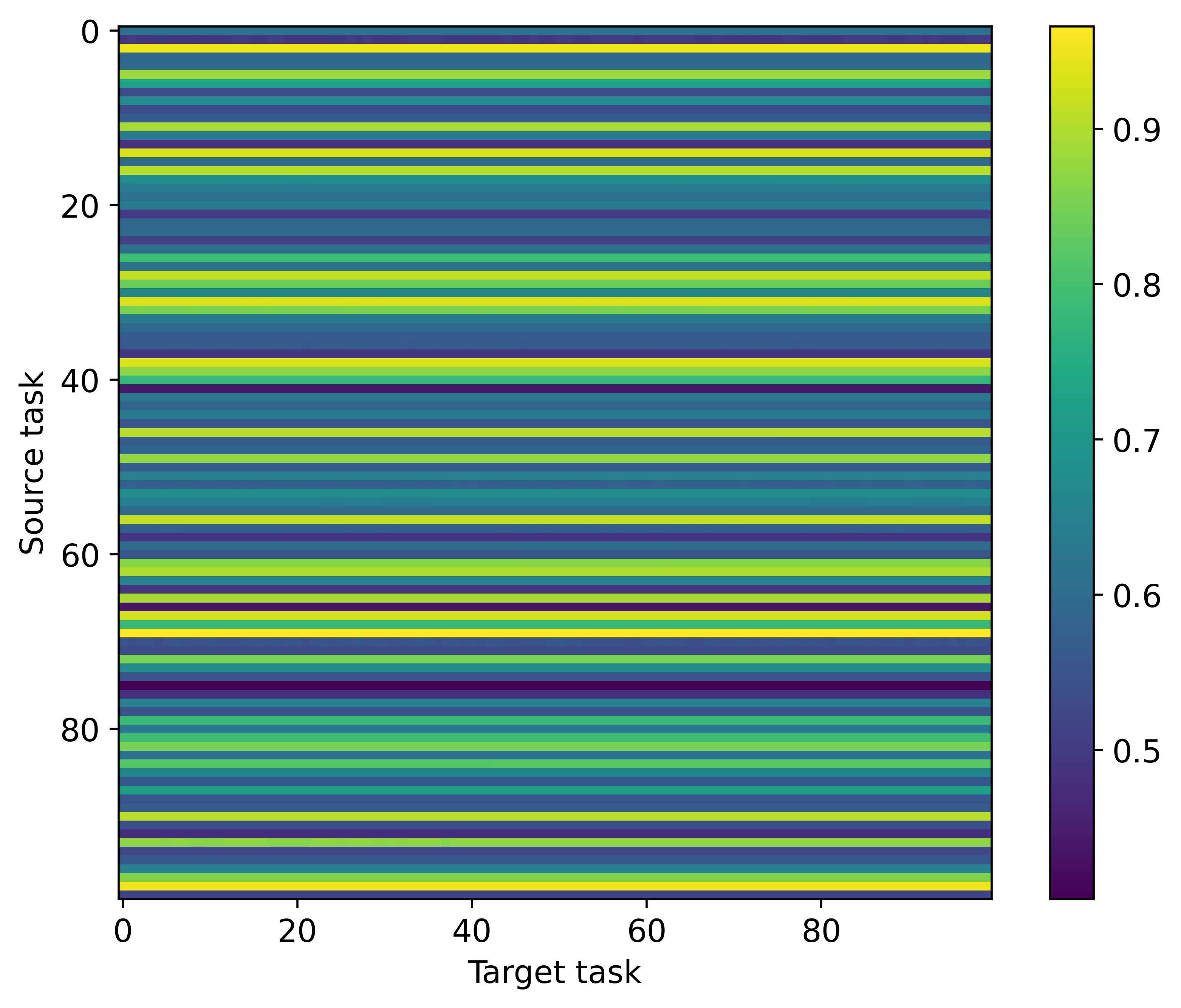}
        \caption{Friction variation}
        \label{fig:heatmap-halfcheetah-friction}
    \end{subfigure}
    \hfill 
    \centering
    \begin{subfigure}[t]{0.32\textwidth}
        \includegraphics[width=\textwidth]{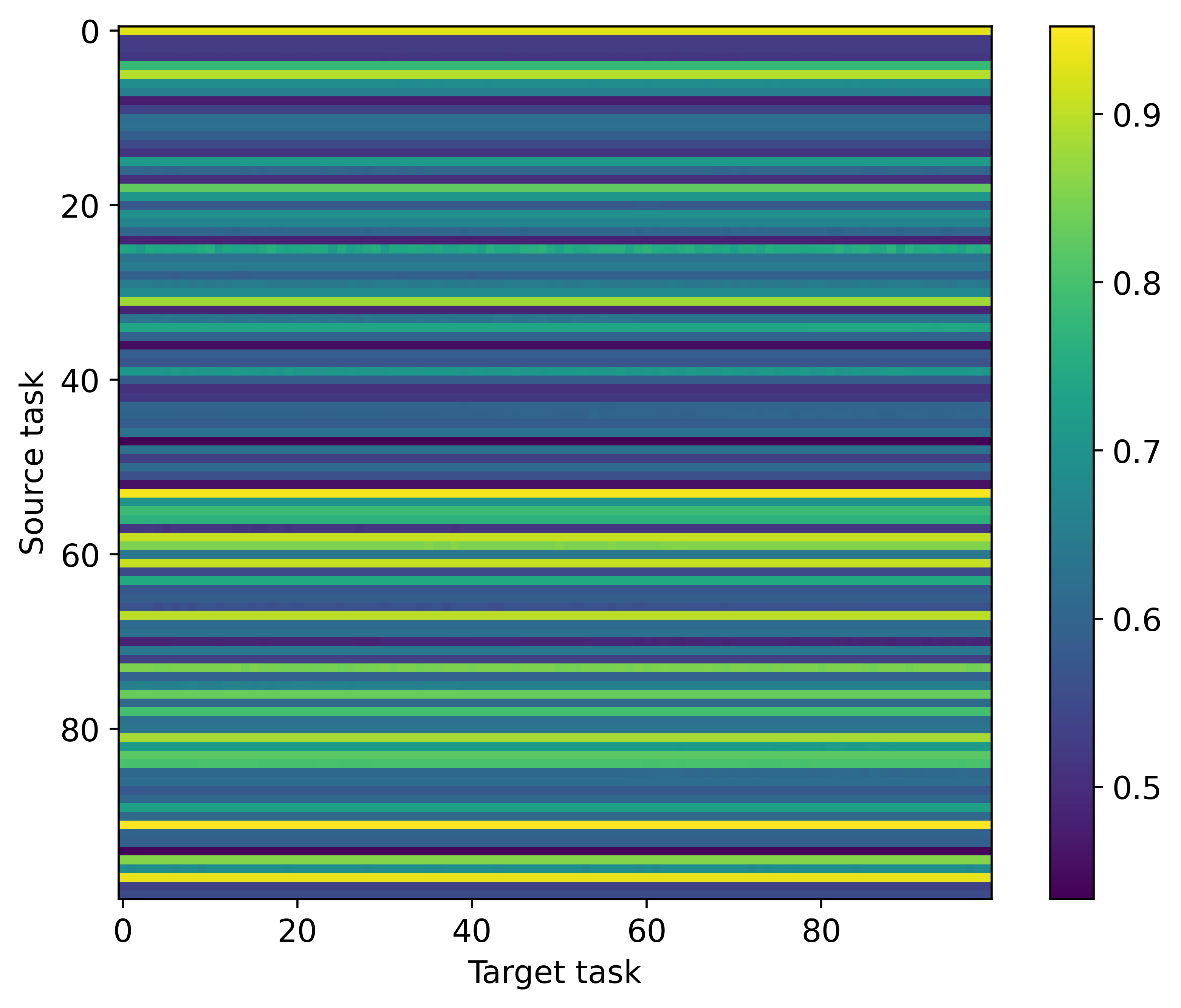}
        \caption{Gravity variation}
        \label{fig:heatmap-halfcheetah-gravity}
    \end{subfigure}
    \hfill 
    \begin{subfigure}[t]{0.32\textwidth}
        \includegraphics[width=\textwidth]{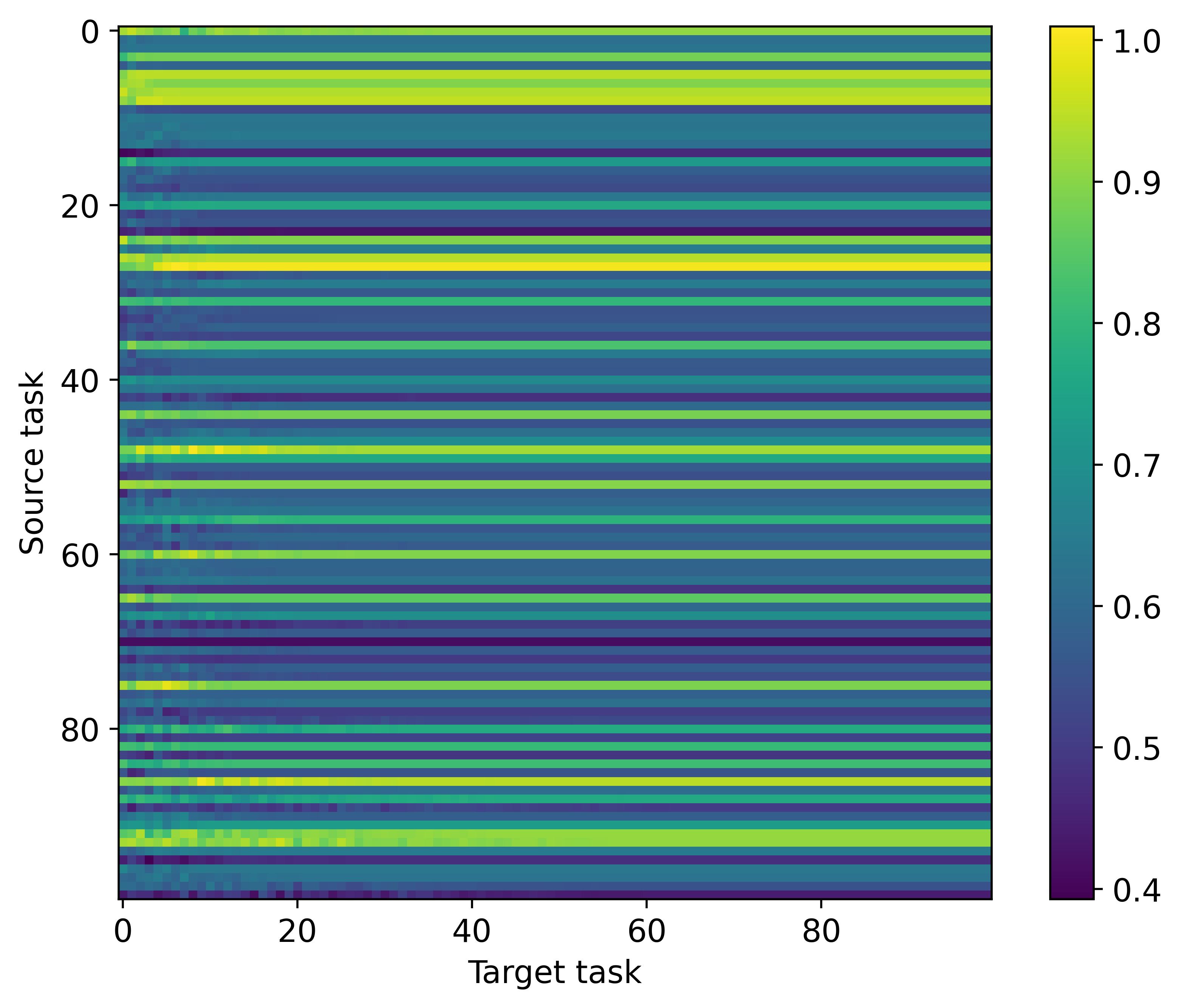}
        \caption{Stiffness variation}
        \label{fig:heatmap-halfcheetah-stiffness}
    \end{subfigure}
    \caption{Examples of transferability heatmap for HalfCheetah.}
    \label{fig:heatmap-halfcheetah}
\end{figure}

\paragraph{Results}

Figure~\ref{fig:result-halfcheetah} presents a comparison of normalized generalized performance across various strategies for the HalfCheetah task with respect to the varied physical properties. The results indicate that the MBTL generally outperforms others, particularly in managing variations in gravity and stiffness, suggesting the superior adaptability of these models to physical changes in the task environment. The trends across different parameters confirm the critical impact of task-specific dynamics on the effectiveness of the tested strategies.

\begin{figure}[!h]
    \centering
    \includegraphics[width=0.999\textwidth]{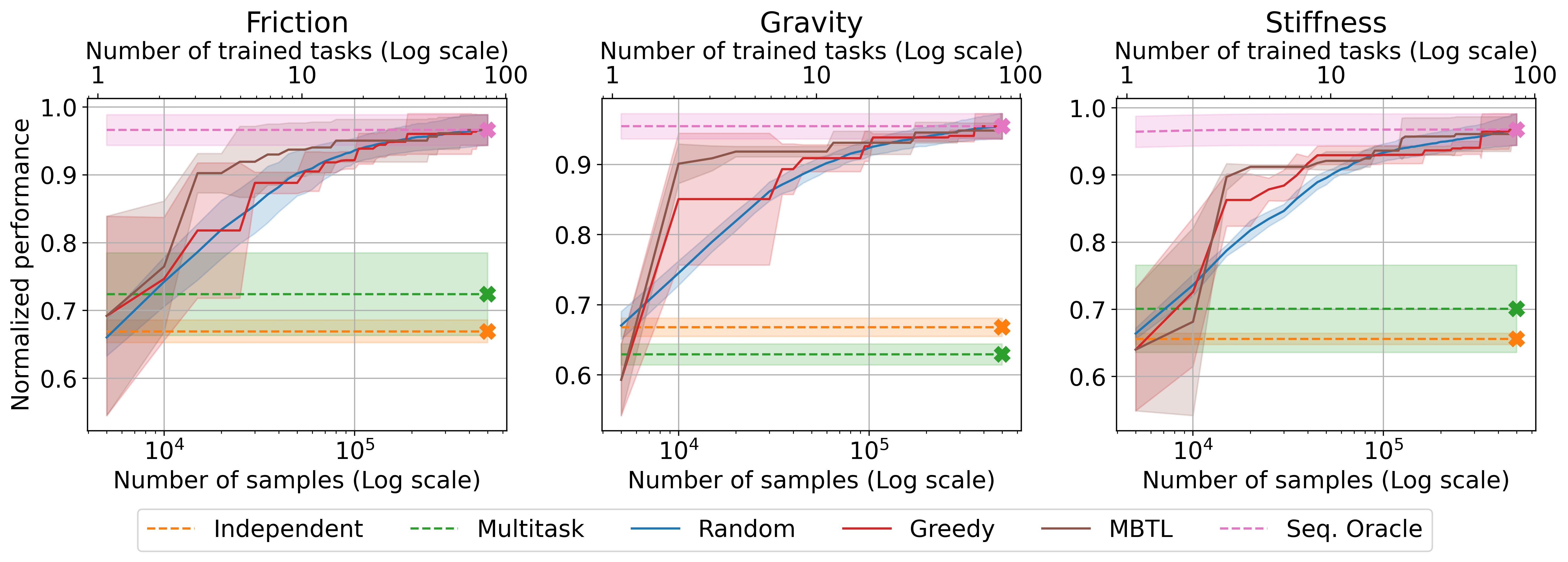}
    \caption{Comparison of normalized generalized performance of all target tasks: HalfCheetah.}
    \label{fig:result-halfcheetah}
\end{figure}

\subsubsection{\jhhedit{Implementation of the recent multi-task baselines}}
\jhhedit{In Figure~\ref{fig:result-cartpole-multitask}, we compare two recent multi-task algorithms—PaCo~\cite{sun_paco_2022} and MOORE~\cite{hendawy_multi_2023}—with several our baselines tested on Cartpole CMDP benchmark. Although MOORE underperforms relative to our baseline multitask implementation, PaCo achieves competitive or even higher performance at certain points, demonstrating its potential to generalize across multiple tasks. Also, it is important to note that those algorithms are naively implemented without thorough investigation. These trends show that enhanced multi-task strategies can be beneficial in some CMDP settings, whereas not all multi-task methods readily adapt to broader parameter variations.}

\begin{figure}[!h]
    \centering
    \includegraphics[width=0.999\textwidth]{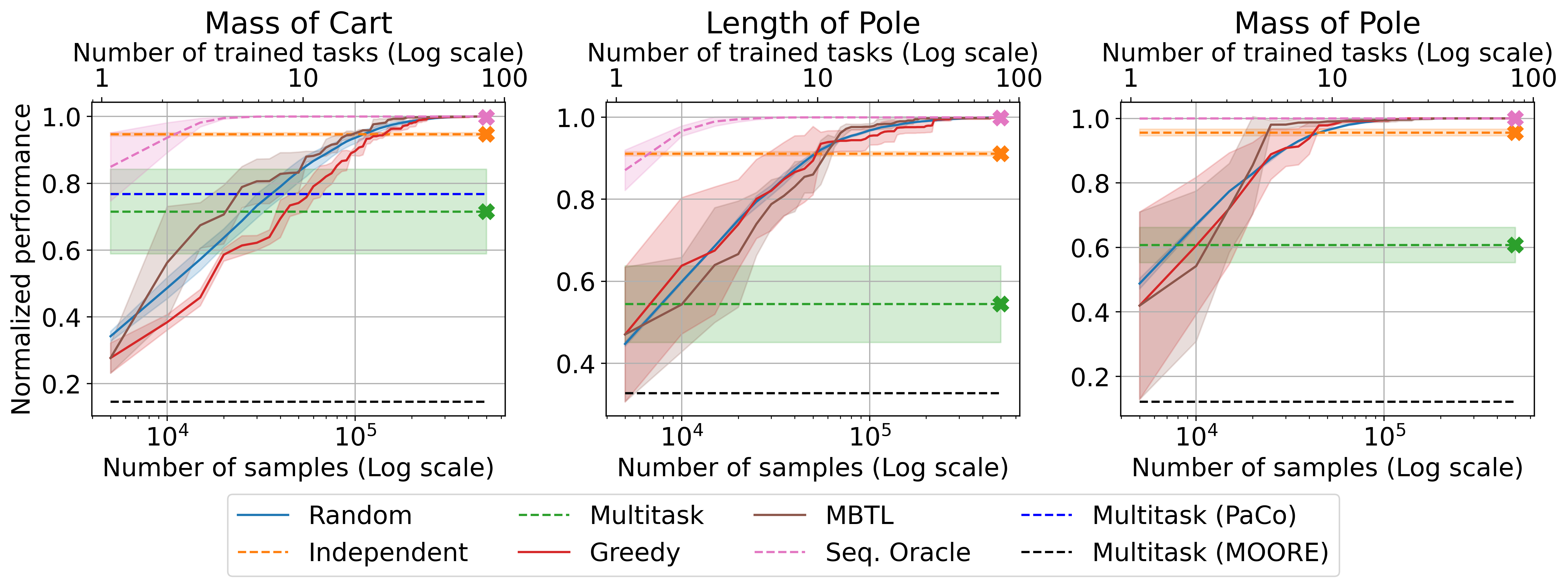}
    \caption{\jhhedit{Normalized performance comparison of PaCo \cite{sun_paco_2022} and MOORE \cite{hendawy_multi_2023} on Cartpole benchmark.}}
    \label{fig:result-cartpole-multitask}
\end{figure}

\clearpage
\subsection{Potential impacts} 
Our work has the potential to reduce the computational effort needed to solve complex real-world problems, offering scalable solutions for implementing deep reinforcement learning in dynamic environments. While there are no immediate negative societal impacts identified, ongoing research will continue to assess the broader implications of deploying these technologies in urban settings.

% %%%%%%%%%%%%%%%%%%%%%%%%%%%%%%%%%%%%%%%%%%%%%%%%%%%%%%%%%%%%

% \clearpage
% \input{checklist}

%%%%%%%%%%%%%%%%%%%%%%%%%%%%%%%%%%%%%%%%%%%%%%%%%%%%%%%%%%%%

\end{document}